\documentclass{article} 
\usepackage{iclr2027_conference,times}

\usepackage{amsmath,amsfonts,bm}

\def\eqref#1{equation~\ref{#1}}

\def\1{\bm{1}}

\DeclareMathAlphabet{\mathsfit}{\encodingdefault}{\sfdefault}{m}{sl}
\SetMathAlphabet{\mathsfit}{bold}{\encodingdefault}{\sfdefault}{bx}{n}

\usepackage{hyperref}
\usepackage{url}
\usepackage{amssymb}
\usepackage{pifont}
\usepackage{graphicx} 
\usepackage[table,xcdraw]{xcolor}
\usepackage{multirow}
\usepackage[table,xcdraw]{xcolor}
\usepackage{algorithm}
\usepackage{algorithmic}
\usepackage{amsthm}
\usepackage[table]{xcolor}
\usepackage{capt-of}
\usepackage{subcaption}
\usepackage{wrapfig}
\usepackage{todonotes}
\usepackage{graphicx}
\usepackage{makecell}
\usepackage{booktabs}
\usepackage{makecell}

\usepackage{arydshln} 

\newtheorem{theorem}{Theorem}

\newtheorem{proposition}{Proposition}

\title{RACE: RoPE-Aligned Attention Compression Based on Query-Key Activation Energies}
\iclrfinalcopy

\author{
\begin{minipage}{\textwidth}
\centering
\small
\textbf{%
Irene Tenison\textsuperscript{1}\thanks{Contact: itenison@mit.edu},
Abdulrahman Alabdulkareem\textsuperscript{1},
Daniel Karl I. Weidele\textsuperscript{2},\\
Nima Dehmamy\textsuperscript{2},
Dan Gutfreund\textsuperscript{2}
}
\\[0.7em]
\normalfont
\textsuperscript{1}MIT CSAIL
\qquad
\textsuperscript{2}MIT-IBM Computing Research Lab
\end{minipage}
}
\begin{document}

\maketitle

\begin{abstract}
Head-dimension compression can reduce the memory required by the
key-value (KV) cache, but existing methods often reconstruct
full-dimensional representations during inference, adding significant decoding FLOPs overhead. We formalize the design considerations for converting pretrained attention into a natively reduced-dimensional form that reduces both KV-cache storage and attention computation using a calibration set. Specifically, we consider reductions that (1) apply a shared projection to queries and keys so that they remain in a common reduced feature space, (2) derive this projection jointly from query and key activations to better preserve their interaction, and (3) are compatible with the pretrained RoPE transformation. To obtain a reconstruction-free realization of this reduction, we restrict it to coordinate selection and theoretically show that RoPE compatibility then holds if and only if complete rotary pairs are retained. Building on this characterization, we introduce RACE
(\textbf{R}oPE-Aligned \textbf{A}ttention \textbf{C}ompression based on
activation \textbf{E}nergy), a post-training method that selects rotary
pairs using their joint query-key activation energy. We theoretically
show that this criterion minimizes a separable, calibration-dependent
upper bound on post-RoPE query-key product approximation error. The
resulting reductions are folded into the model weights and require
neither runtime reconstruction nor specialized kernels. Across three model families and different compression levels, RACE outperforms structured
pruning and remains competitive with reconstruction-based methods on language-modeling, zero-shot, and long-context tasks. At 50\% KV-cache retention and 256K context, RACE reduces total
inference FLOPs by 47.8\%, yielding $1.79\times$ higher throughput,
43.2\% lower time-to-first-token (TTFT), and 28.1\% lower allocated GPU memory in comparison to an uncompressed model. We also show that RACE is compatible with token eviction or compaction as well as quantization based KV-cache compression methods.
\end{abstract}

\section{Introduction}
\label{sec:intro}
Autoregressive language-model inference is increasingly constrained by the key-value (KV) cache, especially at long contexts. Each attention layer stores a key and value vector for every preceding token, so cache memory grows linearly with context length and with the dimensionality of the cached representations. Reducing the key and value head dimensions offers the possibility of decreasing both KV-cache memory and native attention computation. Realizing these savings is more difficult than merely storing a low-dimensional encoding. Several head-dimension compression methods project keys and values into compact representations but reconstruct full-dimensional keys during inference to effectively apply rotary positional embeddings (RoPE) \citep{saxena2024eigen,chang2025palu}. Such methods reduce cache storage but introduce decompression and additional arithmetic into the attention path during inference. To obtain the full benefit of head-dimension compression, the attention layer itself must operate directly in the reduced dimensions.

This raises a more fundamental question: \emph{when can a pretrained RoPE-based attention layer be converted into a natively lower-dimensional attention layer without reconstructing the discarded coordinates?} Since attention logits depend on the post-RoPE product \(QK^\top\), we view head-dimension reduction as an approximate matrix multiplication (AMM) problem along the shared feature dimension. However, preserving this interaction in the post-RoPE space is not sufficient for native inference. It must be realizable before RoPE. This imposes an additional structural constraint that the reduction must commute with the pretrained rotary transformation. We characterize this condition exactly for coordinate-selection reductions.

These observations define an admissible family of native Q/K reductions: queries and keys can share a reduced feature space, and the reduction must respect the pretrained RoPE structure. The remaining question is which admissible dimensions to retain. We introduce RACE, which ranks complete rotary pairs using joint activation energy for each RoPE pair $m$, \(E_m=\|Q_m\|_F^2+\|K_m\|_F^2\), and retains the highest-energy pairs. This criterion minimizes a separable, calibration-dependent upper bound on post-RoPE query-key product approximation error. The resulting Q/K reduction, together with a calibrated value subspace, is folded into the attention projections, enabling RoPE, caching, and attention to operate entirely in the reduced dimensions without runtime reconstruction or specialized kernels. Summarizing our contributions:

\vspace{-0.15cm}
\begin{itemize}
    \item \textbf{Constrained-AMM formulation.} We formulate native attention compression as preservation of the post-RoPE query-key product subject to pre-RoPE realizability constraint.

    \item \textbf{Structural Characterization.}
    We prove that a coordinate-selection reduction can commute with the pretrained RoPE if and only if it retains complete rotary pairs,
    characterizing the admissible family of coordinate reductions.

    \item \textbf{Forward-only pair selection within the family.} RACE ranks RoPE pairs using joint query-key activation energy, minimizing a separable calibration-dependent upper bound.

    \item \textbf{Empirical validation.} Across multiple model families and compression levels, RACE remains competitive with reconstruction-based methods while reducing KV-cache memory, attention cost, parameters, TTFT, and runtime overhead without specialized kernels.
    
\end{itemize}
\vspace{-0.2cm}

\begin{figure}
    \centering
    \begin{minipage}[t]{0.495\columnwidth}
        \centering
        \includegraphics[width=\linewidth]{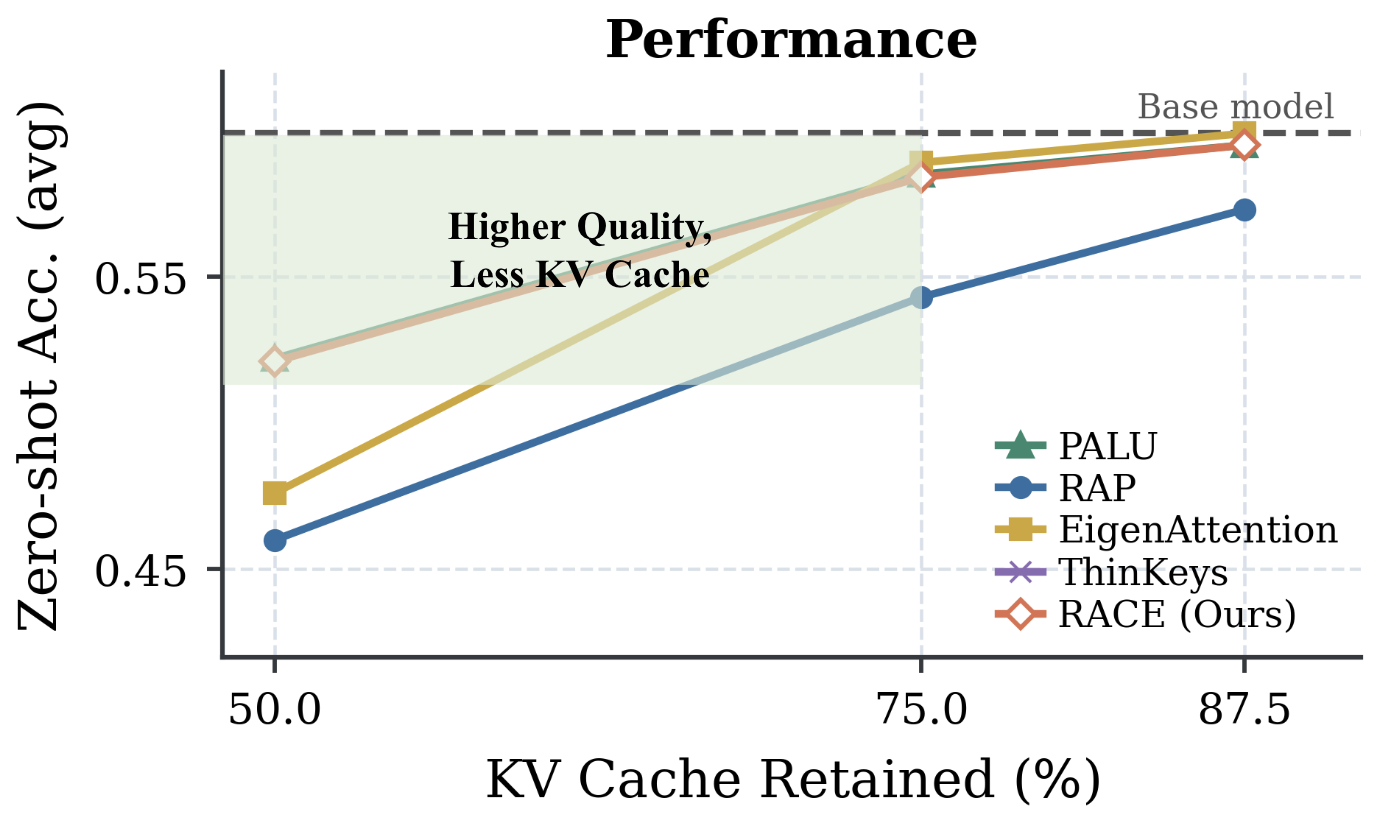}
    \end{minipage}
    \hfill
    \begin{minipage}[t]{0.495\columnwidth}
        \centering
        \includegraphics[width=\linewidth]{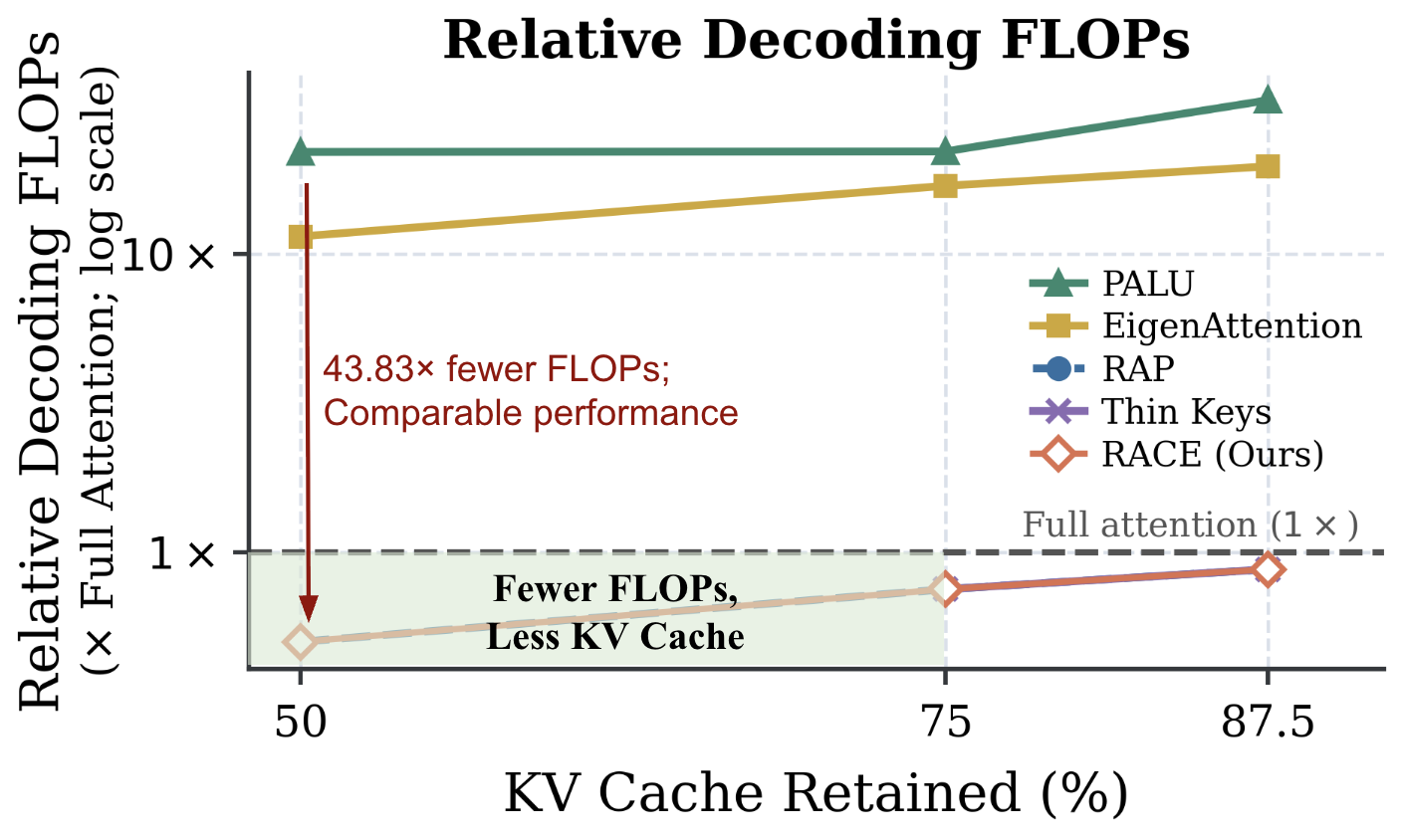}
    \end{minipage}
    \vspace{-0.6cm}
    \caption{\textbf{Quality-efficiency trade-off across KV-cache retention levels on Llama-3.2-3B-Instruct.} Left: average zero-shot accuracy. Right: relative decoding FLOPs on a logarithmic scale, normalized to full attention ($1\times$). The dashed line represent the uncompressed model. RAP and Thin Keys have decoding FLOPs that overlap those of RACE, while Thin Keys' accuracy falls below the displayed y-axis. \emph{RACE combines accuracy competitive with reconstruction-based methods with the low decoding cost of native dimensional reduction.}}
    \label{fig:quality-efficiency}
    \vspace{-0.3cm}
\end{figure}

\begin{table}[]
\caption{
\textbf{Comparison of head-dimension KV-cache compression methods.}
RACE combines a shared, jointly selected, RoPE-compatible Q/K reduction
with native reduced-dimensional execution. FLOPs are normalized to full attention. See ~\ref{app:flops_analysis} for
derivations of FLOPs.
}
\vspace{-0.2cm}
\label{tab:flops_analysis}
\resizebox{\columnwidth}{!}{%

\begin{tabular}{ccccccccccc}
\hline
\rowcolor[HTML]{EFEFEF} 
\textbf{}                & \multicolumn{3}{c}{\cellcolor[HTML]{EFEFEF}\textbf{Excecution}}                                                                     & \multicolumn{3}{c}{\cellcolor[HTML]{EFEFEF}\textbf{Q/K Reduction Criterion}}                                                                                                                                                      & \multicolumn{4}{c}{\cellcolor[HTML]{EFEFEF}\textbf{Efficiency}}                                                                                                                                                                                                                      \\
\textbf{Method}          &  \textbf{\begin{tabular}[c]{@{}c@{}}Reconst-\\ ruction free\end{tabular}} & \textbf{\begin{tabular}[c]{@{}c@{}}Custom-\\ Kernel free\end{tabular}} & 
\textbf{\begin{tabular}[c]{@{}c@{}}Fwd.-only\\ Calibration\end{tabular}} & \textbf{\begin{tabular}[c]{@{}c@{}}Shared \\Q/K Map\end{tabular}} & \textbf{\begin{tabular}[c]{@{}c@{}}Joint Q/K\\ Signal\end{tabular}} & \textbf{\begin{tabular}[c]{@{}c@{}}Native\\ RoPE\end{tabular}} & \textbf{\begin{tabular}[c]{@{}c@{}}KV Cache\\ Retained\end{tabular}} & \textbf{\begin{tabular}[c]{@{}c@{}}Attn. Params\\ Retained\end{tabular}} & \textbf{\begin{tabular}[c]{@{}c@{}}Prefill\\ FLOPs\end{tabular}} & \textbf{\begin{tabular}[c]{@{}c@{}}Decode\\  FLOPs\end{tabular}} \\ \hline
\rowcolor[HTML]{EFEFEF} 
Attention \citep{attention} &       -                                                               & -                                                                  &            -                                                      &            -         &      -                                                                   &             -                                                                  & 100\%                                                                  & 100\%                                                                     & $\times$1                                                                & $\times$1                                                                \\
EigenAttention \citep{saxena2024eigen}         &                             \ding{55}                                       &          \ding{51}       &          \ding{51}                                                 &   \ding{55}                                                                      &            \ding{51}                                                             &               \ding{55}                                                                & 50\%                                                                   & 69.8\%                                                                    & $\times$0.75                                                             & $\times$11.42                                                            \\
\rowcolor[HTML]{EFEFEF} 
PALU \citep{chang2025palu}                 &   \ding{55}                                                                      &     \ding{55}                                                                 &                 \ding{51}            &      \ding{55}                                                              &      \ding{55}                                                                     &           \ding{55}                                                                      & 50\%                                                                   & 58.3\%                                                                    & $\times$0.75                                                                & $\times$21.92                                                            \\
Thin Keys \citep{thinkeys}   &         \ding{51}                                                               &                          \ding{51}    &   \ding{51}                                         &         \ding{55}                                                                  &                  \ding{55}                                                         &                            \ding{55}                                                     & 75\%                                                                   & 50\%                                                                      & $\times$0.75                                                             & $\times$0.75                                                             \\
\rowcolor[HTML]{EFEFEF} 
RAP          \citep{rap}            &         \ding{51}                                                               &        \ding{51}                                                             &       \ding{55}  &   \ding{51}                                                                        &       \ding{55}                                                                    &              \ding{51}                                                                   & 50\%                                                                  & 50\%                                                                     & $\times$0.5                                                              & $\times$0.5                                                              \\
\textbf{RACE (Ours)}     & \textbf{ \ding{51} }                                                            & \textbf{ \ding{51} }                                                         & \textbf{ \ding{51} }                                                               & \textbf{ \ding{51} }                                                               & \textbf{ \ding{51} }              & \textbf{ \ding{51} }                                                       & \textbf{50\%}                                                          & \textbf{50\%}                                                             & \textbf{$\times$0.5}                                                     & \textbf{$\times$0.5}                                                     \\ \hline
\vspace{-0.8cm}
\end{tabular}}
\end{table}

\vspace{-0.2cm}
\section{Related Work}

\paragraph{Compression along other KV-cache axes.}
KV-cache memory can be reduced along the token, precision, or head-dimension
axes as shown in Fig. \ref{fig:cache_axes}. Token-level methods reduce the number of cached positions ($S$ in Eq.\ref{eq1}), through eviction, retention, or compaction
\citep{zhang2023h2o,xiao2024streamingllm,li2024snapkv,cai2024pyramidkv, fu2025headkv,shi2025lacache}. Quantization methods instead reduce the number of
bits used for each cached element, ($b$ in Eq.\ref{eq1}),
\citep{liu2024kivi,hooper2024kvquant,wang2025squat}. Our method preserves the
retained token positions and reduces the number of elements in every key and
value vector. It is therefore complementary to both token reduction and
quantization; we evaluate representative combinations in
Sec.~\ref{sec:complementary}.

\paragraph{Head-dimension KV-cache compression.}
Head-dimension methods reduce the width of the cached representations ($d_k$ and $d_v$ in Eq.\ref{eq1}) for keys and values respectively.
EigenAttention \citep{saxena2024eigen} derives low-rank subspaces from
calibrated activations, whereas PALU \citep{chang2025palu} factorizes attention
projection weights. The dense low-rank compression subspaces in these and similar methods are not generally invariant under the pretrained
RoPE transformation and so these methods recover full-dimensional keys before applying RoPE. This can be compute-intense and so, methods like PALU \citep{chang2025palu} mitigate this by using a custom kernel that does not materialize the full-dimensional reconstructed key tensor. However the inference FLOPs remains dominated by the dense reconstruction basis. 

Thin Keys \citep{thinkeys} avoids key reconstruction by
reducing the shared Q/K dimension, but retains all dimensions of Values. A concurrent preprint, RAP \citep{rap}, performs RoPE-aligned structured parameter pruning using empirical-Fisher scores computed from gradients of $W_K$ and arrives at a similar RoPE-aligned native architecture from a different starting point to RACE. While RAP approaches the problem as structured parameter pruning, RACE begins from the problem of constrained-AMM as shown in Sec.\ref{sec:amm_motivation}. The methods can therefore produce similar reduced-dimensional architectures, but differ in their compression objective and calibration signal. Additionally, our calibration is forward-only and does not depend on a task loss or the batch-sensitivity of Fisher estimates \citep{adabet}. Figure \ref{fig:quality-efficiency} compares our solution to previous techniques in terms of accuracy and FLOPs across a range of compression levels.

Methods including MatryoshkaKV \citep{lin2025matryoshkakv}, STAR-KV
\citep{starkv}, ThinK \citep{think}, and OjaKV
\citep{ojakv} additionally study heterogeneous ranks or
input-dependent compression. These methods primarily
address \emph{where} or \emph{when} to allocate compression across layers,
heads, or inputs. Our work instead addresses \emph{how} to construct a fixed,
RoPE-compatible reduced representation that operates natively without runtime
selection, adaptation, or full-dimensional reconstruction.

\begin{figure}[t]
    
        \centering
        \includegraphics[width=\linewidth]{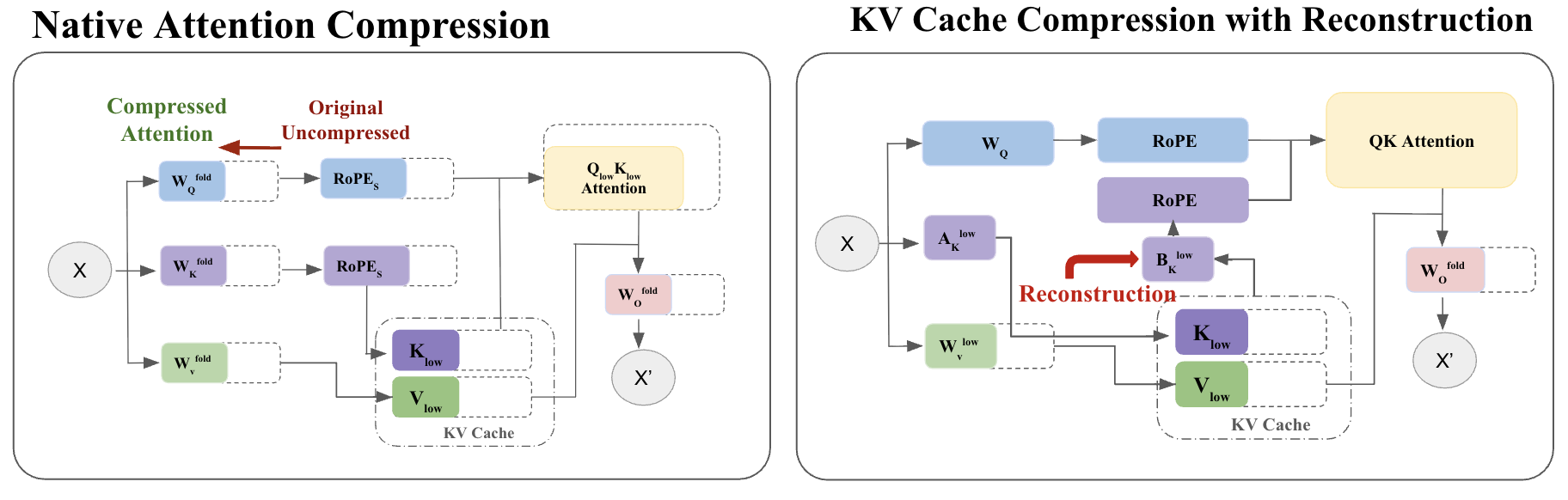}
\vspace{-0.5cm}
    \caption{\textbf{Native attention compression along the head-dimension.}  Reconstruction-based methods (right) may cache low-dimensional representations yet return to the original feature dimension during inference. This is often done using $A_K^{low}$ and $B_K^{low}$ that are decomposition of $W_K$. We instead seek a native reduction (left) that can be applied before RoPE and folded into the pretrained projections, so that attention is executed natively in the reduced dimension. Dashed boxes indicate the corresponding full-dimensional operations.}
    \label{fig:compressed_attn}
    \vspace{-0.6cm}
\end{figure}

\section{Method}

\paragraph{Problem Formulation.}
\label{sec:problem_formulation}
Consider an attention head with pre-RoPE activations
$Q,K\in\mathbb{R}^{S\times d_k}$ and
$V\in\mathbb{R}^{S\times d_v}$.
Given target dimensions $r_k<d_k$ and $r_v<d_v$, we seek
reduced-dimensional representations
\begin{equation}
     Q' =  QP_Q,\qquad
     K' = KP_K,\qquad
    V' = VP_V,
\end{equation}
with $P_Q,P_K\in\mathbb{R}^{d_k\times r_k}$ and
$P_V\in\mathbb{R}^{d_v\times r_v}$, such that the resulting layer
stores only $K'$ and $V'$ in the KV cache and performs attention
directly in dimensions $r_k$ and $r_v$, without reconstructing the
discarded coordinates at inference.

For native reduced-dimensional execution, the Q/K reductions must
ultimately be applied to the \emph{pre-}RoPE activations so that they can be
folded into the query and key projection weights. However, the quantity
that determines attention logits is the \emph{post-}RoPE query-key product. 
We therefore first characterize the low-dimensional Q/K representation
that preserves this interaction, and then determine when it can be
realized equivalently before RoPE. Finally we introduce RACE, a criterion that satifies these conditions and effectively convert a pretrained attention to a natively reduced-dimensional attention.

\subsection{Native Attention Compression as Constrained AMM}
\label{sec:amm_motivation}

Let $\widetilde Q=\operatorname{RoPE}(Q),
    \widetilde K=\operatorname{RoPE}(K)$.
Since attention logits depend on the post-RoPE product
$\widetilde Q\widetilde K^\top$, we formulate Q/K head-dimension
reduction as an approximate matrix multiplication (AMM) problem.
We are seeking a column-orthonormal reduction
$P\in\mathbb{R}^{d_k\times r_k}$ such that,
\begin{equation}
    (\widetilde QP)(\widetilde KP)^\top
    =
    \widetilde QPP^\top\widetilde K^\top
    \approx
    \widetilde Q\widetilde K^\top .
    \label{eq:amm_qk}
\end{equation}
This is consistent with classical matrix-sketching and AMM results,
where a common lower-dimensional map is applied along the shared
dimension of the two operands
\citep{woodruff2014sketching,cohen2016optimal}. Detailed explanation is given in \ref{app:amm}.

This view motivates two design principles. First, \emph{queries and keys should be represented in a shared low-dimensional subspace}, since their
interaction is evaluated along the same feature dimension. We therefore
seek a common reduction $P_{QK}$ for both operands. Second, when this
subspace is chosen from calibration data, \emph{its selection criterion should
reflect both matrices whose product is being approximated}. Because the
target interaction depends jointly on $\widetilde Q$ and $\widetilde K$,
we construct $P_{QK}$ from statistics of both query and key activations
rather than either operand in isolation as in the concurrent work, RAP\citep{rap} or ThinKeys\citep{thinkeys}. This is validated empirically in Fig.\ref{fig:error}.



These observations specify how the Q/K interaction should be reduced, but not yet whether the reduction can be executed natively before RoPE to get the desired systems benefits. Applying \(P_{QK}\) only after RoPE would still require full-dimensional query and key projections and full-dimensional rotary encoding, reducing the KV cache width without fully reducing the attention computation as in reconstruction-based methods. \emph{Native execution therefore requires the shared reduction to be realizable before RoPE}. Section \ref{sec:amm_motivation} establishes how the query-key interaction should be reduced, but not whether the shared reduction can be realized before RoPE. Section \ref{sec:rope_qk} characterizes the additional structural condition under which this is possible.

\subsection{RoPE-Compatible Query--Key Reduction}
\label{sec:rope_qk}

Section~\ref{sec:amm_motivation} motivates representing queries and keys
through a shared low-dimensional map. To apply this reduction
\emph{before} RoPE, it must preserve the corresponding blocks of the
pretrained rotary transformation as detailed in \ref{app:rope_details}. 

Let $R_t
    =
    \operatorname{diag}(R_{1,t},\ldots,R_{M,t})
    \in\mathbb{R}^{d_k\times d_k}$
be the pretrained RoPE operator at position $t$, where each
$R_{m,t}\in\mathbb{R}^{2\times2}$ acts on one rotary pair. For a retained
pair set $\mathcal S\subseteq\{1,\ldots,M\}$, define $R_t^{(\mathcal S)}
    =
    \operatorname{diag}(R_{m,t}:m\in\mathcal S)
    \in\mathbb{R}^{r_k\times r_k}.
$
Thus, $R_t^{(\mathcal S)}$ is the restriction of the original pretrained
operator to the retained pairs. We say that a reduction
$P_{QK}\in\mathbb{R}^{d_k\times r_k}$ is RoPE-compatible if
\begin{equation}
    R_tP_{QK}
    =
    P_{QK}R_t^{(\mathcal S)}
    \qquad\text{for all }t.
    \label{eq:rope_compatibility}
\end{equation}
This condition ensures that applying the reduction before RoPE is equivalent to applying full-dimensional RoPE and then restricting to the retained coordinates. Consequently, query/key projection, rotary encoding, and attention can all be performed directly in dimension \(r_k\). A simple admissible family of $P_{QK}$ is obtained by retaining complete RoPE pairs. The following theorem shows that this family exactly characterizes RoPE-compatible coordinate-selection reductions.


\begin{theorem}[RoPE-compatible coordinate selection]
\label{thm:rope_compatibility}
Suppose for every rotary pair $m$, there exists a position $t$ such that $\sin \phi_{m,t} \neq 0$. A coordinate-selection matrix $P_{QK}\in\mathbb{R}^{d_k\times r_k}$
satisfies Eq.~\ref{eq:rope_compatibility} for all $t$ if and only if
it selects exactly the coordinates belonging to a union of complete
rotary pairs indexed by $\mathcal S\subseteq\{1,\ldots,M\}$,
with the selected coordinates ordered consistently with
$R_t^{(\mathcal S)}$.
\end{theorem}

The proof is provided in ~\ref{app:proofs}. The remaining question is which admissible pairs should be retained. Section \ref{sec:pair_selection} introduces a calibration criterion for ranking the admissible pairs. 

\subsection{Calibration-Guided Selection Within the RoPE-Compatible Family}
\label{sec:pair_selection}

Section \ref{sec:rope_qk} characterizes a family of structurally valid native reductions by restricting selection to complete RoPE pairs. What remains is no longer a compatibility question, but a ranking question over this admissible pairs to determine what pairs should be retained. Because the approximation target in Section \ref{sec:amm_motivation} depends on both queries and keys, we rank candidate pairs using a joint calibration statistic over the two operands.

Let $\mathcal{P}_{\mathrm{rope}}
    =
    \{(i_m,j_m)\}_{m=1}^{M}$
denote the RoPE pairs of an attention head, and let $Q_m = Q_{:,\{i_m,j_m\}},
    K_m = K_{:,\{i_m,j_m\}}$.
We assign pair $m$ the joint query-key energy
\begin{equation}
    E_m
    =
    \|Q_m\|_F^2
    +
    \|K_m\|_F^2.
    \label{eq:pair_energy}
\end{equation}
For target dimension $r_k$, we retain the $r_k/2$ pairs with the
largest energies, $\mathcal{S}
    =
    \operatorname{TopK}_{m\in[M]}
    \left(E_m,\frac{r_k}{2}\right)$ 
and construct $P_{QK}
    =
    [e_i]_{i\in\mathcal{I}_{QK}}$ where $
    \mathcal{I}_{QK}
    =
    \bigcup_{m\in\mathcal{S}}\{i_m,j_m\}.
$

The following result provides a calibration-dependent justification for
this selection rule.

\begin{proposition}[Pair-energy approximation bound]
\label{prop:pair_energy_bound}
Let
$\widetilde Q=\operatorname{RoPE}(Q)$ and
$\widetilde K=\operatorname{RoPE}(K)$, and let
$\widetilde Q_m,\widetilde K_m$ denote the coordinates associated with
RoPE pair $m$. For any retained pair set
$\mathcal{S}$,
\begin{equation}
\left\|
    \widetilde Q\widetilde K^\top
    -
    \sum_{m\in\mathcal{S}}
    \widetilde Q_m\widetilde K_m^\top
\right\|_F
\leq 
\sum_{m\notin\mathcal{S}} \left\|\widetilde Q_m\right\|_F \left\|\widetilde K_m\right\|_F
\leq
\frac{1}{2}
\sum_{m\notin\mathcal{S}} E_m.
\label{eq:pair_energy_bound}
\end{equation}

\end{proposition}

Consequently, among all subsets containing the same number of complete
RoPE pairs, selecting the pairs with the largest $E_m$ minimizes the separable upper bound in Eq.~\ref{eq:pair_energy_bound}.  Proof of this proposition is given in \ref{app:proofs}. Proposition~\ref{prop:pair_energy_bound} does not imply that the selected
pair set minimizes the exact post-RoPE query-key product error. Rather, it shows
that the proposed top-energy rule exactly minimizes a worst-case, separable calibration-dependent upper bound in the right-hand side of 
Eq.~\ref{eq:pair_energy_bound}. We use this result as a tractable forward-only criterion for choosing within the RoPE-compatible family. 
 We evaluate the robustness of this
calibration-dependent selection empirically in
Section~\ref{sec:ablation1}. 


\paragraph{Value compression.}
Values are not subject to RoPE, so we follow the standard low-rank activation projection used in prior head-dimension compression methods~\citep{saxena2024eigen}. Given calibration activations
$V\in\mathbb{R}^{S_{\mathrm{cal}}\times d_v}$, we retain the top-$r_v$ right singular vectors
$P_V\in\mathbb{R}^{d_v\times r_v}$ and compute $
V' = VP_V.
$
\begin{algorithm}[t]
\caption{RACE: Post-Training RoPE-Aligned Attention Compression}
\label{alg:method}
\begin{algorithmic}[1]
\REQUIRE Pretrained model $\mathcal M$, calibration set $\mathcal D_{\mathrm{cal}}$,
target dimensions $r_k,r_v$
\FOR{each attention layer}
    \STATE Collect post-RoPE $Q,K,V$ activations on $\mathcal D_{\mathrm{cal}}$
    \STATE For RoPE pairs, compute $E_m
    =
    \|Q_m\|_F^2
    +
    \|K_m\|_F^2$ 
    \STATE Construct $P_{QK}$ that selects the top-$r_k/2$ pairs based on $E_m$
    \STATE Fold in $P_{QK}$. $W_Q^{\mathrm{fold}} = W_Q P_{QK};
    W_K^{\mathrm{fold}} = W_K P_{QK}$
    \STATE Compute the top-$r_v$ right singular vectors of V to get $P_V$
  \STATE Fold in $P_{V}$. $W_V^{\mathrm{fold}} = W_V P_{V};
    W_O^{\mathrm{fold}} = P_{V}^\top W_O$
\ENDFOR
\STATE Optional: recover performance of $\mathcal M'$ from $\mathcal M$ 
\RETURN Native attention-compressed model $\mathcal M'$

\end{algorithmic}
\end{algorithm}

\subsection{Folding into Native Reduced-Dimensional Attention}
\label{sec:folding}

Having obtained the RoPE-compatible query-key selector $P_{QK}$ and the value subspace $P_V$, we fold both transformations directly into
the pretrained attention projections. For queries and keys,
\begin{equation}
Q' = XW_Q^{\mathrm{fold}},
    K' = XW_K^{\mathrm{fold}} \qquad
\text{where,}\qquad
    W_Q^{\mathrm{fold}} = W_Q P_{QK},
    W_K^{\mathrm{fold}} = W_K P_{QK}.
\end{equation}
$Q'$ and $K'$ are
$r_k$-dimensional queries and keys to which the restricted RoPE operator from
Section~\ref{sec:rope_qk} is applied. Hence, neither the discarded
query/key coordinates nor a full-dimensional RoPE representation is
materialized during inference.

For values, we similarly fold the calibrated subspace into the value
and output projections - $W_V^{\mathrm{fold}} = W_V P_V,
    W_O^{\mathrm{fold}} = P_V^\top W_O$ and $V' = XW_V^{\mathrm{fold}}, Y' = (A'V')W_O^{\mathrm{fold}}$
where $A'$ is computed entirely in the reduced dimension
$r_k$ using the retained RoPE pairs.

\begin{figure}[t]
    \centering

    \begin{subfigure}[t]{0.32\textwidth}
        \centering
        \includegraphics[width=\linewidth]{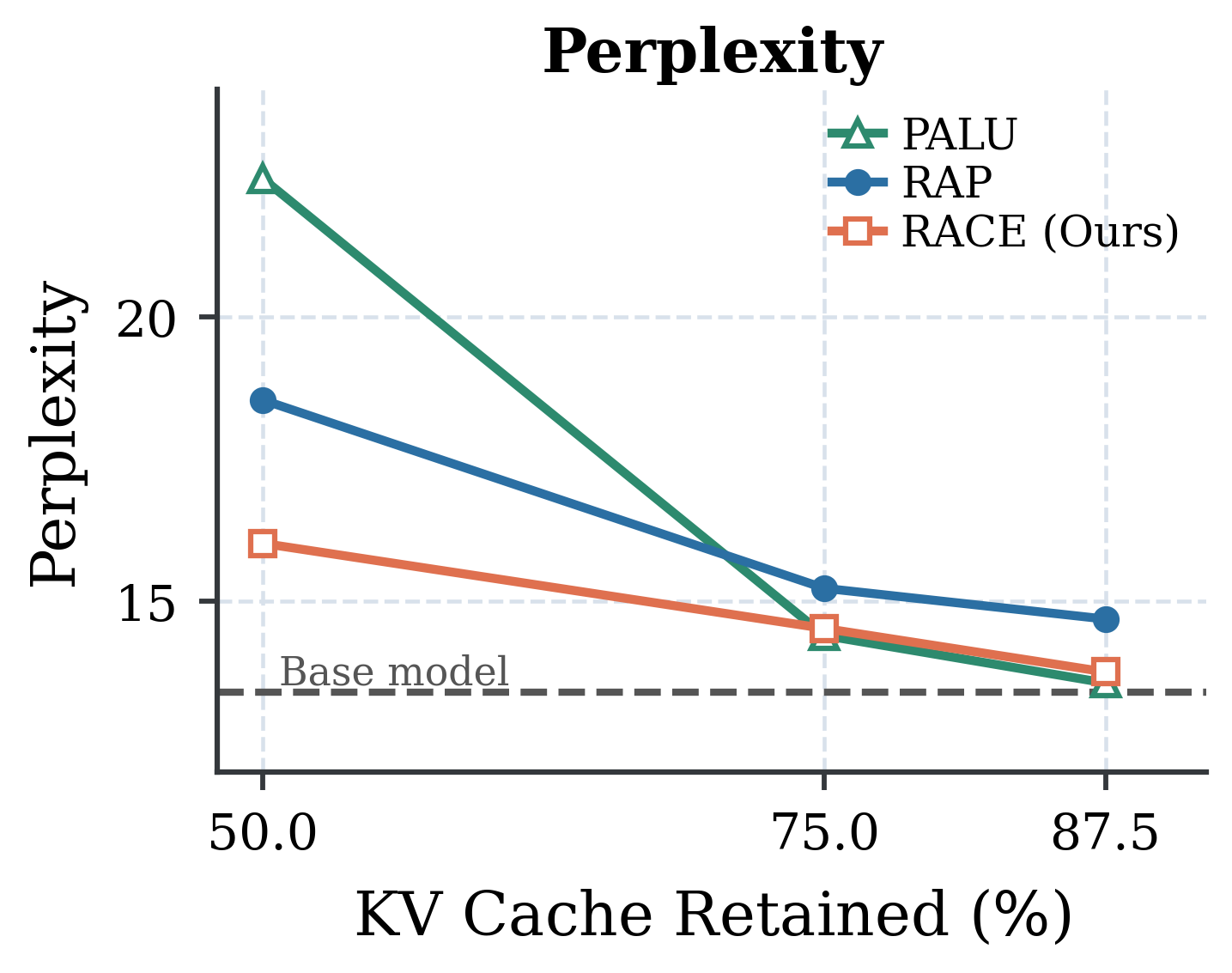}
        \caption{Llama-3.2-3B-Instruct Perplexity}
        \label{fig:result_a}
    \end{subfigure}
    \hfill
    \begin{subfigure}[t]{0.32\textwidth}
        \centering
        \includegraphics[width=\linewidth]{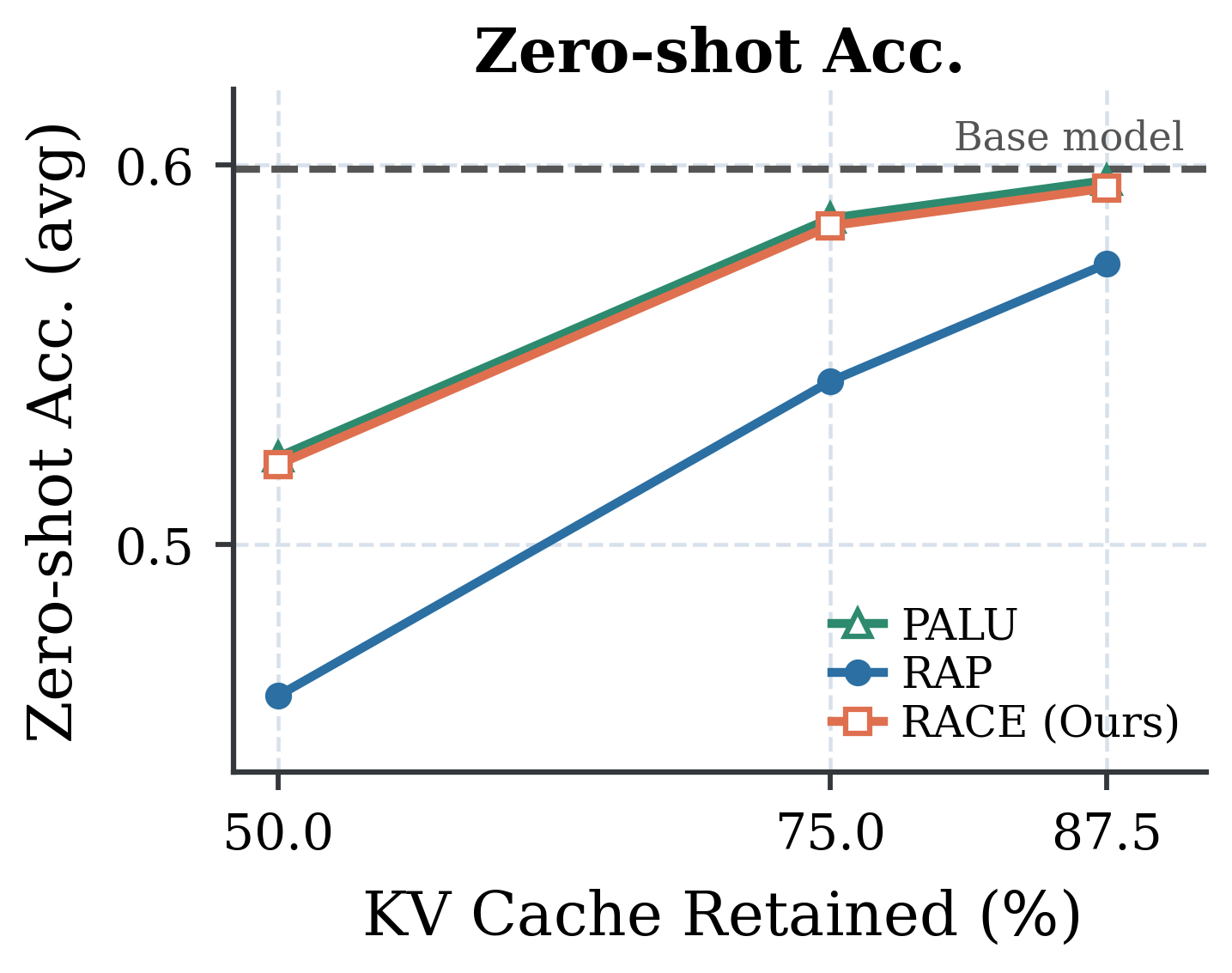}
        \caption{Llama-3.2-3B-Instruct Average zero-shot accuracy}
        \label{fig:result_b}
    \end{subfigure}
    \hfill
    \begin{subfigure}[t]{0.32\textwidth}
        \centering
        \includegraphics[width=\linewidth]{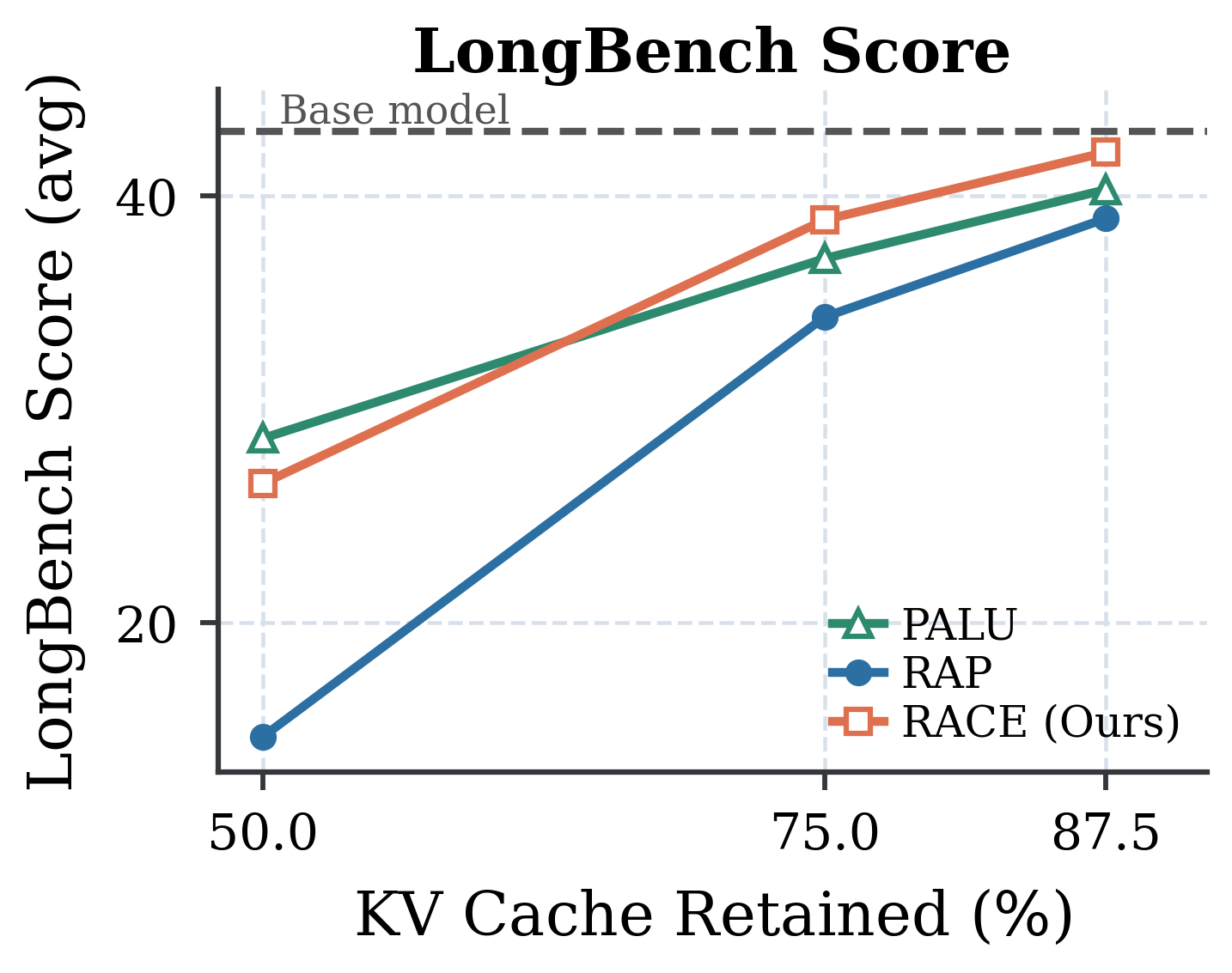}
        \caption{Llama-3.2-3B-Instruct LongBench Score}
        \label{fig:result_c}
    \end{subfigure} \\
    \begin{subfigure}[t]{0.32\textwidth}
        \centering
        \includegraphics[width=\linewidth]{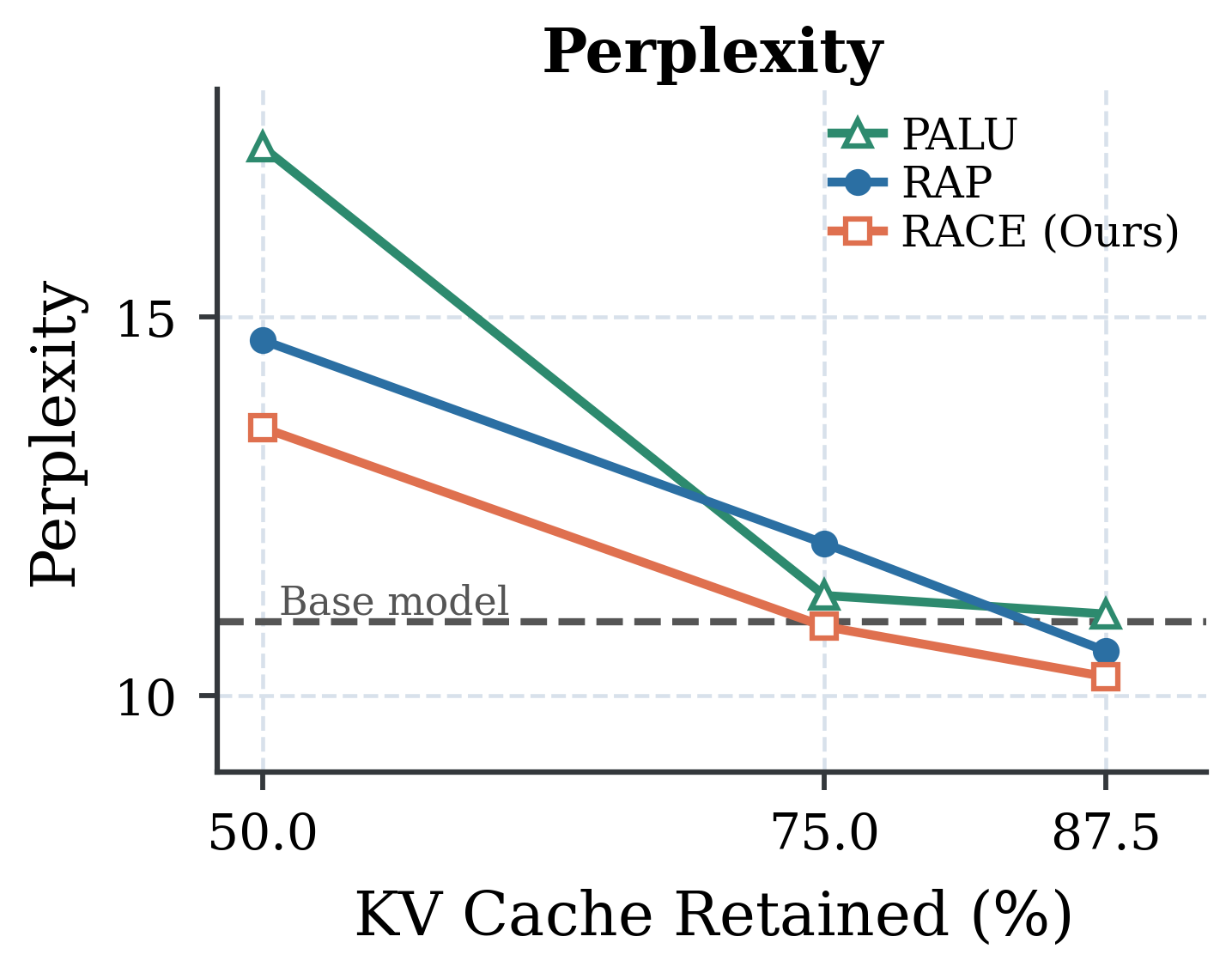}
        \caption{Qwen2.5-7B Perplexity}
        \label{fig:result_a}
    \end{subfigure}
    \hfill
    \begin{subfigure}[t]{0.32\textwidth}
        \centering
        \includegraphics[width=\linewidth]{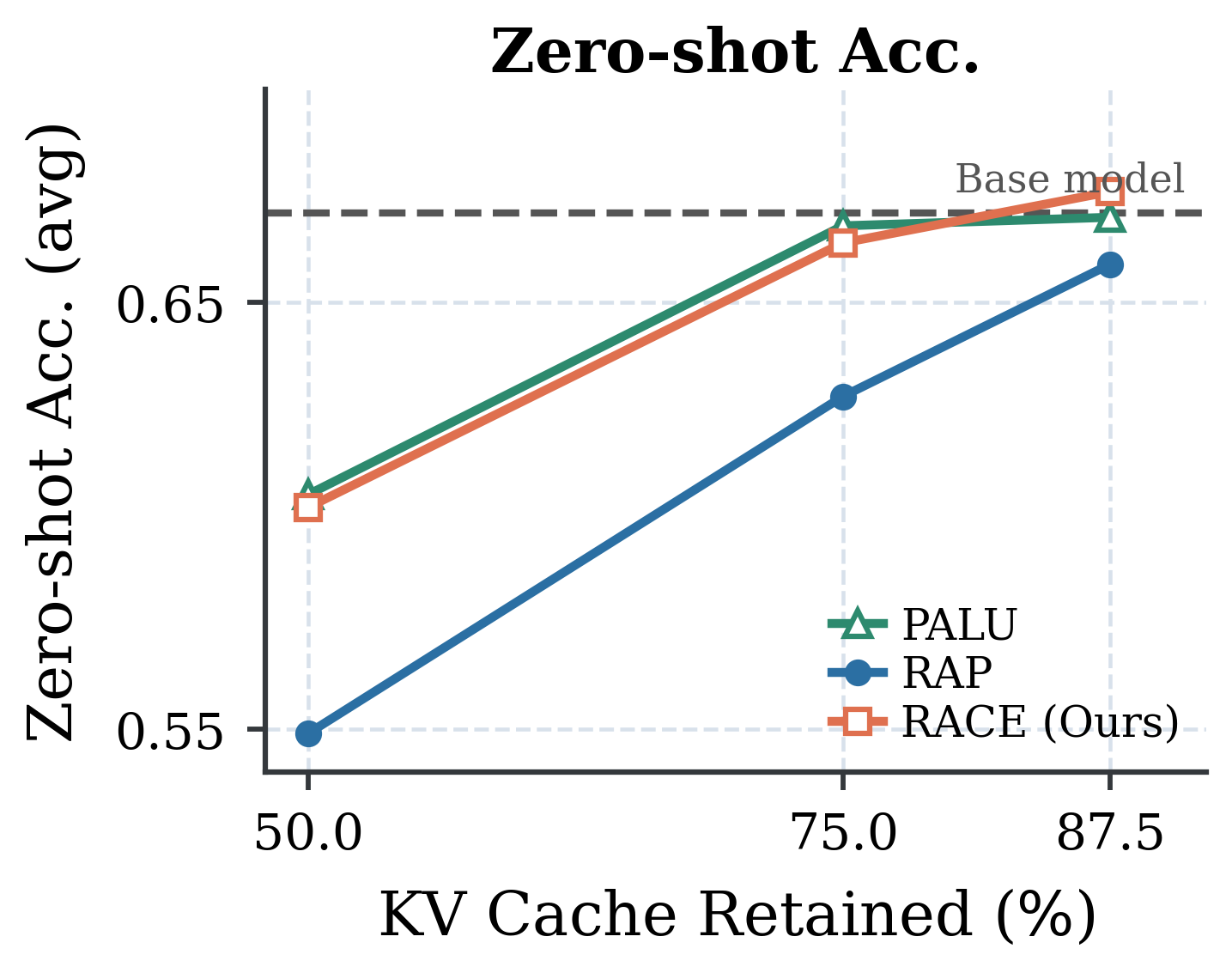}
        \caption{Qwen2.5-7B Average zero-shot accuracy}
        \label{fig:result_b}
    \end{subfigure}
    \hfill
    \begin{subfigure}[t]{0.32\textwidth}
        \centering
        \includegraphics[width=\linewidth]{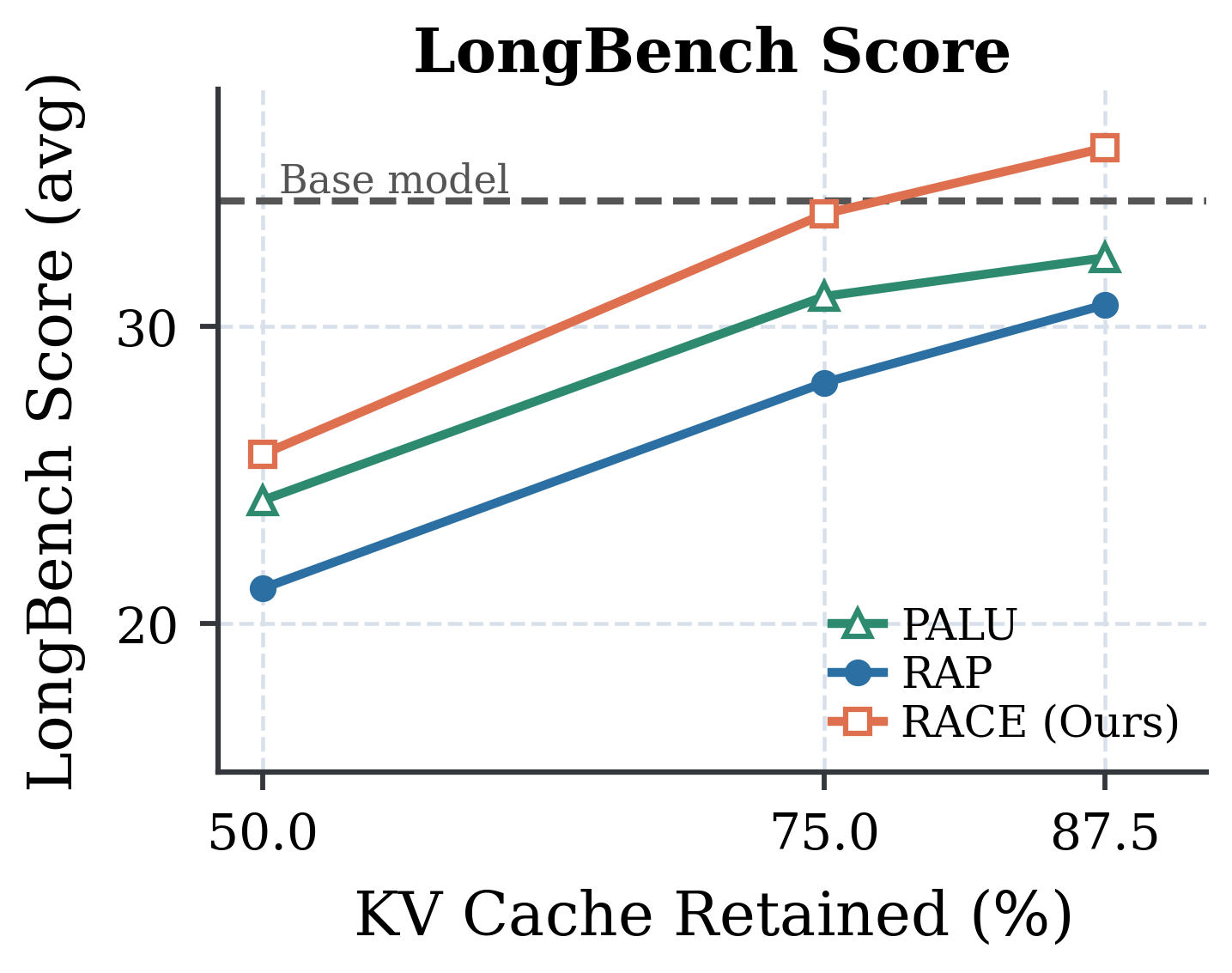}
        \caption{Qwen2.5-7B LongBench Score}
        \label{fig:result_c}
    \end{subfigure}

    \caption{\textbf{Quality across KV-cache retention ratios on
    Llama-3.2-3B-Instruct (Top) and Qwen2.5-7B (Bottom).} WikiText perplexity (lower is better), and mean accuracy across 15 zero-shot tasks and mean LongBench score on 16 english tasks  (higher is better). }
    \label{fig:KV_vary}
    \vspace{-0.6cm} 
\end{figure}

This folding converts the original pretrained attention layer into a
native architecture with query/key dimension $r_k$ and value dimension
$r_v$. The resulting model stores only reduced-dimensional keys and
values and performs RoPE, query-key comparison, value aggregation, and
output projection without reconstructing the discarded dimensions. Consequently, the reduction decreases KV-cache storage, attention
arithmetic, and attention projection parameters simultaneously.

\paragraph{Recovery.} For aggressive compression, we recommend a short one-time post-fold recovery stage using KL distillation from the uncompressed model (Fig. \ref{fig:FT_vary}). We update the reduced attention parameters $W_Q^{\mathrm{fold}}, W_K^{\mathrm{fold}}, W_V^{\mathrm{fold}}, W_O^{\mathrm{fold}}$ while keeping the rest of the model frozen. Recovery is performed directly on the natively compressed checkpoint using 1M FineWeb-Edu tokens. Its effect is modest at upto 75\% KV retention but substantial at 50\%, where it restores much of the performance. Therefore, while moderate compression requires little or no adaptation; for 50\% retention, we recommend a short one-time recovery stage for RACE and RAP.

\section{Experiments and Results}
\label{sec:results}

\paragraph{Implementation Details.}
\label{sec:implementation_details}
We implement RACE in PyTorch and Hugging Face Transformers. Calibration
uses 32,768 FineWeb-Edu tokens, with statistics and SVDs computed in
FP32 during calibration and inference is performed in BF16. We respect each architecture's
physical rotary-pair layout and retain the original frequency indices
and RoPE scaling rules; selected pairs are not re-indexed as consecutive
frequencies. Compressed checkpoints 
resize the attention projections and cache dimensions while maintaining the $\frac{1}{\sqrt{d_k}}$ scaling, and run using
standard layers and PyTorch attention operations. Thus,
RACE requires no specialized CUDA/Triton kernel. While all the results in Sec.\ref{sec:results} were reported with SDPA attention backend, Tab.\ref{tab:kernel_analysis} shows systems results with eager and flash attention backends. Although different kernels optimize for wall-clock time, the number of FLOPs remian the same. Although optional, all the results of RACE and the concurrent work, RAP involve recovery with
KL-distillation over approximately 1M FineWeb-Edu tokens as detailed in \ref{app:recovery}.

\subsection{Quality at Matched KV-Cache Budgets}
\label{sec:quality_cache_tradeoff}

We evaluate Qwen2.5-7B,
Llama-3.2-3B-Instruct, and  Granite-4.1-3B against RoPE-compatible implementations of
PALU-GLRD~\citep{chang2025palu}, EigenAttention\citep{saxena2024eigen}, Thin Keys\citep{thinkeys} and concurrent work,
RAP~\citep{rap}. The full results are shown in Table \ref{tab:ZS_full} in the Appendix. Fig.\ref{fig:KV_vary} shows the same on Qwen2.5-7B and
Llama-3.2-3B-Instruct for RACE, the concurrent work and structured pruning baseline: RAP, and the reconstruction-based baseline: PALU. We finetune reduced attention weights of both RAP and RACE for 1M tokens of FineWeb-Edu using KL distillation from the respective uncompressed model for fair comparison. We report
WikiText perplexity, mean accuracy across 15 zero-shot tasks (Tab.\ref{tab:ZS_full}), and LongBench score average over 16 English tasks (Tab.\ref{tab:lb_full}). At 75\% KV-cache retention, RACE preserves 98.5\%, 96.5\%, and 97.5\%
of the uncompressed zero-shot accuracy on Qwen, Granite, and Llama,
respectively. Across compression
levels, on Llama and Qwen models (Figure~\ref{fig:KV_vary}), RACE closely matches PALU on
zero-shot accuracy and LongBench without runtime reconstruction.  RACE consistently outperforms RAP. RACE also achieves the lowest perplexity at 50\%
and 75\% retention.

\subsection{Measured Systems Benefits}

\begin{figure}[t]
    \centering

    \begin{subfigure}[t]{0.32\textwidth}
        \centering
        \includegraphics[width=\linewidth]{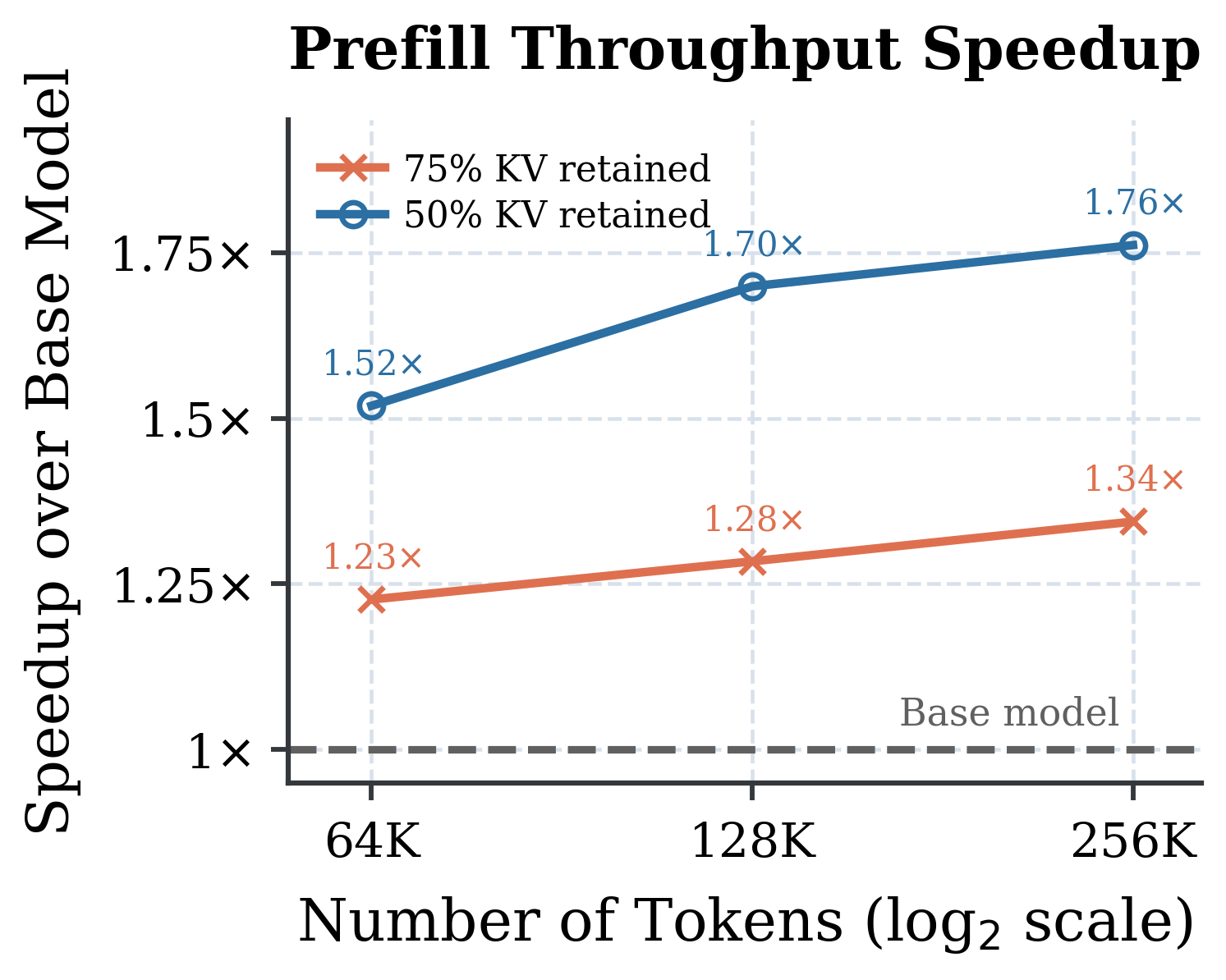}
        \caption{Prefill throughput speedup}
        \label{fig:prefill_speedup}
    \end{subfigure}
    \hfill
    \begin{subfigure}[t]{0.32\textwidth}
        \centering
        \includegraphics[width=\linewidth]{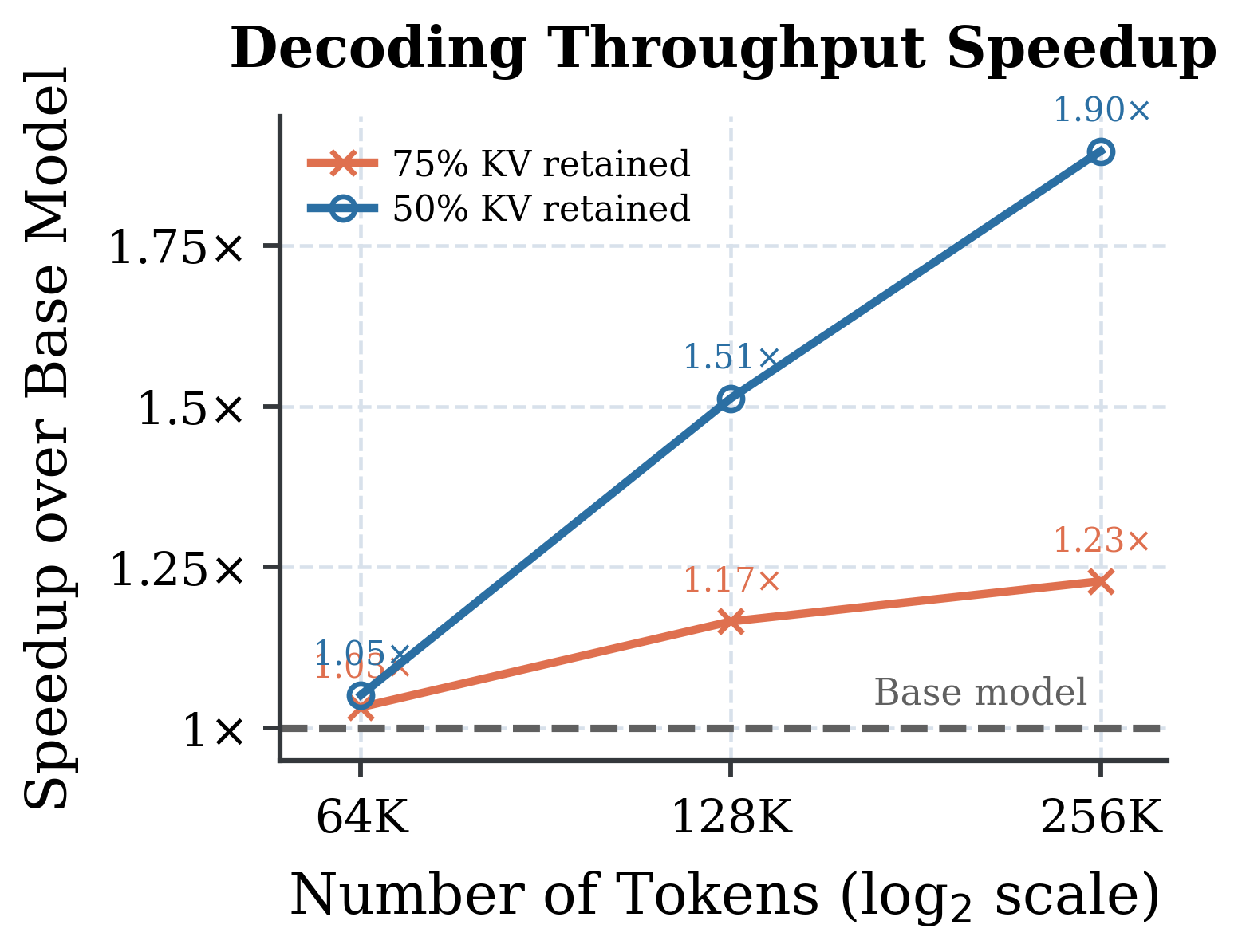}
        \caption{Decode throughput speedup}
        \label{fig:decode_speedup}
    \end{subfigure}
    \hfill
    \begin{subfigure}[t]{0.32\textwidth}
        \centering
        \includegraphics[width=\linewidth]{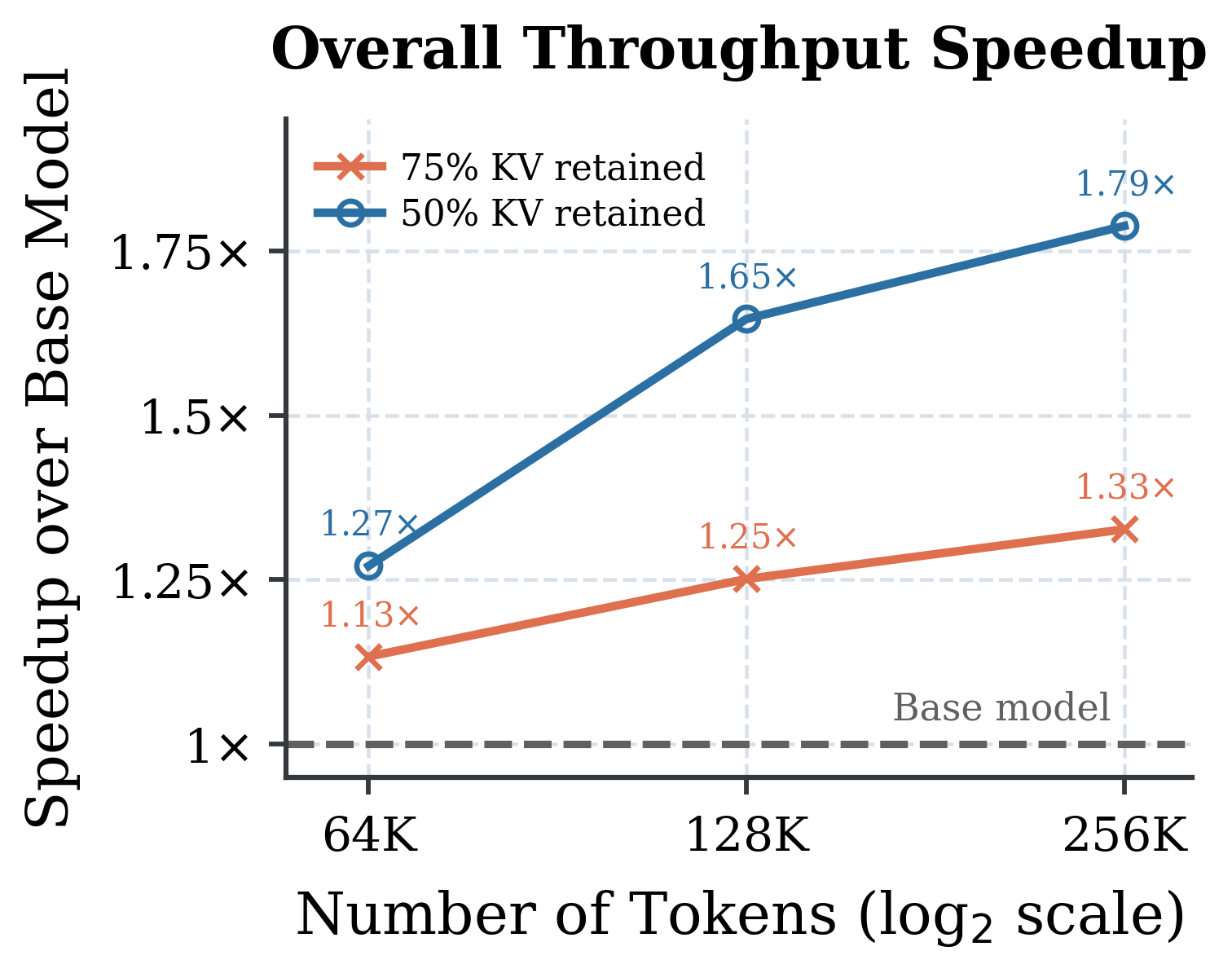}
        \caption{Overall throughput speedup}
        \label{fig:overall_speedup}
    \end{subfigure}

    \vspace{0.6em}

    \begin{subfigure}[t]{0.32\textwidth}
        \centering
        \includegraphics[width=\linewidth]{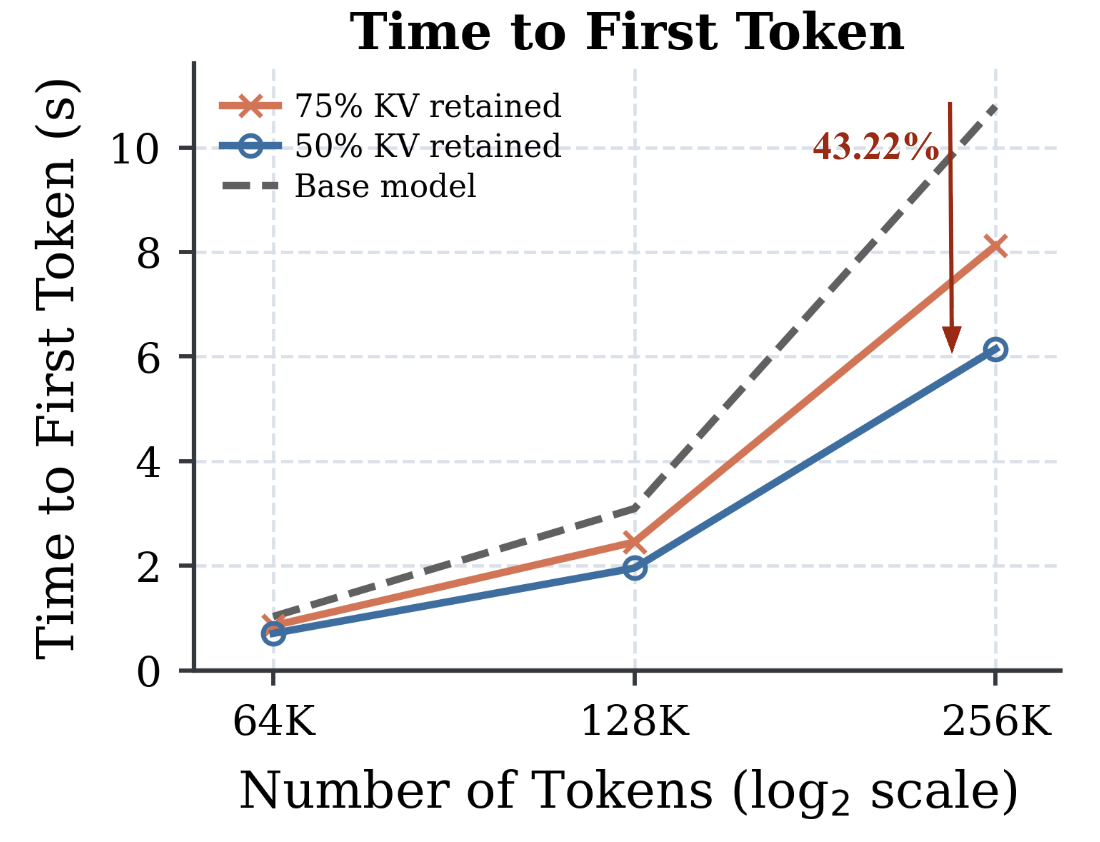}
        \caption{Time to first token}
        \label{fig:ttft}
    \end{subfigure}
    \hfill
    \begin{subfigure}[t]{0.32\textwidth}
        \centering
        \includegraphics[width=\linewidth]{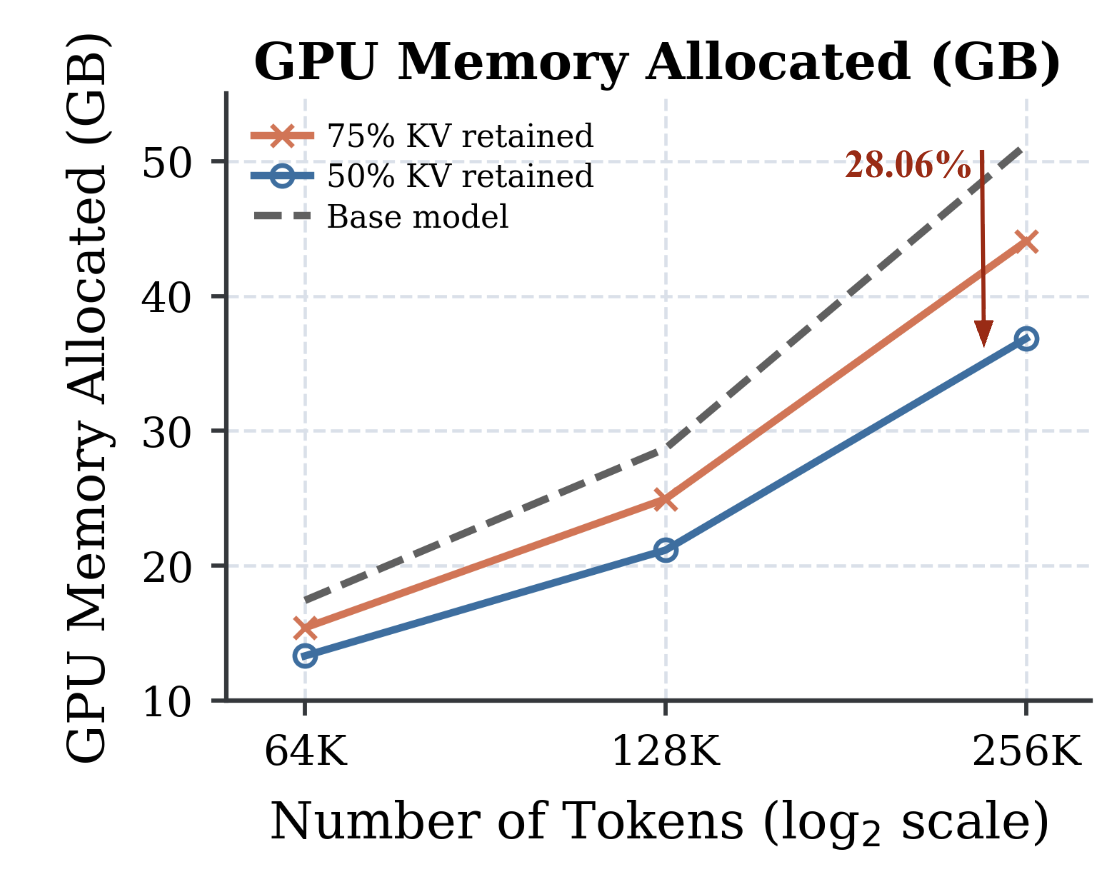}
        \caption{GPU allocated memory}
        \label{fig:gpu_memory}
    \end{subfigure}
    \hfill
    \begin{subfigure}[t]{0.32\textwidth}
        \centering
        \includegraphics[width=\linewidth]{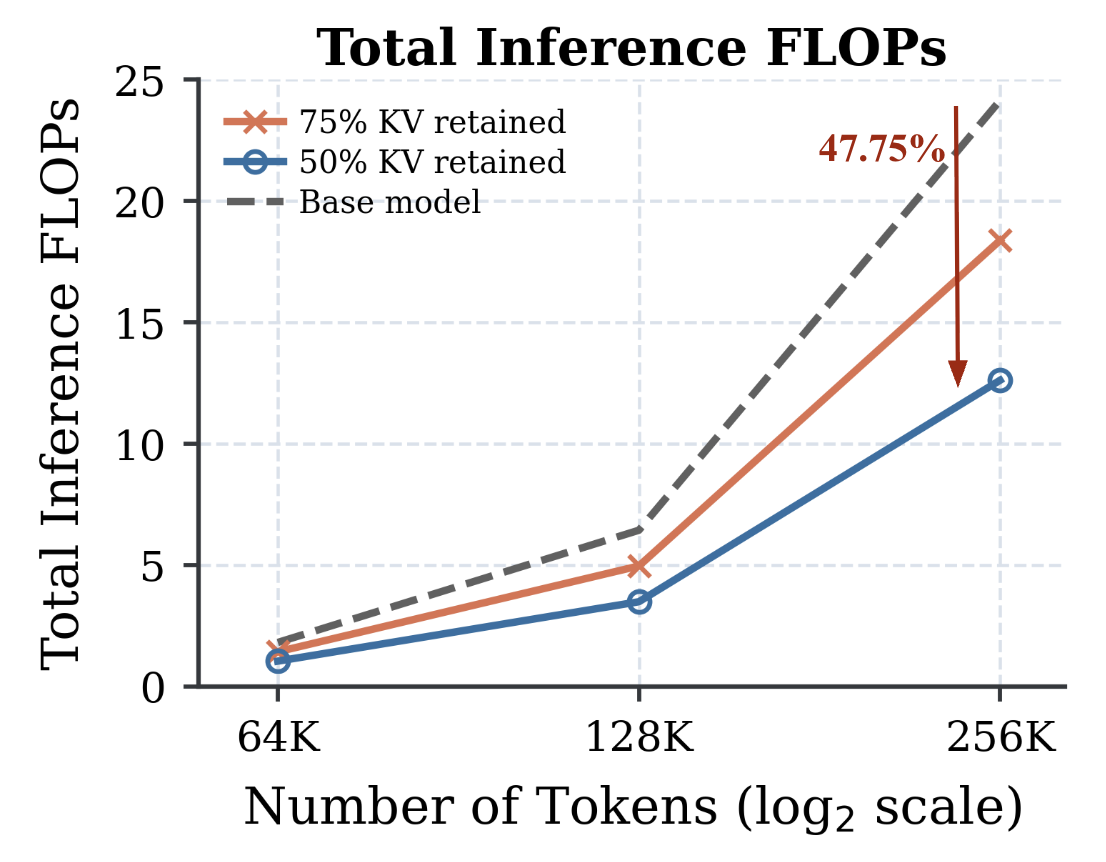}
        \caption{Total inference FLOPs}
        \label{fig:inference_flops}
    \end{subfigure}

    \caption{
        \textbf{Long-context inference efficiency on
        Llama-3.2-3B-Instruct.}
        The top row reports prefill, decoding, and overall throughput
        speedups over the uncompressed model (higher is better). The bottom row reports time to first token, allocated GPU memory, and total inference
        FLOPs (lower is better). Results are shown at 75\% and 50\% KV-cache retention
        for prompt lengths of 64K, 128K, and 256K tokens, with the
        output length fixed at 128 tokens. These results are average over 10 runs.
    }
    \label{fig:long_context_efficiency}
    \vspace{-0.6cm}
\end{figure}

Reducing KV-cache storage does not automatically reduce inference cost. As discussed in Section \ref{sec:intro}, this requires native reduced attention. Such native reduced attention methods like \emph{RACE eliminates reconstruction at the architectural level, exposing standard reduced-dimensional attention operations that can be supported by general-purpose backends}.  Figure~\ref{fig:long_context_efficiency} shows that these measured gains become larger with prompt length as attention accounts for a greater fraction of inference cost. At 256K tokens, 50\% retention achieves $1.76\times$, $1.90\times$, and $1.79\times$ prefill, decoding, and overall throughput speedups, respectively. It also reduces TTFT by 43.22\%, allocated GPU memory by 28.06\%, and total inference FLOPs by 47.75\% relative to the uncompressed model. The end-to-end memory reduction is smaller than the 50\% KV-cache reduction because total allocated memory also includes model parameters, activations, and other components unaffected by the proposed method. Figure.\ref{fig:mems} shows the memory reduction with compression on total model parameters, attention parameters and KV cache. While parameter memory is unaffected by sequence length, KV Cache memory grows with it.

\subsection{Compatibility with Token-Dimension Compression and Quantization}
\label{sec:complementary}
RACE reduces the KV-cache feature dimension or head dimension and can therefore be combined with methods that compress its token dimension or numerical precision. We apply 75\% RACE retention before H$2$O and SnapKV token compression and KIVI quantization, each with a 50\% cache footprint, yielding 37.5\% combined retention. Figure~\ref{fig:compression_complementary} shows that applying RACE to an already compressed model causes greater degradation than applying it to the uncompressed model, indicating that the effects of the two compression stages are not independent. Nevertheless, the combinations provide an additional 25\% reduction in cache size and feature-dependent attention FLOPs as well as attention parameters. The degradation is smaller with KIVI, particularly in perplexity, suggesting that RACE is especially complementary to quantization. Overall, RACE can augment existing KV-cache methods with further memory, bandwidth, and, for token compression, decoding-FLOP reductions.

\begin{figure}[t]
    \centering
    \begin{minipage}[t]{0.62\textwidth}
        \centering
        \includegraphics[width=0.98\linewidth]{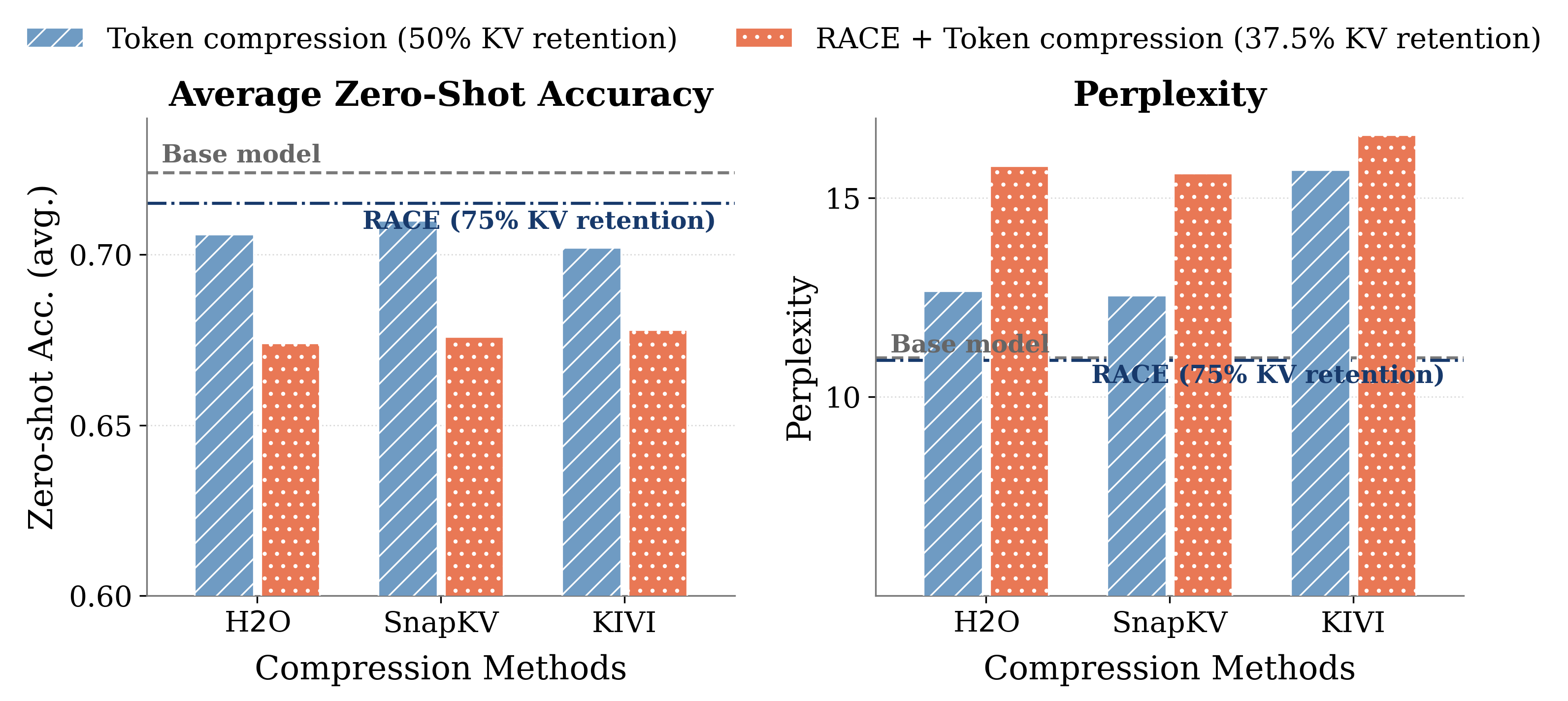}
        \vspace{-0.3cm}
    \caption{
    \textbf{Composing RACE with token-level compression and quantization.}
Combining RACE with H$_2$O, SnapKV, and KIVI further reduces total KV-cache
retention. We report average zero-shot accuracy (left; higher is better) over 5 tasks as described in \ref{tab:complementary}
 and WikiText perplexity (right; lower is
better).
    }
        \label{fig:compression_complementary}
    \end{minipage}
    \hfill
    \begin{minipage}[t]{0.35\textwidth}
        \centering
        \includegraphics[width=0.9\linewidth]{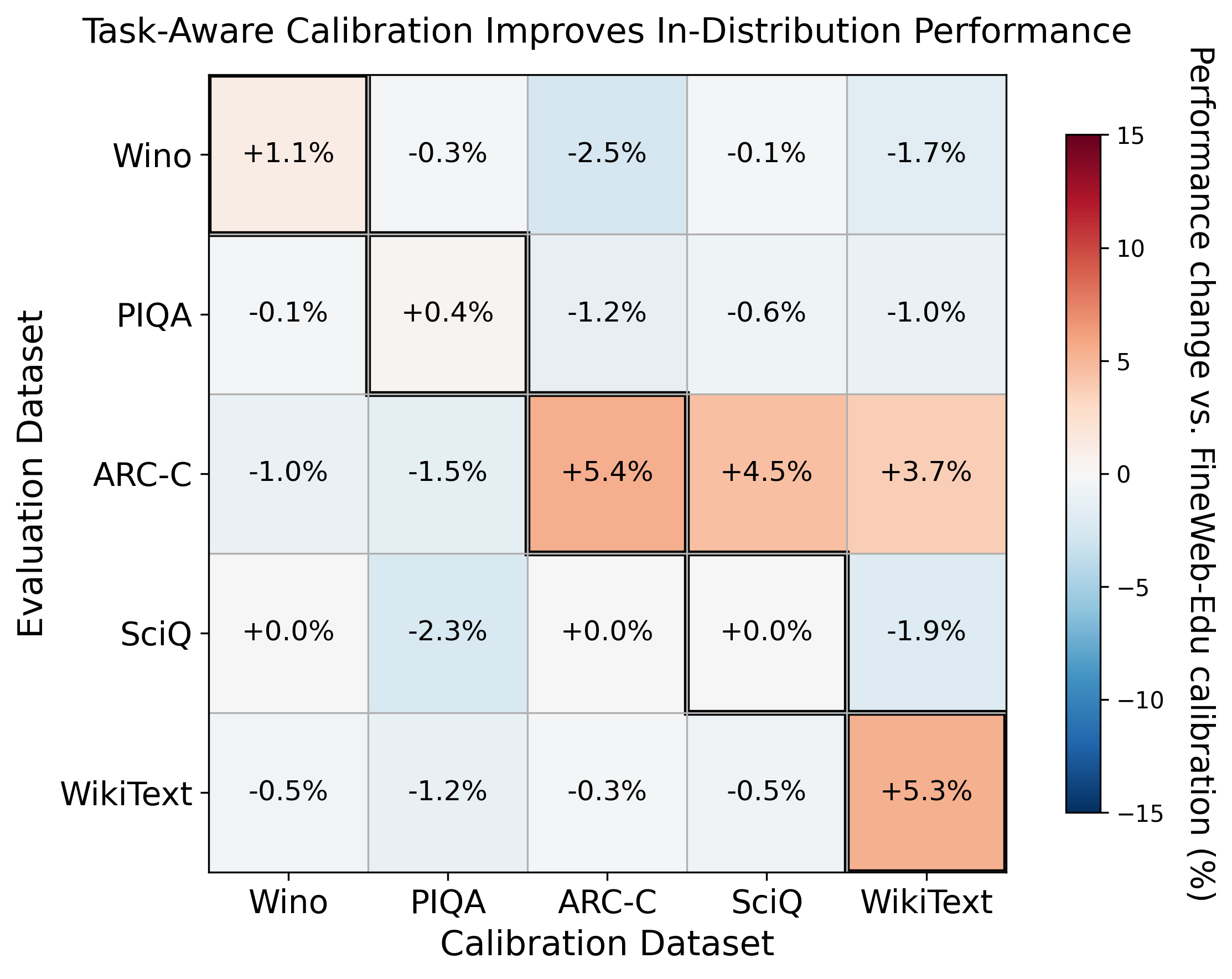}
        \vspace{-0.3cm}
    \caption{
    \textbf{Effect of calibration data on compressed-model quality.}
    Values show change on evaluation with the respective calibration data relative to FineWeb-Edu calibration. 
    }
    \label{fig:calibration_heatmap}
    \end{minipage}
    \vspace{-0.5cm}
\end{figure}

\subsection{Effect of Calibration Data}

Our selection criterion explained in Section \ref{sec:pair_selection} is derived from a calibration-dependent upper bound on the query-key approximation error, so the retained subspace is expected to reflect the calibration distribution. Figure \ref{fig:calibration_heatmap} evaluates this dependence by varying the calibration dataset and reporting performance relative to FineWeb-Edu calibration. Note that for WikiText, the perplexity change is sign-inverted so that positive values consistently denote improvement. Using the same dataset as the downstream task generally gives the strongest results - upto +5.4\% on the evaluated datasets with respect to calibration on FineWen-Edu, and calibration on a related task family may also improve performance as observed in evaluation on arc-c with calibration on SciQ. However, the differences are modest $(<6\%)$ for most tasks. Therefore, \emph{FineWeb-Edu remains a strong generic calibration source, especially when in-distribution task-specific samples are unavailable or the downstream task family is unknown.}

\subsection{Ablation: Selection Criteria}
\label{sec:ablation1}






        

\begin{figure*}[t]
    \centering

    \begin{subfigure}[t]{0.35\textwidth}
        \centering
        \small
        \vspace{-2.5cm}
        \resizebox{\linewidth}{!}{

\begin{tabular}{ccccc}
\hline
\rowcolor[HTML]{EFEFEF} 
                                                                              & \textbf{\begin{tabular}[c]{@{}c@{}}Energy\\ (Ours)\end{tabular}} & \textbf{Magnitude}         & \textbf{\begin{tabular}[c]{@{}c@{}}Attn.\\ Product\end{tabular}} & \textbf{\begin{tabular}[c]{@{}c@{}}Grad.\\ Fisher\end{tabular}} \\ \hline
\textbf{Backprop-free}                                                        & \ding{51}                                       & \ding{51} & \ding{51}                                       & \ding{55}                                      \\
\rowcolor[HTML]{EFEFEF} 
\textbf{Task-loss-free}                                                       & \ding{51}                                       & \ding{51} & \ding{51}                                       & \ding{55}                                      \\
\textbf{\begin{tabular}[c]{@{}c@{}}Explicit \\ $O(S^2)$ scoring\end{tabular}} & \ding{55}                                       & \ding{55} & \ding{51}                                       & \ding{55}                                      \\
\rowcolor[HTML]{EFEFEF} 
\textbf{Time (sec)}    & 2.717                                                            & 2.925                      & 11.616                                                           & 8.290                                                           \\
\textbf{\begin{tabular}[c]{@{}c@{}}Throughput\\  (tok./sec)\end{tabular}}     & 12061.6                                                          & 11202.4                    & 2821                                                             & 3952.7                                                          \\
\rowcolor[HTML]{EFEFEF} 
\textbf{Memory (GB)}          & 6.142                                                            & 6.142                      & 10.592                                                           & 9.218                                                           \\ \hline
\end{tabular}

        }
        \caption{Calibration cost of alternative data-dependent RoPE-pair scoring criteria.}
        \label{fig:score_cost}
    \end{subfigure}
    \hfill
    \begin{subfigure}[t]{0.31\textwidth}
        \centering
        \includegraphics[width=\linewidth]{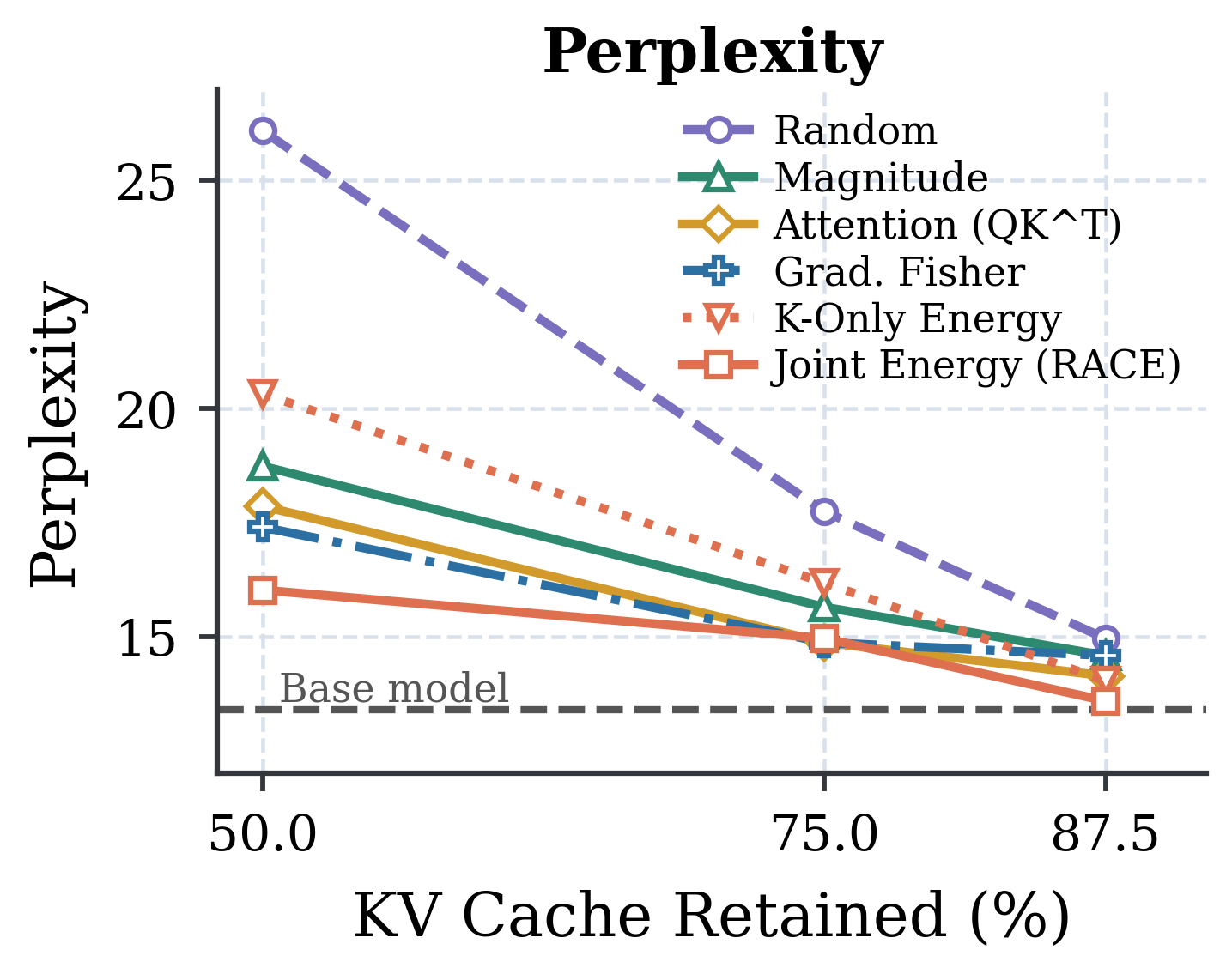}
        \caption{Perplexity under alternative pair-selection criteria.}
        \label{fig:diff_metrics_ppl}
    \end{subfigure}
    \hfill
    \begin{subfigure}[t]{0.31\textwidth}
        \centering
        \includegraphics[width=\linewidth]{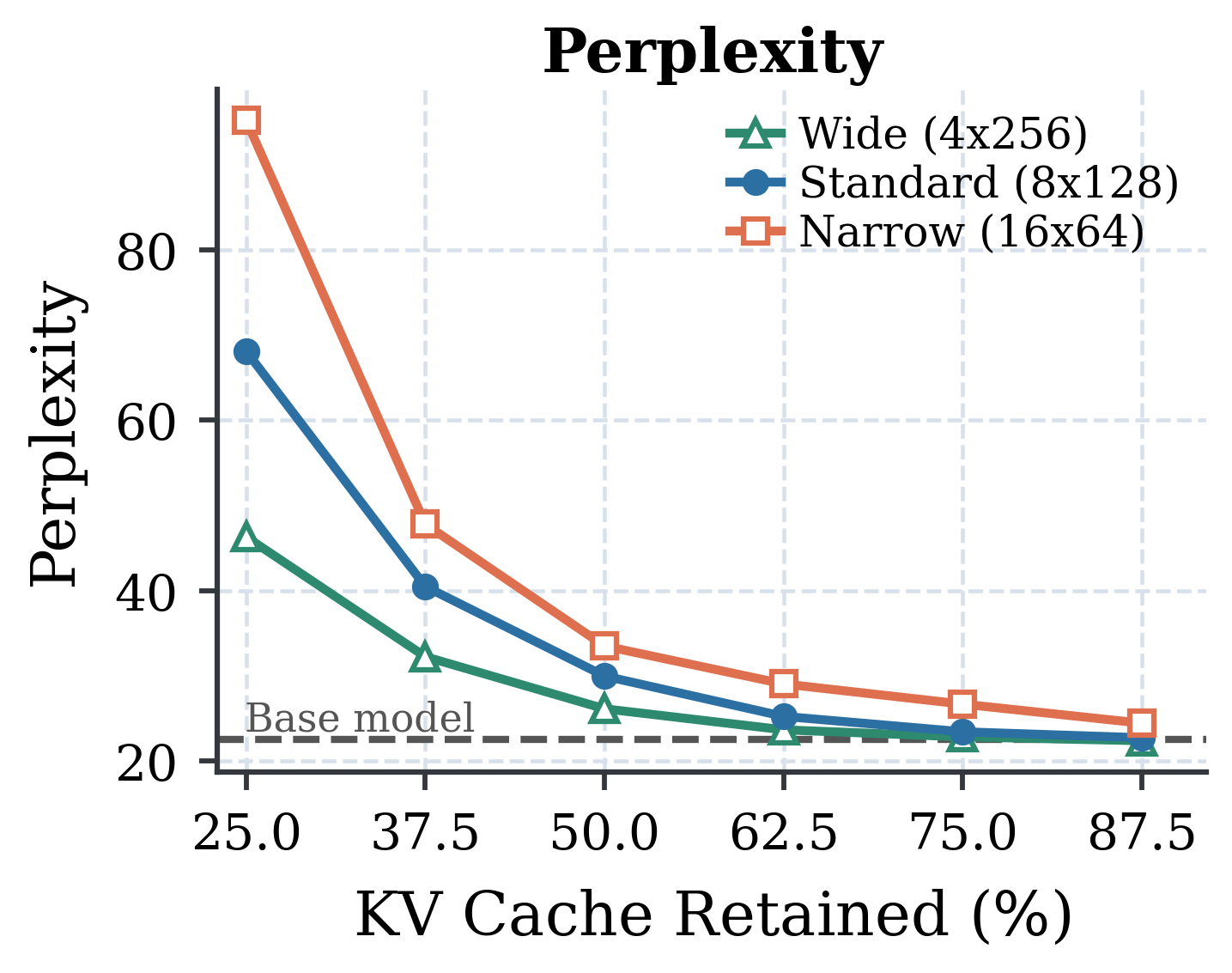}
        \caption{Compression behavior of wide, standard, \& narrow heads.}
        \label{fig:head_width}
    \end{subfigure}
\vspace{-0.2cm}
    \caption{
    \textbf{Analysis of RoPE-pair selection and compression structure.}
    \textbf{(a)} Computational requirements of RoPE-compatible scoring
    criteria. 
    \textbf{(b)} Perplexity under alternative pair-selection criteria.
    All methods retain complete RoPE pairs.
    \textbf{(c)} Effect of attention-head width at matched model size across
    compression ratios.
    }
    \label{fig:selection_and_structure}
    \vspace{-0.7cm}
\end{figure*}




The shared Q/K reduction and Rope-compatibility constraint define a family of native coordinate reductions; the remaining options are on how to rank candidate rotary pairs. We first compare RACE's joint Q/K
activation energy with otherwise identical Q-only and K-only selectors. As
shown in Fig.~\ref{fig:selection_and_structure}(b), K-only selection degrades
perplexity more strongly, while joint Q/K selection performs best overall. A similar result was observed for Q-only (Table \ref{tab:ablations}). This
is consistent with preservation of the query-key interaction: the average
normalized error
$\|QK^\top-Q_{\mathcal S}K_{\mathcal S}^\top\|_F/\|QK^\top\|_F$
is $0.1159$, $0.1120$, and $0.1051$ for Q-only, K-only, and joint Q/K
selection, respectively.

We further compare RACE with random, magnitude, attention-product, and
gradient-Fisher scoring. Random selection is data-oblivious; energy(RACE), magnitude and
attention-product use joint Q/K forward statistics; and gradient-Fisher uses
joint Q/K gradients from backward-pass on the calibration data. RACE achieves the lowest perplexity across the evaluated
retention ratios among these selectors. Table ~\ref{fig:selection_and_structure}(a) further compares calibration cost of methods that uses Rope-compatible shared joint Q/K statistics for selection.
Attention-product scoring explicitly incurs $O(S^2)$ computation and
gradient-Fisher requires backpropagation and a task loss. Accordingly,
athese methods incur substantially higher calibration
time and peak memory. Joint QK energy selection therefore offers a favorable
trade-off between perfromance and calibration cost.

\vspace{-0.2cm}
\subsection{Towards Models with Fewer and Wider Heads}
\label{sec:head_width}
If post-training attention compression is anticipated at deployment time,
larger head dimensions may be a useful architectural choice during
pretraining. We study three matched $\sim 245$ M-parameter GPT-NeoX-style
models pretrained from scratch on FineWeb-Edu for 2.5B tokens:
\emph{wide} ($4\times256$), \emph{standard} ($8\times128$), and
\emph{narrow} ($16\times64$). All three reach comparable uncompressed
WikiText validation perplexity. Figure~\ref{fig:selection_and_structure}(c) shows that wider-head models
retain perplexity more effectively under aggressive RACE compression. At each
retention ratio, the models have approximately matched aggregate retained
attention dimension, KV-cache memory, attention FLOPs, and compressed
attention parameter count; the primary difference is how this capacity is
partitioned across heads. Wider heads expose more RoPE pairs from which RACE
can select, providing greater flexibility in structured dimensional reduction.
Thus,\emph{ when post-training attention compression is a foreseeable deployment
requirement, allocating attention width across fewer, wider heads during
pretraining may yield a more favorable quality-efficiency trade-off.}

\vspace{-0.2cm}
\section{Conclusion}

In this work, we study when and how pretrained RoPE attention can be converted into a natively
lower-dimensional form that reduces KV-cache storage, attention computation,
and parameters without runtime reconstruction. Viewing Q/K compression as
constrained approximate matrix multiplication motivates a shared reduction
derived jointly from queries and keys. This reduction has to be RoPE-compatible to enable reconstruction-free compression. We further show that, for
coordinate-selection reductions family, RoPE compatibility holds if and only if
complete rotary pairs are retained. RACE selects within this admissible family using joint Q/K activation energy,
a forward-only criterion that minimizes a separable calibration-dependent upper
bound on post-RoPE product error. The resulting reductions are folded directly
into the projection weights, enabling attention to execute entirely in the
compressed dimensions without specialized kernels. Across multiple model
families and compression levels, RACE outperforms evaluated reconstruction-free baselines and remains competitive with
reconstruction-based methods while reducing KV-cache memory, attention cost,
and latency.

\subsection*{AI use statement}

We used generative AI tools to assist with editing portions of the manuscript, retrieving related works, generating and debugging experimental and plotting code, and formatting figures and tables. We also used AI tools to discuss and refine aspects of theoretical presentation and interpretation of experimental results. All AI-assisted technical suggestions, mathematical arguments, code, and interpretations were independently reviewed and verified by the authors against the underlying derivations, implementations, and experimental outputs. No AI generated datasets were used in this work. We take responsibility for the final content of this work, including all text, claims, code, and artifacts produced with AI assistance.

\subsection*{Ethics statement}

The proposed method is intended to reduce the memory and computational cost of long-context inference, which may improve the accessibility and efficiency of deploying language models. As with other efficiency techniques, however, reducing inference cost can also lower the barrier to large-scale model deployment; the downstream societal effects therefore depend on how the resulting models are used. Our work does not modify model objectives, safety mechanisms, or training data, and we do not make claims about changes to model safety, fairness, or robustness beyond the evaluated performance metrics

\subsection*{Reproducibility statement}

We provide sufficient implementation and experimental details to reproduce the results in this work. \ref{app:implementation} specifies the RACE calibration and folding procedure, RoPE-pair handling, value compression, optional KL-based recovery, model and benchmark configurations, long-context evaluation settings, hardware/software environment, precision, number of tokens, kernels, and analytical FLOP derivations. Detailed per-task results, selector ablations, recovery settings, and systems measurements are also provided in the \ref{app:results}.

\bibliography{iclr2027_conference}
\bibliographystyle{iclr2027_conference}

\newpage

\appendix
\section{Appendix}

\subsection{Axes of KV-Cache Compression}
\label{app:cache_axes}

\begin{figure}[b]
    \centering
   \includegraphics[width=\columnwidth]{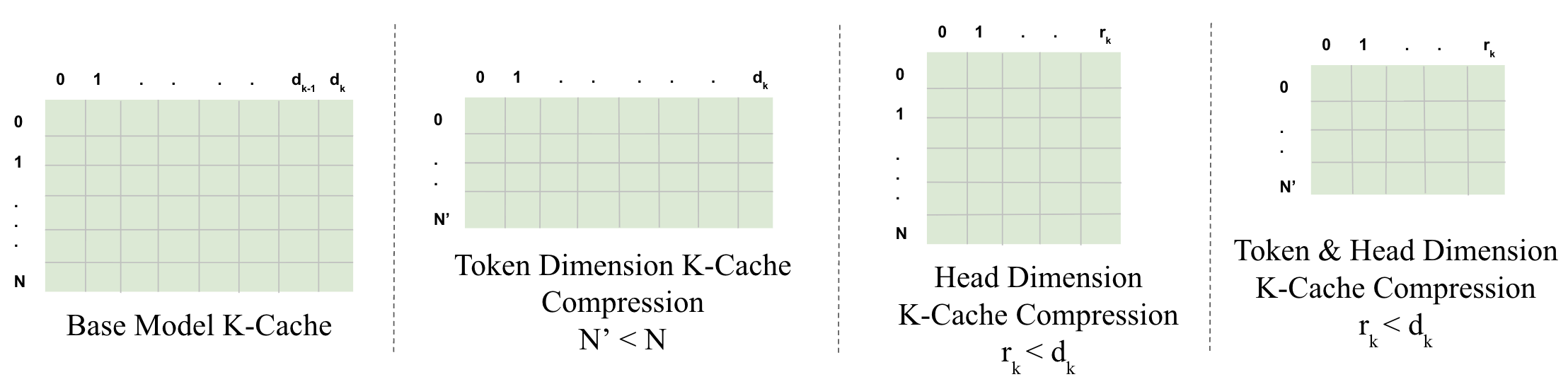}
    \caption{Token- and head- dimension KV-cache compression. Token-level methods reduce the number of retained cache entries N, whereas head-dimension methods reduce the dimensionality of each cached representation. These approaches act on orthogonal axes of the KV cache and can be combined.}
    \label{fig:cache_axes}
\end{figure}

For a model with $L$ layers, batch size $B$, context length $S$, $H_{kv}$
key--value heads, key dimension $d_k$, value dimension $d_v$, and $b$ bytes per
element, the KV-cache memory is
\[
M_{\mathrm{KV}}
    = BLSH_{kv}(d_k+d_v)b.
\]
This expression exposes three independent factors that can be targeted by
compression.

Token-level methods reduce $S$ by evicting, retaining, or consolidating cached
positions. Head-dimension methods reduce $d_k$ and $d_v$ for each retained
position. Quantization reduces $b$, the storage precision of each element.
As illustrated in Figure~\ref{fig:cache_axes}, these approaches operate on
different factors of the cache size and can therefore be applied independently
or jointly. In particular, combining token and head-dimension compression
changes the cache size to
\[
M_{\mathrm{KV}}'
    = BLS'H_{kv}(r_k+r_v)b,
\]
where $S'\leq S$, $r_k\leq d_k$, and $r_v\leq d_v$.

\subsection{Attention and KV caching.}
For hidden states $X\in\mathbb{R}^{S\times d_{\mathrm{model}}}$ and attention weights $W_Q,W_K \in \mathbb{R}^{d_{\mathrm{model}} \times d_k}$ and $W_V \in \mathbb{R}^{d_{\mathrm{model}} \times d_v}$, an attention
layer computes $Q=XW_Q,\qquad K=XW_K,\qquad V=XW_V$.
For each head, let $Q_h,K_h\in\mathbb{R}^{S\times d_k}$ and
$V_h\in\mathbb{R}^{S\times d_v}$. Its output is $O_h = A_hV_h = 
\operatorname{softmax}\!\left(
\frac{Q_hK_h^\top}{\sqrt{d_k}}
\right)V_h. $
During autoregressive decoding, previously computed keys and values are stored
in the KV cache. For $L$ layers, batch size $B$, context length $S$, $H_{kv}$
KV heads, $H_{q}$
Q heads, and $b$ bytes per element, its memory is
\begin{equation}
 M_{\mathrm{KV}}
= BLSH_{kv}(d_k+d_v)b. 
\label{eq1}
\end{equation}
Thus, reducing $d_k$ and $d_v$ decreases cache memory without removing cached
tokens. It also reduces the query-key and attention-value products, whose
context-dependent decoding costs scale as
$\mathcal{O}(BSH_qd_k)$ and $\mathcal{O}(BSH_qd_v)$, respectively. Their prefill costs scale as $\mathcal{O}(BS^2H_qd_k)$ and $\mathcal{O}(BS^2H_qd_v)$, respectively.  Complete
FLOP expressions are provided in ~\ref{app:flops_analysis}.

\subsection{RoPE-Compatible Reduction}
\label{app:rope_details}

\paragraph{Rotary positional embeddings.}
RoPE partitions the rotary query and key dimensions into coordinate pairs and
applies a position-dependent two-dimensional rotation to each pair
\citep{roformer}. For the pair associated with frequency $\theta_i$ at
position $t$,
\[
R_i(t)=
\begin{bmatrix}
\cos(t\theta_i) & -\sin(t\theta_i)\\
\sin(t\theta_i) &  \cos(t\theta_i)
\end{bmatrix}.
\]
Although implementations may store paired coordinates consecutively or across
two halves of the head dimension, the two coordinates remain coupled by the
rotation. A coordinate-selection method can therefore preserve the pretrained
RoPE operation without reconstruction only if it retains or removes complete
rotary pairs. Architectures with partial RoPE additionally contain non-rotary
coordinates, which need not satisfy this pairwise constraint. 

\paragraph{Logical pairs and physical coordinate layouts.}
Let $d_{\mathrm{rot}}\leq d_k$ denote the number of query and key coordinates
to which RoPE is applied, and let $M=d_{\mathrm{rot}}/2$. We represent the
logical rotary pairs as
\[
\mathcal{P}_{\mathrm{rope}}
    = \{(i_m,j_m)\}_{m=1}^{M},
\]
where $(i_m,j_m)$ identifies the two physical coordinates rotated using the
$m$-th rotary frequency. The mapping from logical pairs to physical coordinates
depends on the model implementation. With one-based indexing, common layouts
include
\[
(i_m,j_m)=
\begin{cases}
(2m-1,\,2m), & \text{adjacent-pair layout},\\
(m,\,m+M),   & \text{split-half layout}.
\end{cases}
\]
The former is used by GPT-J-style rotary implementations, whereas the latter
is used by LLaMA-style implementations. Our selection operates on these
logical pairs rather than on neighboring memory locations, so the same
procedure applies to either layout after constructing the appropriate pair
map.

\paragraph{Restriction of the original RoPE operator.}
Let $\mathcal{S}\subseteq[M]$ be the retained set of rotary pairs and let
$P_{\mathcal{S}}\in\{0,1\}^{d_k\times 2|\mathcal{S}|}$ contain the standard
basis vectors corresponding to both coordinates of every pair in
$\mathcal{S}$. The reduced RoPE operator is formed by retaining the original
rotation block associated with each selected pair:
\[
R_t^{(\mathcal{S})}
    = \bigoplus_{m\in\mathcal{S}} R_{t,m}.
\]
Provided that the columns of $P_{\mathcal{S}}$ and the blocks of
$R_t^{(\mathcal{S})}$ use the same pair ordering,
\[
R_tP_{\mathcal{S}}
    = P_{\mathcal{S}}R_t^{(\mathcal{S})}.
\]
Consequently, for row-wise query activations,
\[
(QR_t)P_{\mathcal{S}}
    = (QP_{\mathcal{S}})R_t^{(\mathcal{S})},
\]
and analogously for keys. Thus, selecting coordinates before RoPE and applying
the restricted operator produces exactly the selected coordinates of the
original RoPE output. This is an algebraic restriction of the pretrained
operator, not an approximation introduced by the reduced-dimensional
implementation.

\paragraph{Preserving pretrained rotary frequencies.}
Reducing the query--key dimension must not cause the rotary frequencies to be
recomputed from the new dimension $r_k$. Standard RoPE implementations derive
frequencies from their original rotary dimension; reinitializing RoPE with
$r_k$ would generally assign different angles to the retained coordinates and
would therefore alter the pretrained positional encoding. Instead, for every
retained pair $m\in\mathcal{S}$, we copy its original frequency $\theta_m$:
\[
\boldsymbol{\theta}^{(\mathcal{S})}
    = [\theta_m]_{m\in\mathcal{S}}.
\]
The selected pairs may be stored in score order or original-coordinate order,
provided that the queries, keys, and frequencies use the same ordering. In our
implementation, we retain the original pair order after selection to make the
folded checkpoints deterministic.

For models using scaled RoPE, the same principle applies: the model's original
position-dependent frequency transformation and any associated attention
scaling are preserved, and only the entries corresponding to selected pairs
are gathered. This avoids silently replacing the pretrained positional
encoding with a new reduced-dimensional frequency schedule.

\paragraph{Partial RoPE.}
Some architectures apply RoPE to only
$d_{\mathrm{rot}}<d_k$ coordinates. Let
$\mathcal{I}_{\mathrm{NoPE}}$ denote the complementary non-rotary coordinates.
The complete-pair constraint is imposed only on
$\mathcal{P}_{\mathrm{rope}}$; non-rotary coordinates may be selected
individually. If $\mathcal{J}\subseteq\mathcal{I}_{\mathrm{NoPE}}$ is the
retained non-rotary set, the complete selector is
\[
P_{QK}
    =
    \left[
        P_{\mathcal{S}}\;
        P_{\mathcal{J}}
    \right],
    \qquad
r_k=2|\mathcal{S}|+|\mathcal{J}|,
\]
and the corresponding reduced positional operator is
\[
R_t^{(\mathcal{S},\mathcal{J})}
    =
    R_t^{(\mathcal{S})}\oplus I_{|\mathcal{J}|}.
\]
The identity block leaves the retained non-rotary coordinates unchanged.
Hence, the same commutation relation holds for architectures with full or
partial RoPE.

\paragraph{Grouped-query attention.}
Under grouped-query attention, several query heads attend to a shared key head.
For each KV head $g$, let $\mathcal{G}_g$ denote its associated query heads. We
compute the Q/K pair statistics jointly over the key activations of head $g$
and the query activations from all heads in $\mathcal{G}_g$. A single pair set
$\mathcal{S}_g$ is then applied to the shared key head and to every query head
in its group. This ensures that each query and the key cache against which it
is evaluated use identical retained coordinates and rotary frequencies. Pair
sets may differ across KV heads, but they cannot differ among query heads that
interact with the same cached key representation.

\paragraph{Implementation consequence.}
After selection, the columns indexed by $P_{QK}$ are folded into $W_Q$ and
$W_K$, and the RoPE module is restricted to the corresponding original
frequencies. The resulting layer produces $r_k$-dimensional queries and keys
directly. Neither the discarded coordinates nor a full-dimensional RoPE output
is materialized during prefill or decoding.

\subsection{Inference FLOPs Analysis}
\label{app:flops_analysis}

We analyze the attention-related floating-point operations (FLOPs) incurred during \emph{prefill} and \emph{autoregressive decoding}. We count a multiply-add as two FLOPs and omit element-wise operations such as softmax, normalization, and RoPE, which are lower-order relative to the matrix multiplications considered here.

Let $B$ denote the batch size, $S$ the context length, $d$ the model hidden dimension, $H_q$ the number of query heads, and $H_{\mathrm{kv}}$ the number of key-value heads. The original query/key and value head dimensions are $d_k$ and $d_v$, respectively. After compression, their corresponding dimensions are
$r_k$ and $r_v$.

For prefill, projections are applied to all $S$ input tokens and attention is computed over an $S\times S$ attention matrix. For decoding, we report the FLOPs for generating one new token given a cache of $S$ previous tokens.

\paragraph{Full Attention}

For standard attention, the query, key, and value projections during prefill cost
\begin{equation}
F_{\mathrm{proj}}^{\mathrm{pre}}
=
2BSd
\left(
H_q d_k + H_{\mathrm{kv}} d_k + H_{\mathrm{kv}} d_v
\right).
\end{equation}

The query-key product and attention-value product cost
\begin{align}
F_{QK}^{\mathrm{pre}} &= 2BS^2 H_q d_k, \\
F_{AV}^{\mathrm{pre}} &= 2BS^2 H_q d_v,
\end{align}
and the output projection costs
\begin{equation}
F_O^{\mathrm{pre}}
=
2BS H_q d_v d.
\end{equation}

Thus,
\begin{equation}
F_{\mathrm{Full}}^{\mathrm{pre}}
=
2BSd
\left(
H_qd_k + H_{\mathrm{kv}}d_k + H_{\mathrm{kv}}d_v
\right)
+
2BS^2H_q(d_k+d_v)
+
2BSH_qd_vd.
\end{equation}

During decoding, projections are required only for the newly generated token:
\begin{equation}
F_{\mathrm{proj}}^{\mathrm{dec}}
=
2Bd
\left(
H_qd_k + H_{\mathrm{kv}}d_k + H_{\mathrm{kv}}d_v
\right).
\end{equation}

The query attends to all $S$ cached keys and values, giving
\begin{align}
F_{QK}^{\mathrm{dec}} &= 2BSH_qd_k,\\
F_{AV}^{\mathrm{dec}} &= 2BSH_qd_v,
\end{align}
while the output projection costs
\begin{equation}
F_O^{\mathrm{dec}}=2BH_qd_vd.
\end{equation}

Therefore,
\begin{equation}
F_{\mathrm{Full}}^{\mathrm{dec}}
=
2Bd
\left(
H_qd_k + H_{\mathrm{kv}}d_k + H_{\mathrm{kv}}d_v
\right)
+
2BSH_q(d_k+d_v)
+
2BH_qd_vd.
\end{equation}

For long contexts, the cache-length-dependent term dominates:
\begin{equation}
F_{\mathrm{Full,LC}}^{\mathrm{dec}}
\simeq
2BSH_q(d_k+d_v).
\end{equation}

\paragraph{EigenAttention with RoPE}

EigenAttention projects attention representations K and V into learned low-dimensional subspaces. In its RoPE-compatible formulation, compressed keys must be decompressed into the original dimension, which is RoPE-compatible, before query-key attention.

During prefill, the projection costs
\begin{equation}
F_{proj}^{\mathrm{pre}}
=
2BSd(H_qd_k + H_{kv}r_k + H_{kv}r_V).
\end{equation}

Transforming the compressed keys to the original key space costs
\begin{equation}
F_{K,\mathrm{rec}}^{\mathrm{pre}}
=
2BSH_{\mathrm{kv}}r_kd_k.
\end{equation}

Since queries and reconstructed keys interact in the original key dimension,
the query-key product is the same as full attention. The value representation remains compressed to dimension $r_v$, giving
\begin{equation}
F_{AV}^{\mathrm{pre}}
=
2BS^2H_qr_v \qquad
F_{O}^{\mathrm{pre}}
=
2BSH_qr_vd.
\end{equation}

Thus, the prefill cost is
\begin{equation}
F_{\mathrm{Eigen}}^{\mathrm{pre}}
= 2BS \left[d(H_qd_k + H_{kv}r_k + H_{kv}r_V) +
H_{\mathrm{kv}}r_kd_k
+
SH_q(d_k+r_v) + H_qr_vd \right]
\end{equation}

During decoding, all $S$ cached compressed keys must be reconstructed before
RoPE, giving
\begin{equation}
F_{proj}^{\mathrm{dec}}
=
2Bd(H_qd_k + H_{kv}r_k + H_{kv}r_V).
\qquad
F_{K,\mathrm{rec}}^{\mathrm{dec}}
=
2BSH_{\mathrm{kv}}r_kd_k.
\end{equation}

The subsequent query-key, attention-value operations, and output cost
\begin{equation}
F_{QK}^{\mathrm{dec}}
=
2BSH_qd_k \qquad
F_{AV}^{\mathrm{dec}}
=
2BSH_qr_v, \qquad
F_{O}^{\mathrm{dec}}
=
2BH_qr_vd.    
\end{equation}

Therefore, the cache-length-dependent cost of generating one token is
\begin{equation}
F_{\mathrm{Eigen,LC}}^{\mathrm{dec}}
=
2BS
\left[
H_{\mathrm{kv}}r_kd_k
+
H_q(d_k+r_v)
\right].
\end{equation}

This additional $O(BSH_{kv}r_kd_k)$ term arises from reconstructing cached keys
into the RoPE-compatible space. 

\paragraph{PALU with G-LRD and RoPE}

PALU~\citep{chang2025palu} applies group-head low-rank decomposition
(G-LRD) to the key and value projection matrices. Let $s$ denote the
number of KV heads in each group and
$G = \frac{H_{\mathrm{kv}}}{s}$
the number of groups. For each key group $g$, PALU factorizes the
concatenated projection of $s$ heads as
\[
W_K^{(g)}
\approx
A_K^{(g)}B_K^{(g)},
\]
where
\[
A_K^{(g)}\in\mathbb{R}^{d\times r_k^{(g)}},
\qquad
B_K^{(g)}\in\mathbb{R}^{r_k^{(g)}\times s d_k}.
\]
The intermediate representation produced by $A_K^{(g)}$ is stored in the KV cache. For RoPE-based attention, the cached key latent is reconstructed to the original $s d_k$-dimensional group representation before applying RoPE and computing attention.

Assuming uniform group ranks
$r_k^{(g)}=r_k^{G}$ and $r_v^{(g)}=r_v^{G}$, the Q/K/V projection cost
during prefill is
\begin{equation}
F_{\mathrm{proj}}^{\mathrm{pre}}
=
2BSd
\left(
H_qd_k
+
G r_k^{G}
+
G r_v^{G}
\right).
\end{equation}

For one decoding step,
\begin{equation}
F_{\mathrm{proj}}^{\mathrm{dec}}
=
2Bd
\left(
H_qd_k
+
G r_k^{G}
+
G r_v^{G}
\right).
\end{equation}

For each group, reconstructing $S$ cached key latents from dimension $r_k^{G}$ to $s d_k$ costs $2BS r_k^{G}s d_k$. Across all $G$ groups,
\begin{align}
F_{K,\mathrm{rec}}^{\mathrm{pre}}
 =
2BSG r_k^{G}s d_k  
 =
2BSH_{\mathrm{kv}}r_k^{G}d_k.
\end{align}
The same cache-wide reconstruction is required for each decoding step:
\begin{equation}
F_{K,\mathrm{rec}}^{\mathrm{dec}}
=
2BSH_{\mathrm{kv}}r_k^{G}d_k.
\end{equation}

Because reconstructed keys have the original head dimension $d_k$,
the query-key products and value products cost
\begin{align}
F_{QK}^{\mathrm{pre}}
=
2BS^2H_qd_k, \qquad
F_{QK}^{\mathrm{dec}}
=
2BSH_qd_k.
\end{align}
\begin{align}
F_{AV}^{\mathrm{pre}}
=
2BS^2Gr_V^G, \qquad
F_{AV}^{\mathrm{dec}}
=
2BSGr_V^G.
\end{align}

PALU folds the value-side reconstruction into the output projection, so attention-value aggregation can be performed in the grouped latent value space. Let the total latent value width across all groups be $d_V^{\mathrm{latent}} = G r_v^{G}$. The precise arithmetic depends on the mapping between query heads and value groups; we therefore retain the value-side latent width explicitly.

The fused output projection maps the concatenated latent value
representation back to the model dimension, giving
\begin{align}
F_O^{\mathrm{pre}} =
2BS\,G r_v^{G}d, \qquad
F_O^{\mathrm{dec}}
= 2B\,G r_v^{G}d.
\end{align}

The prefill cost is 
\begin{equation}
F_{PALU}^{pre} = 
2BS\left[d
\left(
H_qd_k
+
G r_k^{G}
+
G r_v^{G}
\right) + H_{\mathrm{kv}}r_k^{G}d_k + SH_qd_k + Gr_V^G (S+d) \right].
\end{equation}

Ignoring the $S$-independent projection and output terms, the dominant
cache-length-dependent decoding cost is
\begin{equation}
F_{\mathrm{Palu,LC}}^{\mathrm{dec}}
\simeq
2BS
\left(
H_{\mathrm{kv}}r_k^{G}d_k
+
H_qd_k + Gr_v^{G}
\right) .
\end{equation}

\paragraph{Ours}

Our method reduces the architectural query/key and value head dimensions to
$r_k$ and $r_v$, respectively. All compression transformations are folded into
the query, key, value, and output projection matrices. Consequently, neither
key nor value reconstruction is required during inference. The native-reduced dimensional attention FLOPs are as follows

During prefill,
\begin{equation}
F_{\mathrm{Ours}}^{\mathrm{pre}}
=
2BSd
\left(
H_qr_k + H_{\mathrm{kv}}r_k + H_{\mathrm{kv}}r_v
\right)+
2BS^2H_q(r_k+r_v)
+
2BSH_qr_vd.
\end{equation}

During decoding,
\begin{equation}
F_{\mathrm{Ours}}^{\mathrm{dec}}
=
2Bd
\left(
H_qr_k + H_{\mathrm{kv}}r_k + H_{\mathrm{kv}}r_v
\right)+
2BSH_q(r_k+r_v)
+
2BH_qr_vd.
\end{equation}

The cache-length-dependent decoding cost is therefore
\begin{equation}
F_{\mathrm{Ours,LC}}^{\mathrm{dec}}
\simeq
2BSH_q(r_k+r_v).
\end{equation}

For $d_k=d_v$ and $r_k=r_v=d_k/2$, this term is exactly half that of
full attention:
\begin{equation}
\frac{
F_{\mathrm{Ours,LC}}^{\mathrm{dec}}
}{
F_{\mathrm{Full,LC}}^{\mathrm{dec}}
}
=
\frac{r_k+r_v}{d_k+d_v}
=
\frac{1}{2}.
\end{equation}

RAP follows similar compressed attention as Ours and therefore the prefill and decoding FLOPs would be alike. Similarly, Think Keys Full Values also compresses attention while keeping Values in full dimension $r_v = d_v$.

\subsubsection{Derivation of Table~\ref{tab:flops_analysis}}
\label{app:table1_derivation}

Table \ref{tab:flops_analysis} compares methods under the representative grouped-query attention configuration \(d=1536\), \(H_q=6\), \(H_{kv}=2\), and \(d_k=d_v=128\). For RACE, RAP, and EigenAttention, we use \(r_k=r_v=64\), while Thin Keys uses \(r_k=64\) and \(r_v=128\). For PALU G-LRD, the two KV heads form a single group (\(G=1\), \(s=2\)) with group ranks \(r_k^G=r_v^G=128\); these ranks are defined per group rather than per head. We count a multiply–add as two FLOPs and omit elementwise operations. For prefill, we report the coefficient of the dominant \(O(S^2)\) attention computation, while for single-token decoding, we report all operations that scale linearly with the cached sequence length \(S\). Common factors such as \(2BS^2\) for prefill and \(2BS\) for decoding therefore cancel in the normalized ratios. This convention captures the distinction between native dimensional reduction, where attention is performed directly in the compressed representation, and reconstruction-based methods, which incur additional cache-length-dependent work during decoding.

\paragraph{RACE, RAP, and Thin Keys.} Both methods execute attention natively with \(r_k=r_v=64\), without reconstructing the discarded dimensions. Their cache retention is
\[
\rho_{\mathrm{cache}}
=\frac{H_{kv}(r_k+r_v)}{H_{kv}(d_k+d_v)}
=\frac{64+64}{128+128}=0.5.
\]The same dimensional reduction applies to the Q, K, V, and output projections, giving \(P_{\mathrm{native}}/P_{\mathrm{full}}=0.5\). Since both the query–key and attention–value products are evaluated directly in dimensions \(r_k\) and \(r_v\), the normalized dominant prefill and decoding costs are also
\[
\rho_{\mathrm{prefill}}
=\rho_{\mathrm{decode}}
=\frac{H_q(r_k+r_v)}{H_q(d_k+d_v)}
=0.5.
\]

For the Thin Keys baseline, $r_v = 128$ as they do not compress values. This gives $\rho_{cache}=0.75$ and $\rho_{prefill}= \rho_{decode} = 0.75$.

\paragraph{EigenAttention.}
EigenAttention stores keys and values in reduced dimensions
$r_k=r_v=64$, giving 50\% KV-cache retention. Following the
parameter-counting convention in Table~1, we count the compressed
projections together with the learned key and value bases:
\begin{align}
P_{\text{Eigen}}
&=
d\left(
H_q d_k + H_{kv} r_k + H_{kv} r_v + H_q r_v
\right)
+ H_{kv}(r_k d_k + r_v d_v).
\end{align}
For the representative configuration used for parameter counting,
$d=1536$, $H_q=H_{kv}=6$, $d_k=d_v=128$, and
$r_k=r_v=64$, this gives
\[
\frac{P_{\text{Eigen}}}{P_{\text{full}}}
=0.698 
\]

During prefill, key reconstruction scales as $O(S)$, while the
query-key and attention-value products scale as $O(S^2)$. Hence,
asymptotically the dominant prefill cost is
\[
\rho_{\text{prefill}}^{\text{Eigen}}
=
\frac{H_q(d_k+r_v)}
     {H_q(d_k+d_v)}
=
\frac{128+64}{128+128}
=0.75.
\]

For decoding, however, every cached compressed key must be mapped back
to the original key dimension before query-key attention. The
cache-length-dependent cost $
F_{\text{Eigen,LC}}^{\text{dec}}
\simeq
2BS\left[
H_{kv}r_kd_k
+
H_q(d_k+r_v)
\right].
$
Normalizing by full attention,
$
F_{\text{Full,LC}}^{\text{dec}}
=
2BSH_q(d_k+d_v),
$
gives
\[
\rho_{\text{decode}}^{\text{Eigen}}
=
\frac{
H_{kv}r_kd_k + H_q(d_k+r_v)
}{
H_q(d_k+d_v)
}=
11.42.
\]

\paragraph{PALU.}
PALU G-LRD jointly compresses groups of KV heads. For the parameter
comparison in Table~1, we use $H_{kv}=2$ heads divided into
$G=1$ groups of $s=2$ heads, with group ranks
$r_k^G=r_v^G=128$. The resulting cache stores
\[
C_{\text{PALU}}
=
G(r_k^G+r_v^G)
=
3(128+128)
=
768
\]
scalars per token, corresponding to 50\% of the full KV cache.

Following the parameter-counting convention in Table~1, we count the
query projection, both low-rank factor matrices for the grouped key and
value projections, and the fused value-output map:
$
P_{\text{PALU}}
=
dH_qd_k
+
Gr_k^G(d+s d_k)
+
Gr_v^G(d+s d_v)
+
dGr_v^G $
which gives
\[
\frac{P_{\text{PALU}}}{P_{\text{full}}}
=
0.583.
\]

PALU reconstructs cached keys into the original key dimension before
query-key attention, while value aggregation is performed directly in
the grouped latent value space. Consequently, the dominant
cache-length-dependent decoding cost is
\[
F_{\text{PALU,LC}}^{\text{dec}}
\simeq
2BS\left[
H_{kv} r_k^G d_k
+
H_q d_k
+
G r_v^G
\right].
\]
Normalizing by full attention gives
\[
\rho_{\text{decode}}^{\text{PALU}}
=
\frac{
H_{kv} r_k^G d_k + H_q d_k + G r_v^G
}{
H_q(d_k+d_v)
}.
=
21.92.
\]

Fore prefill,
\[\rho_{\text{prefill}}^{\text{PALU}}
=
\frac{H_qd_k+Gr_v^G}
{H_q(d_k+d_v)}. = 0.75\]

\subsection{Approximate Matrix Multiplication View of Attention Compression}
\label{app:amm}

This section formalizes the connection between head-dimension compression and
approximate matrix multiplication (AMM), and clarifies how our
RoPE-constrained coordinate selection relates to existing sampling-based AMM
results.

\paragraph{AMM formulation.}
Consider a single attention head with post-RoPE query and key matrices $ Q,K\in\mathbb{R}^{S\times d_k}$.
The pre-softmax attention interaction is $L=QK^\top$.

Following the standard AMM formulation, let
\[
A=Q^\top\in\mathbb{R}^{d_k\times S},
\qquad
B=K^\top\in\mathbb{R}^{d_k\times S},
\]
so that $A^\top B = QK^\top$.
Given a dimensionality-reducing operator $\Pi\in\mathbb{R}^{r_k\times d_k}$,
AMM replaces this product by
\begin{equation}
(\Pi A)^\top(\Pi B)
=
Q\Pi^\top\Pi K^\top.
\label{eq:app_amm_attention}
\end{equation}
Thus, dimensionality reduction along the shared row dimension of $A$ and $B$
is exactly dimensionality reduction along the head-feature dimension of
$Q$ and $K$.

\paragraph{Stable-rank AMM.}
For conforming matrices $A$ and $B$, ~\citep{cohen2016optimal}
define a $(k,\epsilon)$ spectral-norm AMM property of the form
\begin{equation}
\left\|
(\Pi A)^\top(\Pi B)-A^\top B
\right\|_2
\leq
\epsilon
\sqrt{
\left(
\|A\|_2^2+\frac{\|A\|_F^2}{k}
\right)
\left(
\|B\|_2^2+\frac{\|B\|_F^2}{k}
\right)
}.
\label{eq:app_k_amm}
\end{equation}
They show that suitable oblivious embeddings can satisfy this property using
\[
m=
\mathcal{O}\left(
\frac{k+\log(1/\delta)}{\epsilon^2}
\right)
\]
rows with probability at least $1-\delta$.

The stable rank of a matrix $A$ is
\begin{equation}
\operatorname{sr}(A)
=
\frac{\|A\|_F^2}{\|A\|_2^2}.
\label{eq:app_stable_rank}
\end{equation}
Setting
\[
k
\gtrsim
\max\{\operatorname{sr}(A),\operatorname{sr}(B)\}
\]
makes the Frobenius-to-spectral terms in
Eq.~\eqref{eq:app_k_amm} controlled by the corresponding spectral norms.
Consequently, the required sketch dimension can scale with the effective
spectral dimension of the operands rather than their ambient dimension.
Applied to attention with $A=Q^\top$ and $B=K^\top$, this provides theoretical
motivation for reducing $d_k$ when the query and key activations have low
stable rank.




\paragraph{From row sampling in AMM to feature selection in attention.}
AMM sampling methods act on the shared row dimension of the two operands.
Under the identification $A=\widetilde Q^\top;
    B=\widetilde K^\top$
the rows of $A$ and $B$ correspond to the feature dimensions of
$\widetilde Q$ and $\widetilde K$. Let
\[
    P_{\mathcal I}
    =
    [e_{i_1},\ldots,e_{i_r}]
    \in\{0,1\}^{d_k\times r_k}
\]
select the coordinate set $\mathcal I$, and take
$\Pi=P_{\mathcal I}^{\top}$. The AMM approximation then becomes
\begin{equation}
\begin{aligned}
    (\Pi A)^\top(\Pi B)
    &=
    \widetilde Q
    P_{\mathcal I}P_{\mathcal I}^{\top}
    \widetilde K^\top \\
    &=
    (\widetilde QP_{\mathcal I})
    (\widetilde KP_{\mathcal I})^\top .
\end{aligned}
\label{eq:amm_selection}
\end{equation}
Thus, row selection in AMM is exactly feature selection in attention,
with the \emph{same} selection matrix applied to queries and keys.

Equation~\ref{eq:amm_selection} motivates two design choices. First,
queries and keys should use an identical reduction $P_{\mathcal I}$ or in the context of RACE, $P_{QK}$ rather than
independently constructed maps $P_Q$ and $P_K$. 

Separate maps
would produce $(\widetilde QP_Q)(\widetilde KP_K)^\top
    =
    \widetilde QP_QP_K^\top\widetilde K^\top,
$
which is not the shared-sketch AMM approximation
$\widetilde QP_{QK}P_{QK}^\top\widetilde K^\top$.

Second, because the approximation target is the joint product
$\widetilde Q\widetilde K^\top$, a data-dependent choice of
$P_{QK}$ should account for both operands. Selecting dimensions using
only query or only key statistics can discard features that have low
importance for one operand but materially affect their interaction.

These observations determine the form of the reduction - a common,
jointly informed Q/K selector, but not its precise selection rule.
General AMM results show that appropriately chosen subsets of the shared
dimension can preserve matrix products. RACE further restricts this
family to selectors that are compatible with RoPE and derives an
activation-energy criterion for choosing among them.

\paragraph{Relation to our selector.}

RACE's compression operator is a more restricted member of this general family.
For each attention head, we use an unweighted binary selection matrix
$P_{QK}
=
[e_i]_{i\in\mathcal I_{QK}}$,
where $\mathcal I_{QK}$ is chosen from calibration activations. Unlike the general data-dependent sampling operators covered by AMM theory, our selector must additionally satisfy the structural constraints induced by RoPE. Specifically, if $
\mathcal P_{\mathrm{rope}}
=
\{(i_m,j_m)\}_{m=1}^{M}
$
denotes the rotary coordinate pairs, then
\begin{equation}
\mathcal I_{QK}
=
\bigcup_{m\in\mathcal S}\{i_m,j_m\},
\qquad
|\mathcal S|=\frac{r_k}{2}.
\label{eq:app_rope_pairs}
\end{equation}
The selected set is therefore a union of complete two-dimensional rotary
blocks.

\paragraph{Energy-based selection and scope of the guarantee.}
The AMM results above establish that low-dimensional and sampling-based
representations capable of preserving matrix products exist under appropriate
conditions. They do not imply that every coordinate subset, or specifically
our RoPE-constrained subset, satisfies the same AMM bound.

Our method chooses among the admissible RoPE pairs using calibration energy.
For pair $(i_m,j_m)$, we define
\begin{equation}
E_m
=
\|Q_{:,i_m}\|_2^2
+
\|Q_{:,j_m}\|_2^2
+
\|K_{:,i_m}\|_2^2
+
\|K_{:,j_m}\|_2^2,
\end{equation}
and retain the $r_k/2$ highest-energy pairs. This criterion favors coordinates
that carry substantial query/key activation energy while remaining compatible
with RoPE.

We therefore use stable-rank AMM and sampling-based AMM as theoretical
motivation for reducing the shared query--key dimension, rather than claiming
a direct AMM guarantee for the proposed energy-based selector. The exact
property guaranteed by our construction is structural: selecting complete
rotary pairs permits RoPE and attention to be executed natively in the reduced
space.

\subsection{Proofs}
\label{app:proofs}

\paragraph{Proof of Theorem~\ref{thm:rope_compatibility}.} [RoPE-compatible coordinate selection]
\begin{proof}

Suppose $P_{QK}$ selects exactly the coordinates belonging to the
complete rotary pairs indexed by
$\mathcal S\subseteq\{1,\ldots,M\}$, with its columns ordered
consistently with the blocks of $R_t^{(\mathcal S)}$.
Since
\[
R_t=\operatorname{diag}(R_{1,t},\ldots,R_{M,t})
\]
acts independently on each rotary pair, applying $R_t$ to the selected
coordinates applies precisely the blocks $R_{m,t}$ for $m\in\mathcal S$.
Therefore,
\[
R_tP_{QK}
=
P_{QK}\operatorname{diag}(R_{m,t}:m\in\mathcal S)
=
P_{QK}R_t^{(\mathcal S)}
\]
for every $t$. Hence $P_{QK}$ is RoPE-compatible.

Conversely, suppose
\[
R_tP_{QK}
=
P_{QK}R_t^{(\mathcal S)}
\qquad\text{for all }t.
\]
The right-hand side lies in $\operatorname{col}(P_{QK})$, and therefore
\[
R_t\operatorname{col}(P_{QK})
\subseteq
\operatorname{col}(P_{QK})
\qquad\text{for all }t.
\]
Thus, the coordinate subspace selected by $P_{QK}$ must be invariant
under every $R_t$.

Consider any rotary pair $(i_m,j_m)$. Suppose, for contradiction, that
$P_{QK}$ selects $i_m$ but not $j_m$. Then
$e_{i_m}\in\operatorname{col}(P_{QK})$, whereas
$e_{j_m}\notin\operatorname{col}(P_{QK})$.
For the $m$-th rotary block,
\[
R_{m,t}
=
\begin{bmatrix}
\cos\phi_{m,t} & -\sin\phi_{m,t}\\
\sin\phi_{m,t} & \cos\phi_{m,t}
\end{bmatrix},
\]
so, up to the coordinate convention for the pair,
\[
R_te_{i_m}
=
\cos\phi_{m,t}\,e_{i_m}
+
\sin\phi_{m,t}\,e_{j_m}.
\]
By assumption, pair $m$ undergoes a nontrivial rotation for at least one
position, so there exists $t$ such that
$\sin\phi_{m,t}\neq0$. For this $t$, $R_te_{i_m}$ has a nonzero
component along $e_{j_m}$ and therefore
\[
R_te_{i_m}
\notin
\operatorname{col}(P_{QK}),
\]
contradicting the invariance established above.

The same argument applies if $j_m$ is selected but $i_m$ is not.
Hence, for every rotary pair, either both coordinates are selected or
neither is selected. Therefore the coordinates selected by $P_{QK}$
form a union of complete rotary pairs, indexed by some
$\mathcal S\subseteq\{1,\ldots,M\}$.
    
\end{proof}

\paragraph{Proof of Proposition 1.}
[Error bound for RoPE-pair truncation] 
\begin{proof}
    
Partition the post-RoPE query and key matrices into RoPE-aligned coordinate pairs,
\[
\widetilde Q=[\widetilde Q_1,\ldots,\widetilde Q_M],
\qquad
\widetilde K=[\widetilde K_1,\ldots,\widetilde K_M],
\]
where each
\[
\widetilde Q_m, \widetilde K_m\in\mathbb{R}^{S\times 2}.
\]
Because matrix multiplication along the feature dimension decomposes additively
across these blocks,
\[
\widetilde Q \widetilde K^\top
=
\sum_{m=1}^{M} \widetilde Q_m \widetilde K_m^\top.
\]
For a retained pair set $\mathcal S\subseteq[M]$, the corresponding truncated
product is
\[
\widetilde Q_{\mathcal S} \widetilde K_{\mathcal S}^\top
=
\sum_{m\in\mathcal S} \widetilde Q_m \widetilde K_m^\top.
\]
Hence the truncation error would be,
\begin{equation}
\widetilde Q \widetilde K^\top-\widetilde Q_{\mathcal S} \widetilde K_{\mathcal S}^\top
=
\sum_{m\notin\mathcal S} \widetilde Q_m \widetilde K_m^\top.
\label{eq:pair_error_decomposition}
\end{equation}

Applying the triangle inequality for the Forbenius norm gives
\begin{align}
\left\|
\widetilde Q \widetilde K^\top-\widetilde Q_{\mathcal S}\widetilde K_{\mathcal S}^\top
\right\|_F
&\le
\sum_{m\notin\mathcal S}
\left\|\widetilde Q_m \widetilde K_m^\top\right\|_F.
\end{align}
Since for any conforming matrices $A$ and $B$, $\left\|AB^T\right\|_F \leq \left\|A\right\|_F \left\|B\right\|_F$, the above equation can be written as 

\begin{align}
\left\|
\widetilde Q \widetilde K^\top-\widetilde Q_{\mathcal S} \widetilde K_{\mathcal S}^\top
\right\|_F
&\le
\sum_{m\notin\mathcal S}
\left\|\widetilde Q_m\right\|_F \left\|\widetilde K_m\right\|_F.
\end{align}

Finally, by the arithmetic--geometric mean inequality, we obtain
\begin{equation}
\|\widetilde Q_m\|_F\|\widetilde K_m\|_F
\le
\frac{1}{2}
\left(
\|\widetilde Q_m\|_F^2+\|\widetilde K_m\|_F^2
\right)
=
\frac{1}{2}E_m.
\end{equation}
Combining the above inequalities yields
\begin{equation}
\boxed{
\left\|
\widetilde Q \widetilde K^\top-\widetilde Q_{\mathcal S} \widetilde K_{\mathcal S}^\top
\right\|_2
\le
\sum_{m\notin\mathcal S}
\left\|\widetilde Q_m\right\|_F \left\|\widetilde K_m\right\|_F
\le
\frac{1}{2}
\sum_{m\notin\mathcal S}E_m.
}
\end{equation}

For a fixed budget $|\mathcal S|=k$, minimizing the right-hand side is
equivalent to maximizing
\[
\sum_{m\in\mathcal S}E_m.
\]
Therefore, the subset consisting of the $k$ largest-energy RoPE pairs
minimizes this upper bound among all subsets of $k$ pairs.
\hfill$\square$

\paragraph{Evaluation in the pre- or post-RoPE space}
Because each RoPE pair is transformed by an orthogonal $ 2\times2\ $ rotation,

$$\left\|\widetilde Q_m\right\|_F = \left\|Q_m\right\|_F \qquad \left\|\widetilde K_m\right\|_F = \left\|K_m\right\|_F
$$

Hence,
$$E_m = \left\|\widetilde Q_m\right\|_F ^2 + \left\|\widetilde K_m\right\|_F ^2 = \left\|Q_m\right\|_F ^2 + \left\| K_m\right\|_F ^2$$
\end{proof}
The pair-energy criterion can therefore be evaluated equivalently in the pre- or post-RoPE space. The implementation accumulates the joint Q/K second moment and ranks each pair by the sum of its two diagonal entries, which is proportional to this $E_m$. In pour implementations we use the post-RoPE values for the calibration. 

\paragraph{Connection to attention-output perturbation.}
Let
\[
Z=\frac{\widetilde Q\widetilde K^\top}{\sqrt{d_k}},
\qquad
Z_{\mathcal S}
=
\frac{\widetilde Q_{\mathcal S}\widetilde K_{\mathcal S}^\top}{\sqrt{d_k}}.
\]
The Jacobian of softmax at a probability vector $p$ is
$\operatorname{diag}(p)-pp^\top$, whose spectral norm is at most
$1/2$. Applying this result row-wise gives
\[
\|A-A_{\mathcal S}\|_F
\leq
\frac{1}{2}\|Z-Z_{\mathcal S}\|_F.
\]
Because $Z-Z_{\mathcal S}\in
\mathbb{R}^{S_{\mathrm{cal}}\times S_{\mathrm{cal}}}$,
\[
\begin{aligned}
\|Z-Z_{\mathcal S}\|_F
&\leq
\sqrt{S_{\mathrm{cal}}}\,
\|Z-Z_{\mathcal S}\|_2 \\
&=
\frac{\sqrt{S_{\mathrm{cal}}}}{\sqrt{d_k}}
\left\|
\widetilde Q \widetilde K^\top-\widetilde Q_{\mathcal S} \widetilde K_{\mathcal S}^\top
\right\|_2.
\end{aligned}
\]
For fixed values $V$,
\[
\begin{aligned}
\|AV-A_{\mathcal S}V\|_F
&\leq
\|A-A_{\mathcal S}\|_F\|V\|_2 \\
&\leq
\frac{\sqrt{S_{\mathrm{cal}}}\|V\|_2}
     {4\sqrt{d_k}}
\sum_{m\notin\mathcal S}E_m,
\end{aligned}
\]
where the final inequality follows from
Proposition~\ref{prop:pair_energy_bound}.

More generally, if the downstream calibration loss is locally
$L_{\mathrm{down}}$-Lipschitz with respect to the attention output,
then
\[
\left|
\mathcal L_{\mathrm{down}}(AV)
-
\mathcal L_{\mathrm{down}}(A_{\mathcal S}V)
\right|
\leq
L_{\mathrm{down}}
\frac{\sqrt{S_{\mathrm{cal}}}\|V\|_2}
     {4\sqrt{d_k}}
\sum_{m\notin\mathcal S}E_m.
\]
This is a conditional, calibration-dependent bound and does not imply
that energy selection globally minimizes language-model loss.

\subsection{Implementation Details}
\label{app:implementation}
\paragraph{Software and model execution.}
We implement RACE in PyTorch and Hugging Face Transformers. Our implementation
supports GPT-NeoX, Llama, Qwen2, Pythia, and Granite-style attention modules,
including both multi-head and grouped-query attention. Although all models reported in the main paper use full-RoPE or RoPE applied to all dimensions, our implementation supports partial RoPE as well by applying pairwise selection only within the rotary-encoded subspace while leaving any non-rotary dimensions unconstrained by the RoPE-pair requirement. Pre-trained models are loaded and
evaluated in BF16 on GPU, while calibration statistics and projection
construction use higher-precision accumulation during calibration. The folded compressed model
uses standard attention operations and does not require a custom CUDA or
Triton kernel. Unless otherwise stated, we use PyTorch scaled dot-product
attention (SDPA) for the compressed runtime. The implementation also supports
eager attention and FlashAttention-3 as shown below.

\begin{table}[t]
\caption{
\textbf{Comparison of methods that compress KV Cache along head dimension.}
RACE operates natively in the reduced dimensions without reconstruction.
See ~\ref{app:table1_derivation} for the derivations.
}
\label{tab:kernel_analysis}
\resizebox{\columnwidth}{!}{
\begin{tabular}{ccccccc}
\hline
\rowcolor[HTML]{EFEFEF} 
                                                                     & \textbf{\begin{tabular}[c]{@{}c@{}}Total Inference\\ TFLOPs\end{tabular}} & \textbf{\begin{tabular}[c]{@{}c@{}}TTFT\\ (ms)\end{tabular}} & \textbf{\begin{tabular}[c]{@{}c@{}}Prefill\\ T.put (tok/sec)\end{tabular}} & \textbf{\begin{tabular}[c]{@{}c@{}}Decode\\ T.put (tok/sec)\end{tabular}} & \textbf{\begin{tabular}[c]{@{}c@{}}Total\\ T.put (tok/sec)\end{tabular}} & \textbf{\begin{tabular}[c]{@{}c@{}}GPU\\ Memory (GB)\end{tabular}} \\ \hline
\textbf{Eager}                                                       & 3.973                                                                     & 41.051                                                       & 13937                                                                      & 45.186                                                                    & 224.44                                                                   & 5.884                                                              \\
\rowcolor[HTML]{EFEFEF} 
\textbf{SDPA}                                                        & 3.973                                                                     & 28.854                                                       & 17755.36                                                                   & 56.864                                                                    & 282.904                                                                  & 5.882                                                              \\
\textbf{Flash Attention 3} & 3.973                                                                     & 34.170                                                       & 15054                                                                      & 46.23                                                                     & 230.127                                                                  & 5.882                                                              \\ \hline
\end{tabular}
}
\end{table}

Furthermore, PALU rely on method-specific execution paths for efficient inference, including custom fused kernels. We were unable to reliably reproduce these optimized kernels in our hardware/software environment and therefore avoid presenting potentially misleading wall-clock comparisons. Instead, we compare algorithmic attention FLOPs across methods and measure RACE end-to-end using standard PyTorch attention backends. This separates implementation-specific kernel efficiency from the architectural distinction that RACE performs attention directly in the reduced dimensions without runtime reconstruction.  The original rotary-frequency indices and the model's RoPE scaling rule are
preserved; selected frequencies are not re-indexed as consecutive frequencies.
For grouped-query attention (GQA), all query heads associated with a given KV
head use the same selected key subspace. For partial RoPE models like Pythia models, the RoPE pairs are considered as pairs and the non-RoPE dimesnisons are paired with adjacent dimensions as in interleaved RoPE as shown in Figure \ref{fig:selection}.

\paragraph{Calibration in RACE.}

For each attention layer, we collect query and key activations after the
pretrained RoPE transformation. Our implementation also support calibration with pre-RoPE queries and keys and we also show that they are equivalent in \ref{app:proofs}. We preserve the physical rotary-pair layout of
the original architecture rather than assuming that adjacent tensor
coordinates necessarily form a rotary pair. For each rotary pair $m$, we
accumulate the joint activation-energy score as in Eq.\ref{eq:pair_energy}
and retain the $r_k/2$ highest-scoring pairs. Selection is performed
independently for each layer and attention head (or KV head under grouped-query
attention), subject to retaining both coordinates of every selected RoPE pair.

\begin{wrapfigure}{r}{0.65\linewidth}
    \centering
    \vspace{-0.2cm}
\includegraphics[width=0.98\linewidth]{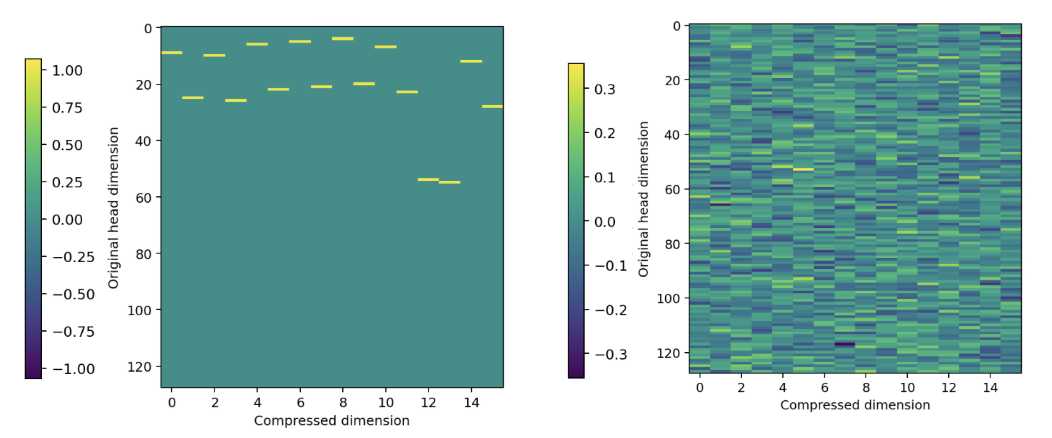}
    \caption{
    \textbf{Representative RACE $P_{QK}$ and $P_V$ matrices before folding.}
Left: Each column corresponds to one compressed Q/K coordinate and each row to an
original head coordinate. Nonzero entries indicate the coordinates retained by
 $P_{QK}$. Right: Each column corresponds to one compressed V coordinate and each row to an
original head coordinate. 
    }
    \label{fig:selection}
    \vspace{-0.3cm}
\end{wrapfigure}

 Note that models like Llama, Qwen, and Granite have full RoPE where all dimensions have RoPE and in such cases the sequential selection shown in \ref{fig:selection} Left will not be present. Such models also have sequential RoPE $[x_0, x_{d/2}], [x_1, x_{d/2+1}], \dots$. Our implementations supports both sequential and interleaved $[x_0, x_{1}], [x_2, x_3], \dots$ RoPE models. The values are not subject to RoPE. We employ $P_v$, the right signular vectors of V for the compression. This do not follow selection as shown in Figure \ref{fig:selection} Right.

\paragraph{Optional post-compression recovery.}
\label{app:recovery}
RACE is a post-training compression method and can be evaluated immediately after constructing and folding the reduced Q/K and V subspaces into the pretrained attention projections. At aggressive compression ratios, however, this discrete change in attention dimensionality can perturb the pretrained representation enough that a short optimization stage is highly beneficial. Although optional, all results in the main paper involve this recovery. Following folding, we use the original uncompressed checkpoint as a frozen teacher and optimize the compressed model using KL distillation on FineWeb data. The implementation uses a temperature of \(T=2\) and minimizes the divergence between the teacher and compressed-model output distributions. We train for 250 optimizer steps with sequences of length 512, a micro-batch size of 2, and gradient accumulation over 4 micro-batches, corresponding to approximately \(1.0\)M training tokens.  

Figure \ref{fig:FT_vary} separates the effect of this recovery from the compression procedure itself by reporting both the directly folded models (“w/o FT”) and their recovered counterparts. At 75\% and 87.5\% KV-cache retention, the directly compressed models are already relatively close to their recovered versions, indicating that moderate dimensionality reduction introduces limited disruption. In contrast, at 50\% retention, both RACE and RAP experience substantially larger degradation without recovery, particularly in perplexity. The optimization stage is not necessary to define or execute the compressed model. RACE produces a valid reconstruction-free reduced-dimensional attention layer immediately after folding. Rather, recovery provides an optional mechanism for adapting the pretrained weights when the dimensionality change becomes sufficiently aggressive. Moreover, the same recovery procedure also substantially improves RAP, suggesting that much of the recoverable degradation at aggressive compression reflects adaptation to the compressed representation rather than a method-specific training advantage. 

\begin{figure}[t]
    \centering

    \begin{subfigure}[t]{0.24\textwidth}
        \centering
        \includegraphics[width=\linewidth]{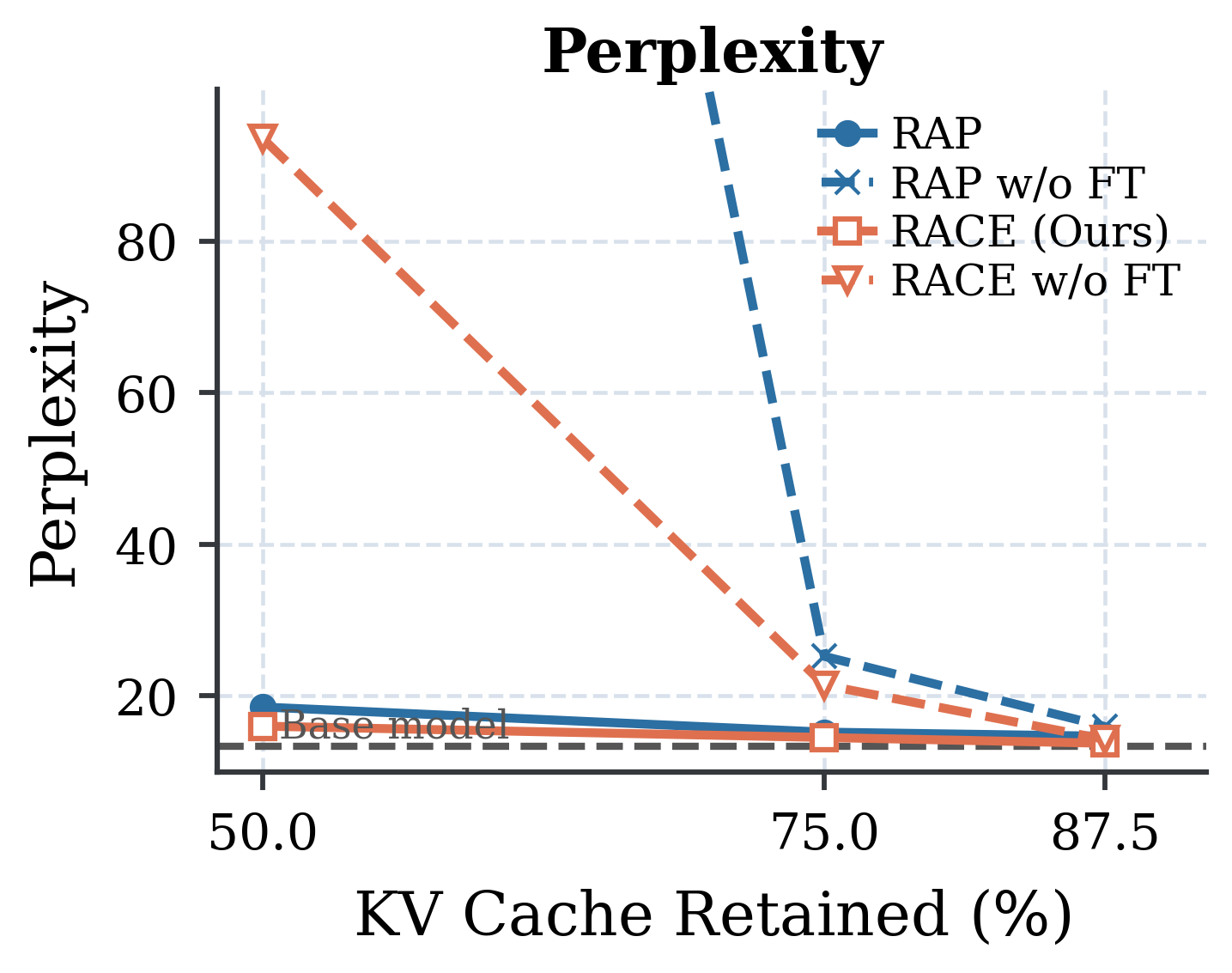}
        \caption{Llama-3.2-3B-Instruct Perplexity}
        \label{fig:result_a}
    \end{subfigure}
    \hfill
    \begin{subfigure}[t]{0.24\textwidth}
        \centering
        \includegraphics[width=\linewidth]{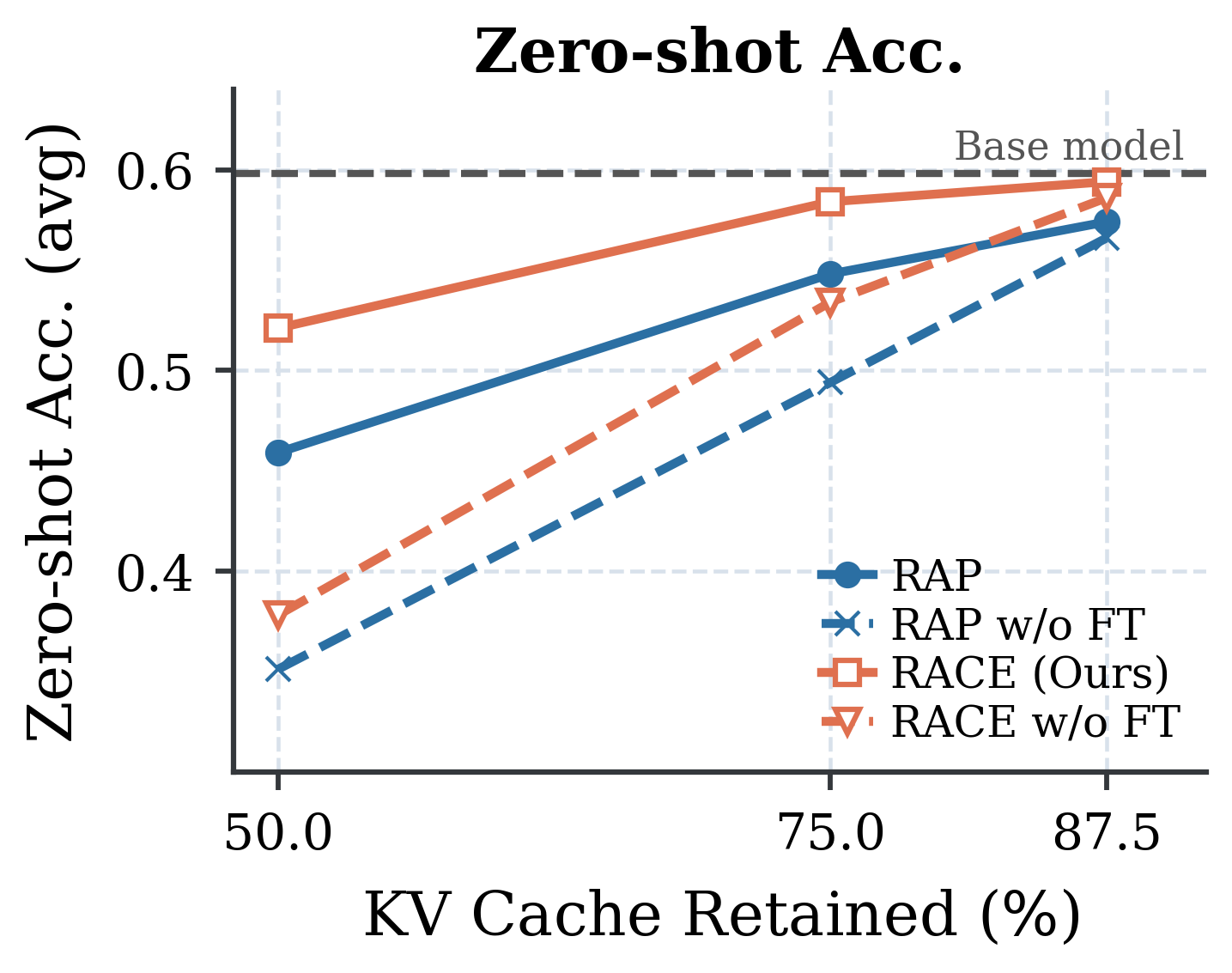}
        \caption{Llama-3.2-3B-Instruct Average zero-shot accuracy}
        \label{fig:result_b}
    \end{subfigure}
    \hfill
    \begin{subfigure}[t]{0.24\textwidth}
        \centering
        \includegraphics[width=\linewidth]{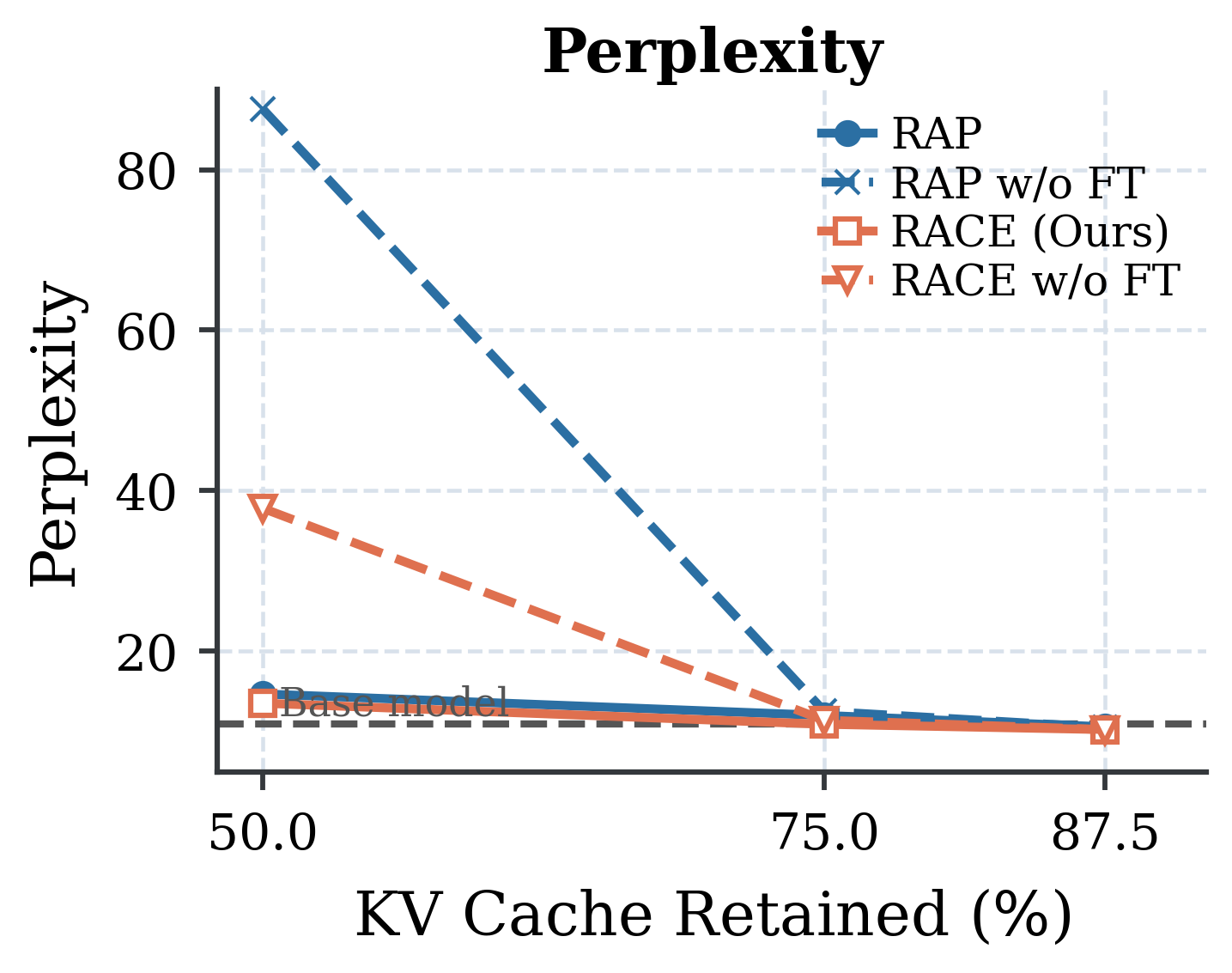}
        \caption{Qwen2.5 7B Perplexity}
        \label{fig:result_c}
    \end{subfigure} 
    \hfill
    \begin{subfigure}[t]{0.24\textwidth}
        \centering
        \includegraphics[width=\linewidth]{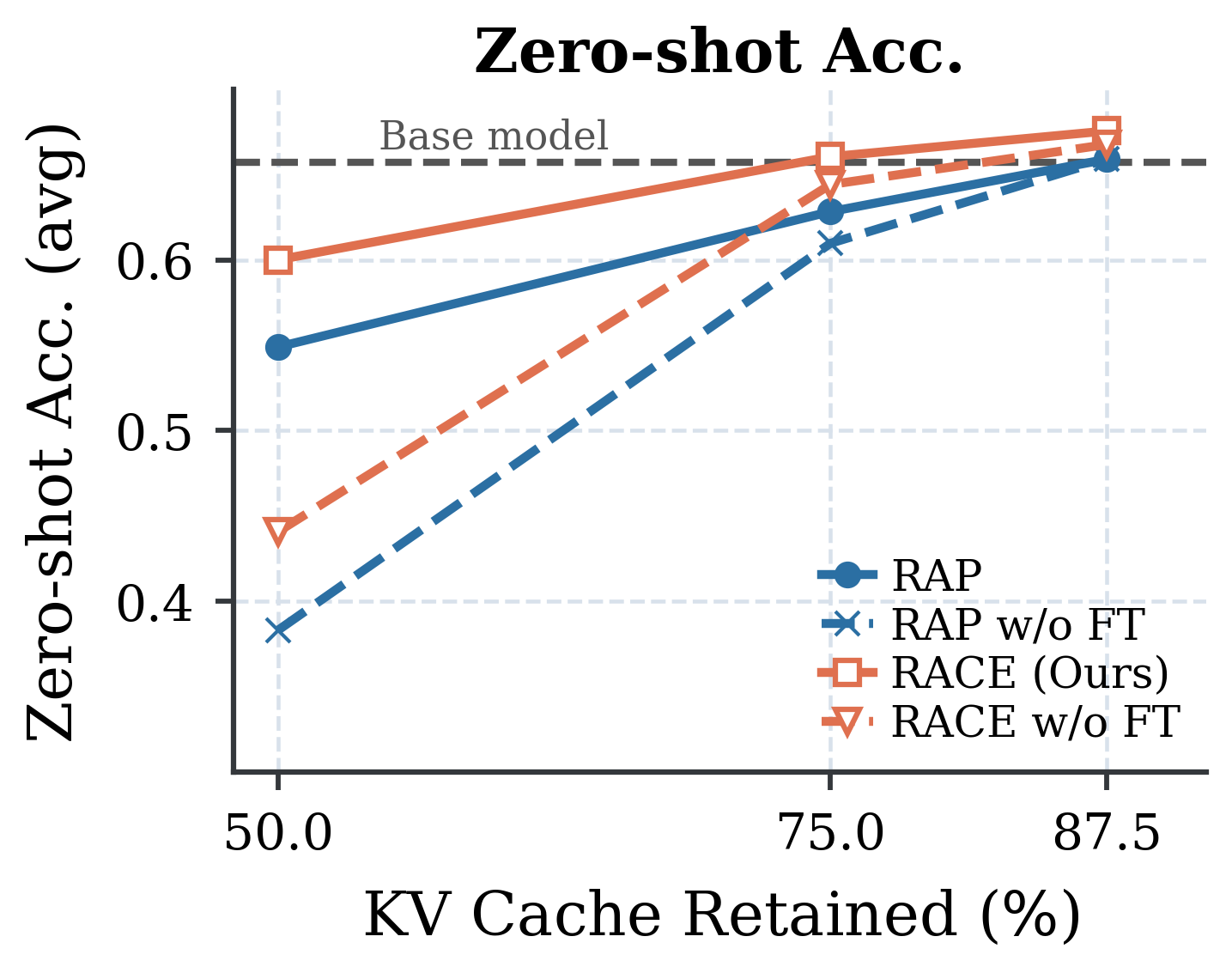}
        \caption{Qwen2.5 7B Average zero-shot accuracy}
        \label{fig:result_c}
    \end{subfigure} 
    \caption{\textbf{Effect of optional post-compression recovery.} Perplexity and average zero-shot accuracy before and after a short recovery stage for RACE and RAP at different KV-cache retention ratios. “w/o FT” denotes the folded compressed checkpoint evaluated directly, while the solid curves apply post-fold KL distillation from the uncompressed model. Recovery has little effect at mild compression but substantially restores performance at 50\% retention, particularly in perplexity.}
    \label{fig:FT_vary}

\end{figure}

\paragraph{Systems Details}
All systems experiments were conducted on a single NVIDIA H100 80GB HBM3 GPU. The machine used NVIDIA driver 590.48.01 with CUDA 13.1, and all reported memory, latency, throughput, and FLOP-related measurements were collected on this hardware configuration. Unless otherwise noted, experiments were run on a single GPU without model parallelism using SDPA kernels, so comparisons between the base model and compressed variants reflect changes induced by the compression method rather than differences in hardware allocation.

\begin{figure}[t]
    \centering

    \begin{subfigure}[t]{0.32\textwidth}
        \centering
        \includegraphics[width=\linewidth]{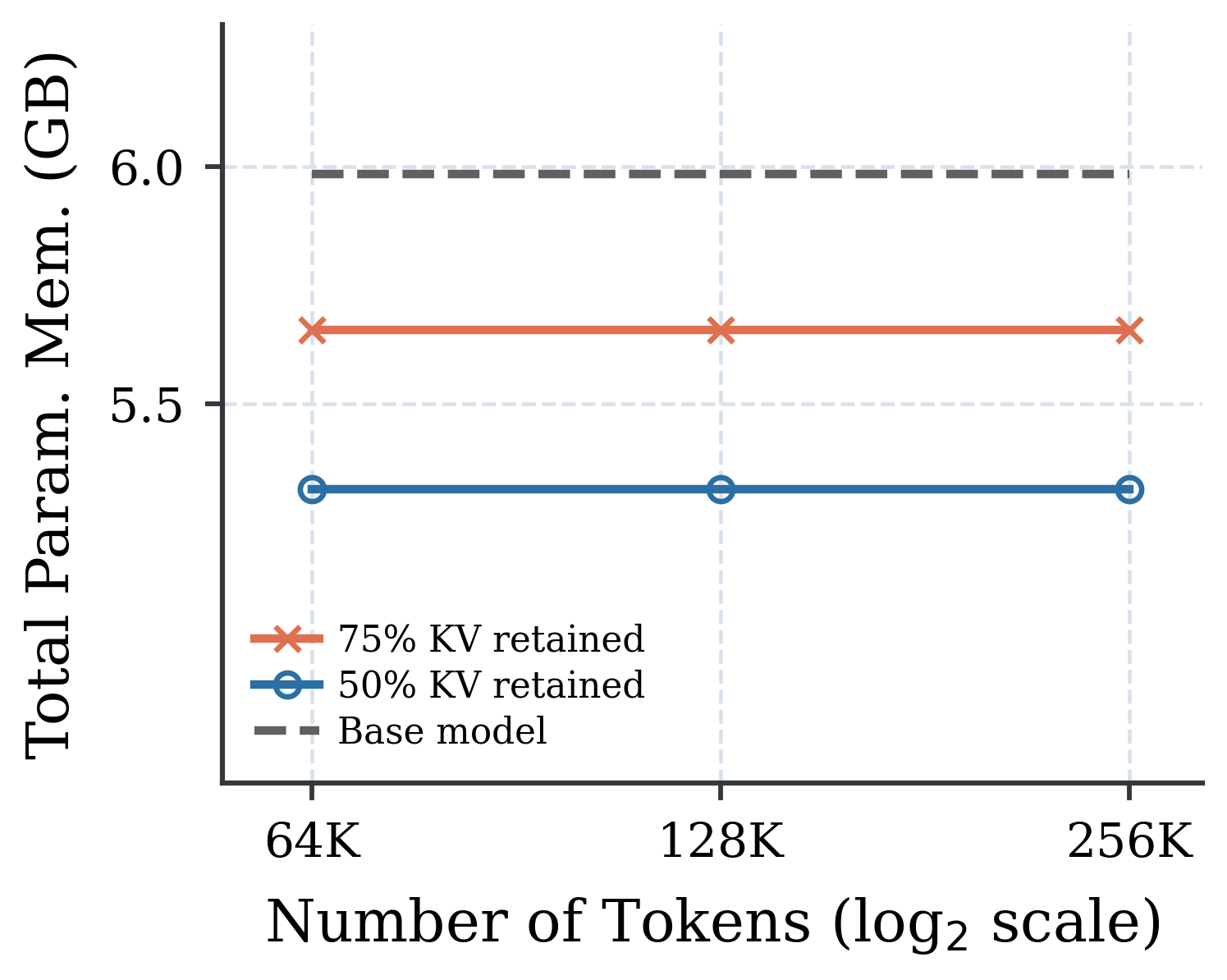}
        \caption{Total Parameter Memory}
        \label{fig:result_a}
    \end{subfigure}
    \hfill
    \begin{subfigure}[t]{0.32\textwidth}
        \centering
        \includegraphics[width=\linewidth]{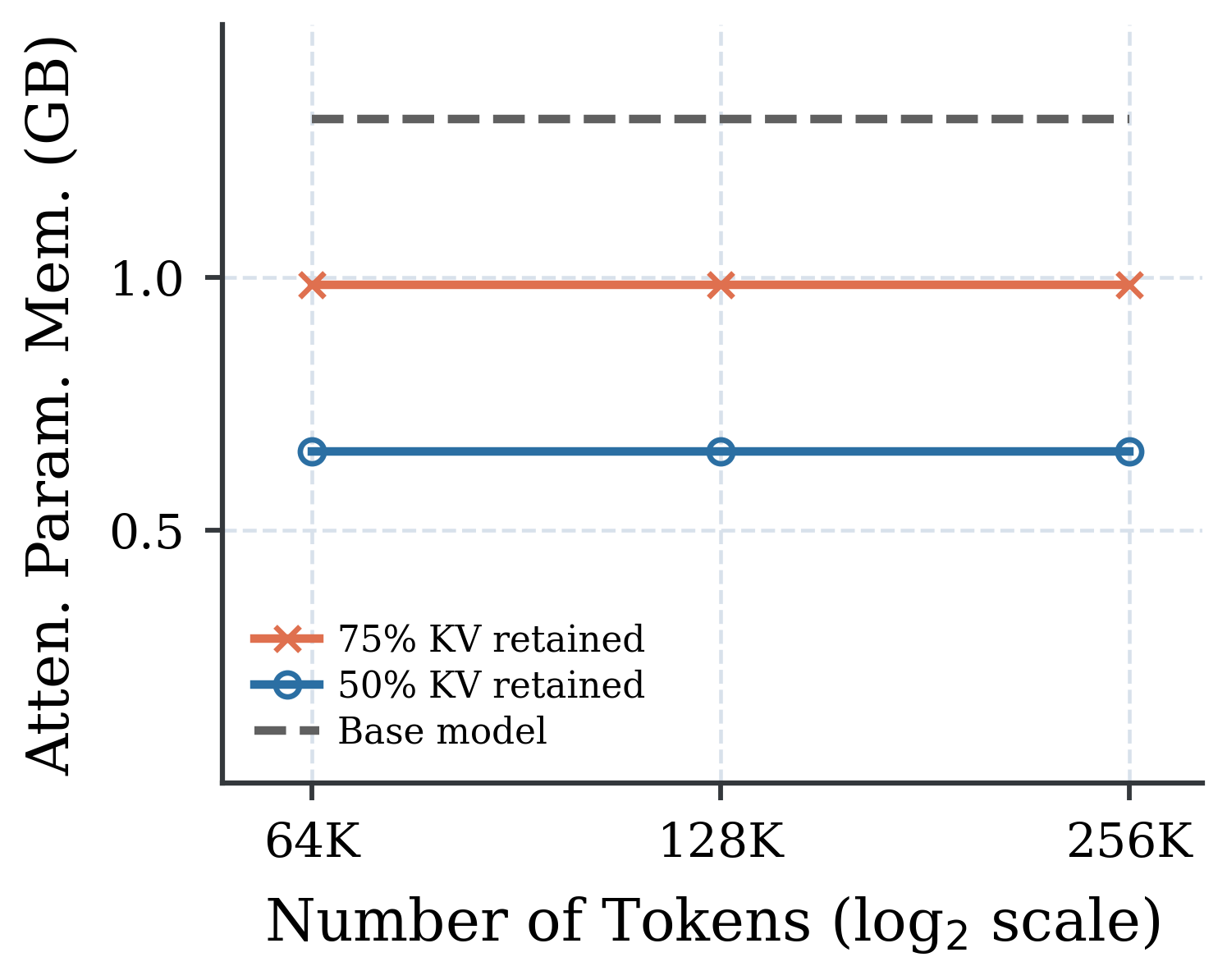}
        \caption{Attention Parameter Memory}
        \label{fig:result_b}
    \end{subfigure}
    \hfill
    \begin{subfigure}[t]{0.32\textwidth}
        \centering
        \includegraphics[width=\linewidth]{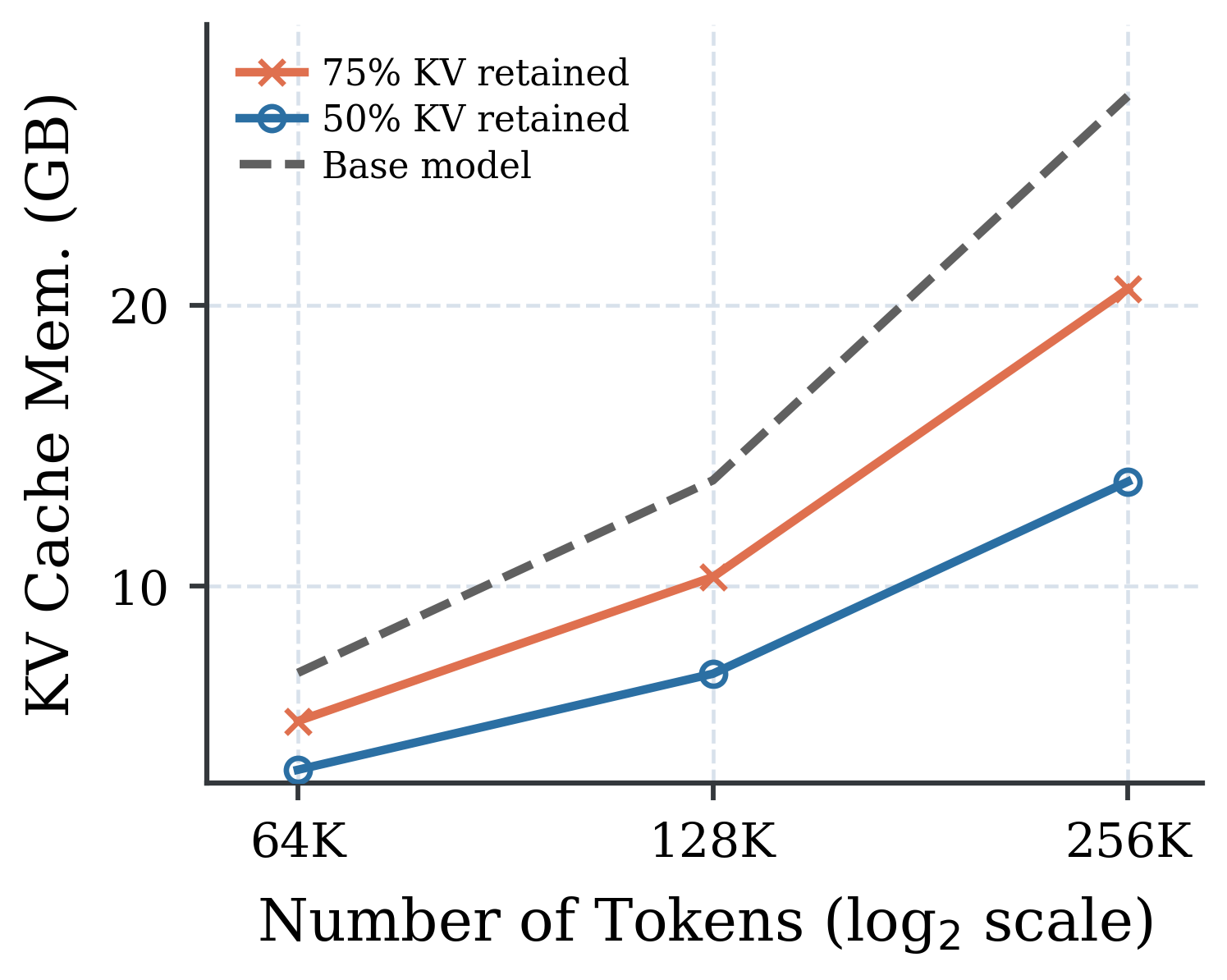}
        \caption{KV Cache Memory}
        \label{fig:result_c}
    \end{subfigure} 

    \caption{Memory scaling of RACE with context length. Total model parameter memory (left), attention-parameter memory (center), and KV-cache memory (right) for the base model and RACE at 75\% and 50\% KV retention across 64K–256K-token contexts. Parameter memory is independent of sequence length and decreases after folding the reduced-dimensional projections, while KV-cache memory scales  with context length; RACE correspondingly reduces both the fixed attention-parameter footprint and the context-dependent cache footprint.}
    \label{fig:mems}

\end{figure}

\subsection{Results}
\label{app:results}
\begin{table}[h]
\caption{Detailed quality results for head-dimension compression.
We report WikiText perplexity and per-task zero-shot accuracy across the
evaluated model families and KV-cache retention ratios. The final column reports
the average zero-shot accuracy across tasks. For each KV Cache compression level per model, we highlight the best performing perplexity and average zero-shot accuracy. RACE is comparable to reconstruction-based methods while not requiring the reconstruction or the additional FLOPs.}
\label{tab:ZS_full}
\resizebox{\columnwidth}{!}{%
\begin{tabular}{ccccccccccccccccccccc}
\hline
\rowcolor[HTML]{EFEFEF} 
\cellcolor[HTML]{EFEFEF}                                                                   & \cellcolor[HTML]{EFEFEF}                                                                                              & \cellcolor[HTML]{EFEFEF}                                                                             & \cellcolor[HTML]{EFEFEF}                                                                            & \cellcolor[HTML]{EFEFEF}                                                                                & \textbf{Completion}           & \multicolumn{4}{c}{\cellcolor[HTML]{EFEFEF}\textbf{Reasoning}}                                                                & \multicolumn{3}{c}{\cellcolor[HTML]{EFEFEF}\textbf{Knowledge QA}}                             & \multicolumn{2}{c}{\cellcolor[HTML]{EFEFEF}\textbf{Understanding}} & \textbf{Truth}                                                 & \multicolumn{4}{c}{\cellcolor[HTML]{EFEFEF}\textbf{Specialized Knowledge}}                                                                                                                                                                                         & \cellcolor[HTML]{EFEFEF}                               \\
\rowcolor[HTML]{EFEFEF} 
\multirow{-2}{*}{\cellcolor[HTML]{EFEFEF}\textbf{Method}}                                  & \multirow{-2}{*}{\cellcolor[HTML]{EFEFEF}\textbf{\begin{tabular}[c]{@{}c@{}}KV \\ Cache\end{tabular}}} & \multirow{-2}{*}{\cellcolor[HTML]{EFEFEF}\textbf{\begin{tabular}[c]{@{}c@{}}QK\\ Dim.\end{tabular}}} & \multirow{-2}{*}{\cellcolor[HTML]{EFEFEF}\textbf{\begin{tabular}[c]{@{}c@{}}V\\ Dim.\end{tabular}}} & \multirow{-2}{*}{\cellcolor[HTML]{EFEFEF}\textbf{\begin{tabular}[c]{@{}c@{}}Wiki\\ (ppl)\end{tabular}}} & \textbf{lambada}              & \textbf{hella}                & \textbf{piqa}                 & \textbf{wino}                 & \textbf{arc\_e}               & \textbf{sciq}                 & \textbf{obqa}                 & \textbf{arc\_c}               & \textbf{boolq}                   & \textbf{race}                   & \textbf{\begin{tabular}[c]{@{}c@{}}truth\\ fulqa\end{tabular}} & \textbf{\begin{tabular}[c]{@{}c@{}}college\\ \_CS\end{tabular}} & \textbf{\begin{tabular}[c]{@{}c@{}}college\_\\ Phy\end{tabular}} & \textbf{\begin{tabular}[c]{@{}c@{}}prof\\ \_law\end{tabular}} & \textbf{\begin{tabular}[c]{@{}c@{}}HS\_\\ psych\end{tabular}} & \multirow{-2}{*}{\cellcolor[HTML]{EFEFEF}\textbf{Avg}} \\ \hline
\multicolumn{21}{c}{\textbf{Qwen2.5 7B}}                                                                                                                                                                                                                                                                                                                                                                                                                                                                                                                                                                                                                                                                                                                                                                                                                                                                                                                                                                                                                                                                                                                                                                                                                                      \\ \hline
\textbf{Base}                                                                              & \cellcolor[HTML]{EFEFEF}100\%                                                                                         & \cellcolor[HTML]{EFEFEF}128                                                                      & \cellcolor[HTML]{EFEFEF}128                                                                     & \cellcolor[HTML]{EFEFEF}10.976                                                                          & \cellcolor[HTML]{EFEFEF}0.706 & \cellcolor[HTML]{EFEFEF}0.795 & \cellcolor[HTML]{EFEFEF}0.784 & \cellcolor[HTML]{EFEFEF}0.707 & \cellcolor[HTML]{EFEFEF}0.759 & \cellcolor[HTML]{EFEFEF}0.960 & \cellcolor[HTML]{EFEFEF}0.460 & \cellcolor[HTML]{EFEFEF}0.495 & \cellcolor[HTML]{EFEFEF}0.845    & \cellcolor[HTML]{EFEFEF}0.418   & \cellcolor[HTML]{EFEFEF}0.524                                  & \cellcolor[HTML]{EFEFEF}0.700                                   & \cellcolor[HTML]{EFEFEF}0.480                                    & \cellcolor[HTML]{EFEFEF}0.529                                 & \cellcolor[HTML]{EFEFEF}0.897                                 & \cellcolor[HTML]{EFEFEF}0.671                          \\ \hline
                                                                                           & 87.50\%                                                                                                               & 112                                                                                              & 112                                                                                             & 11.438                                                                                                  & 0.696                         & 0.791                         & 0.771                         & 0.703                         & 0.772                         & 0.954                         & 0.456                         & 0.504                         & 0.839                            & 0.421                           & 0.457                                                          & 0.710                                                           & 0.500                                                            & 0.511                                                         & 0.881                                                         & 0.664                                                  \\
                                                                                           & \cellcolor[HTML]{EFEFEF}87.50\%                                                                                       & \cellcolor[HTML]{EFEFEF}96                                                                       & \cellcolor[HTML]{EFEFEF}128                                                                     & \cellcolor[HTML]{EFEFEF}11.080                                                                          & \cellcolor[HTML]{EFEFEF}0.693 & \cellcolor[HTML]{EFEFEF}0.797 & \cellcolor[HTML]{EFEFEF}0.780 & \cellcolor[HTML]{EFEFEF}0.711 & \cellcolor[HTML]{EFEFEF}0.782 & \cellcolor[HTML]{EFEFEF}0.956 & \cellcolor[HTML]{EFEFEF}0.462 & \cellcolor[HTML]{EFEFEF}0.514 & \cellcolor[HTML]{EFEFEF}0.844    & \cellcolor[HTML]{EFEFEF}0.414   & \cellcolor[HTML]{EFEFEF}0.516                                  & \cellcolor[HTML]{EFEFEF}0.700                                   & \cellcolor[HTML]{EFEFEF}0.490                                    & \cellcolor[HTML]{EFEFEF}0.512                                 & \cellcolor[HTML]{EFEFEF}0.883                                 & \cellcolor[HTML]{EFEFEF}0.670                          \\
                                                                                           & 75\%                                                                                                                  & 96                                                                                               & 96                                                                                              & 12.290                                                                                                  & 0.687                         & 0.782                         & 0.756                         & 0.703                         & 0.777                         & 0.953                         & 0.460                         & 0.508                         & 0.838                            & 0.414                           & 0.411                                                          & 0.660                                                           & 0.451                                                            & 0.486                                                         & 0.868                                                         & 0.650                                                  \\
                                                                                           & \cellcolor[HTML]{EFEFEF}75\%                                                                                          & \cellcolor[HTML]{EFEFEF}64                                                                       & \cellcolor[HTML]{EFEFEF}128                                                                     & \cellcolor[HTML]{EFEFEF}11.325                                                                          & \cellcolor[HTML]{EFEFEF}0.695 & \cellcolor[HTML]{EFEFEF}0.793 & \cellcolor[HTML]{EFEFEF}0.774 & \cellcolor[HTML]{EFEFEF}0.706 & \cellcolor[HTML]{EFEFEF}0.790 & \cellcolor[HTML]{EFEFEF}0.964 & \cellcolor[HTML]{EFEFEF}0.470 & \cellcolor[HTML]{EFEFEF}0.525 & \cellcolor[HTML]{EFEFEF}0.853    & \cellcolor[HTML]{EFEFEF}0.410   & \cellcolor[HTML]{EFEFEF}0.532                                  & \cellcolor[HTML]{EFEFEF}0.660                                   & \cellcolor[HTML]{EFEFEF}0.471                                    & \cellcolor[HTML]{EFEFEF}0.521                                 & \cellcolor[HTML]{EFEFEF}0.862                                 & \cellcolor[HTML]{EFEFEF}\textbf{0.668}                          \\
\multirow{-5}{*}{\textbf{PALU}}                                                            & 50\%                                                                                                                  & 64                                                                                               & 64                                                                                              & 17.240                                                                                                  & 0.639                         & 0.746                         & 0.662                         & 0.683                         & 0.703                         & 0.941                         & 0.392                         & 0.452                         & 0.789                            & 0.393                           & 0.430                                                          & 0.600                                                           & 0.402                                                            & 0.408                                                         & 0.828                                                         & \textbf{0.605}                                                  \\ \hline
                                                                                           & \cellcolor[HTML]{EFEFEF}87.50\%                                                                                       & \cellcolor[HTML]{EFEFEF}112                                                                      & \cellcolor[HTML]{EFEFEF}112                                                                     & \cellcolor[HTML]{EFEFEF}10.178                                                                          & \cellcolor[HTML]{EFEFEF}0.702 & \cellcolor[HTML]{EFEFEF}0.797 & \cellcolor[HTML]{EFEFEF}0.783 & \cellcolor[HTML]{EFEFEF}0.723 & \cellcolor[HTML]{EFEFEF}0.790 & \cellcolor[HTML]{EFEFEF}0.961 & \cellcolor[HTML]{EFEFEF}0.458 & \cellcolor[HTML]{EFEFEF}0.538 & \cellcolor[HTML]{EFEFEF}0.838    & \cellcolor[HTML]{EFEFEF}0.405   & \cellcolor[HTML]{EFEFEF}0.500                                  & \cellcolor[HTML]{EFEFEF}0.630                                   & \cellcolor[HTML]{EFEFEF}0.471                                    & \cellcolor[HTML]{EFEFEF}0.507                                 & \cellcolor[HTML]{EFEFEF}0.899                                 & \cellcolor[HTML]{EFEFEF}0.667                          \\
                                                                                           & 87.50\%                                                                                                               & 96                                                                                               & 128                                                                                             & \textbf{9.602}                                                                                                   & 0.713                         & 0.803                         & 0.792                         & 0.729                         & 0.757                         & 0.951                         & 0.460                         & 0.506                         & 0.843                            & 0.415                           & 0.561                                                          & 0.690                                                           & 0.490                                                            & 0.533                                                         & 0.892                                                         & \textbf{0.676}                                                  \\
                                                                                           & \cellcolor[HTML]{EFEFEF}75\%                                                                                          & \cellcolor[HTML]{EFEFEF}96                                                                       & \cellcolor[HTML]{EFEFEF}96                                                                      & \cellcolor[HTML]{EFEFEF}11.043                                                                          & \cellcolor[HTML]{EFEFEF}0.666 & \cellcolor[HTML]{EFEFEF}0.793 & \cellcolor[HTML]{EFEFEF}0.768 & \cellcolor[HTML]{EFEFEF}0.713 & \cellcolor[HTML]{EFEFEF}0.757 & \cellcolor[HTML]{EFEFEF}0.953 & \cellcolor[HTML]{EFEFEF}0.454 & \cellcolor[HTML]{EFEFEF}0.497 & \cellcolor[HTML]{EFEFEF}0.811    & \cellcolor[HTML]{EFEFEF}0.398   & \cellcolor[HTML]{EFEFEF}0.500                                  & \cellcolor[HTML]{EFEFEF}0.590                                   & \cellcolor[HTML]{EFEFEF}0.461                                    & \cellcolor[HTML]{EFEFEF}0.480                                 & \cellcolor[HTML]{EFEFEF}0.862                                 & \cellcolor[HTML]{EFEFEF}0.647                          \\
                                                                                           & 75\%                                                                                                                  & 64                                                                                               & 128                                                                                             & \textbf{10.080}                                                                                                  & 0.715                         & 0.804                         & 0.783                         & 0.702                         & 0.722                         & 0.949                         & 0.462                         & 0.495                         & 0.829                            & 0.398                           & 0.545                                                          & 0.650                                                           & 0.431                                                            & 0.497                                                         & 0.884                                                         & 0.658                                                  \\
\multirow{-5}{*}{\textbf{\begin{tabular}[c]{@{}c@{}}Eigen\\ Attention\end{tabular}}}       & \cellcolor[HTML]{EFEFEF}50\%                                                                                          & \cellcolor[HTML]{EFEFEF}64                                                                       & \cellcolor[HTML]{EFEFEF}64                                                                      & \cellcolor[HTML]{EFEFEF}21.050                                                                          & \cellcolor[HTML]{EFEFEF}0.376 & \cellcolor[HTML]{EFEFEF}0.747 & \cellcolor[HTML]{EFEFEF}0.647 & \cellcolor[HTML]{EFEFEF}0.619 & \cellcolor[HTML]{EFEFEF}0.617 & \cellcolor[HTML]{EFEFEF}0.897 & \cellcolor[HTML]{EFEFEF}0.410 & \cellcolor[HTML]{EFEFEF}0.398 & \cellcolor[HTML]{EFEFEF}0.636    & \cellcolor[HTML]{EFEFEF}0.369   & \cellcolor[HTML]{EFEFEF}0.429                                  & \cellcolor[HTML]{EFEFEF}0.330                                   & \cellcolor[HTML]{EFEFEF}0.275                                    & \cellcolor[HTML]{EFEFEF}0.374                                 & \cellcolor[HTML]{EFEFEF}0.659                                 & \cellcolor[HTML]{EFEFEF}0.519                          \\ \hline
                                                                                           & 87.50\%                                                                                                               & 96                                                                                               & 128                                                                                             & 30.547                                                                                                  & 0.294                         & 0.733                         & 0.634                         & 0.547                         & 0.669                         & 0.847                         & 0.330                         & 0.416                         & 0.650                            & 0.321                           & 0.390                                                          & 0.210                                                           & 0.314                                                            & 0.275                                                         & 0.428                                                         & 0.470                                                  \\
\multirow{-2}{*}{\textbf{\begin{tabular}[c]{@{}c@{}}Thin Keys\\ Full Values\end{tabular}}} & \cellcolor[HTML]{EFEFEF}75\%                                                                                          & \cellcolor[HTML]{EFEFEF}64                                                                       & \cellcolor[HTML]{EFEFEF}128                                                                     & \cellcolor[HTML]{EFEFEF}50.975                                                                          & \cellcolor[HTML]{EFEFEF}0.284 & \cellcolor[HTML]{EFEFEF}0.737 & \cellcolor[HTML]{EFEFEF}0.611 & \cellcolor[HTML]{EFEFEF}0.544 & \cellcolor[HTML]{EFEFEF}0.687 & \cellcolor[HTML]{EFEFEF}0.834 & \cellcolor[HTML]{EFEFEF}0.376 & \cellcolor[HTML]{EFEFEF}0.436 & \cellcolor[HTML]{EFEFEF}0.611    & \cellcolor[HTML]{EFEFEF}0.298   & \cellcolor[HTML]{EFEFEF}0.419                                  & \cellcolor[HTML]{EFEFEF}0.250                                   & \cellcolor[HTML]{EFEFEF}0.314                                    & \cellcolor[HTML]{EFEFEF}0.250                                 & \cellcolor[HTML]{EFEFEF}0.363                                 & \cellcolor[HTML]{EFEFEF}0.467                          \\ \hline
                                                                                           & 87.50\%                                                                                                               & 112                                                                                              & 112                                                                                             & 10.743                                                                                                  & 0.691                         & 0.794                         & 0.782                         & 0.708                         & 0.739                         & 0.961                         & 0.450                         & 0.466                         & 0.843                            & 0.393                           & 0.517                                                          & 0.660                                                           & 0.441                                                            & 0.500                                                         & 0.884                                                         & 0.655                                                  \\
                                                                                           & \cellcolor[HTML]{EFEFEF}87.50\%                                                                                       & \cellcolor[HTML]{EFEFEF}96                                                                       & \cellcolor[HTML]{EFEFEF}128                                                                     & \cellcolor[HTML]{EFEFEF}10.594                                                                          & \cellcolor[HTML]{EFEFEF}0.705 & \cellcolor[HTML]{EFEFEF}0.794 & \cellcolor[HTML]{EFEFEF}0.784 & \cellcolor[HTML]{EFEFEF}0.707 & \cellcolor[HTML]{EFEFEF}0.753 & \cellcolor[HTML]{EFEFEF}0.963 & \cellcolor[HTML]{EFEFEF}0.432 & \cellcolor[HTML]{EFEFEF}0.482 & \cellcolor[HTML]{EFEFEF}0.854    & \cellcolor[HTML]{EFEFEF}0.407   & \cellcolor[HTML]{EFEFEF}0.528                                  & \cellcolor[HTML]{EFEFEF}0.680                                   & \cellcolor[HTML]{EFEFEF}0.402                                    & \cellcolor[HTML]{EFEFEF}0.507                                 & \cellcolor[HTML]{EFEFEF}0.888                                 & \cellcolor[HTML]{EFEFEF}0.659                          \\
                                                                                           & 75\%                                                                                                                  & 96                                                                                               & 96                                                                                              & 12.023                                                                                                  & 0.648                         & 0.788                         & 0.769                         & 0.680                         & 0.701                         & 0.963                         & 0.438                         & 0.455                         & 0.806                            & 0.391                           & 0.515                                                          & 0.600                                                           & 0.412                                                            & 0.450                                                         & 0.809                                                         & 0.628                                                  \\
                                                                                           & \cellcolor[HTML]{EFEFEF}75\%                                                                                          & \cellcolor[HTML]{EFEFEF}64                                                                       & \cellcolor[HTML]{EFEFEF}128                                                                     & \cellcolor[HTML]{EFEFEF}12.011                                                                          & \cellcolor[HTML]{EFEFEF}0.630 & \cellcolor[HTML]{EFEFEF}0.789 & \cellcolor[HTML]{EFEFEF}0.769 & \cellcolor[HTML]{EFEFEF}0.611 & \cellcolor[HTML]{EFEFEF}0.762 & \cellcolor[HTML]{EFEFEF}0.963 & \cellcolor[HTML]{EFEFEF}0.442 & \cellcolor[HTML]{EFEFEF}0.493 & \cellcolor[HTML]{EFEFEF}0.774    & \cellcolor[HTML]{EFEFEF}0.389   & \cellcolor[HTML]{EFEFEF}0.498                                  & \cellcolor[HTML]{EFEFEF}0.680                                   & \cellcolor[HTML]{EFEFEF}0.401                                    & \cellcolor[HTML]{EFEFEF}0.402                                 & \cellcolor[HTML]{EFEFEF}0.741                                 & \cellcolor[HTML]{EFEFEF}0.623                          \\
\multirow{-5}{*}{\textbf{RAP}}                                                             & 50\%                                                                                                                  & 64                                                                                               & 64                                                                                              & 14.691                                                                                                  & 0.485                         & 0.747                         & 0.707                         & 0.605                         & 0.758                         & 0.948                         & 0.396                         & 0.444                         & 0.758                            & 0.341                           & 0.403                                                          & 0.430                                                           & 0.303                                                            & 0.305                                                         & 0.606                                                         & 0.549                                                  \\ \hline
                                                                                           & \cellcolor[HTML]{EFEFEF}87.50\%                                                                                       & \cellcolor[HTML]{EFEFEF}112                                                                      & \cellcolor[HTML]{EFEFEF}112                                                                     & \cellcolor[HTML]{EFEFEF}10.254                                                                          & \cellcolor[HTML]{EFEFEF}0.704 & \cellcolor[HTML]{EFEFEF}0.794 & \cellcolor[HTML]{EFEFEF}0.777 & \cellcolor[HTML]{EFEFEF}0.724 & \cellcolor[HTML]{EFEFEF}0.817 & \cellcolor[HTML]{EFEFEF}0.967 & \cellcolor[HTML]{EFEFEF}0.462 & \cellcolor[HTML]{EFEFEF}0.538 & \cellcolor[HTML]{EFEFEF}0.849    & \cellcolor[HTML]{EFEFEF}0.411   & \cellcolor[HTML]{EFEFEF}0.507                                  & \cellcolor[HTML]{EFEFEF}0.670                                   & \cellcolor[HTML]{EFEFEF}0.510                                    & \cellcolor[HTML]{EFEFEF}0.518                                 & \cellcolor[HTML]{EFEFEF}0.886                                 & \cellcolor[HTML]{EFEFEF}\textbf{0.676}                          \\
                                                                                           & 87.50\%                                                                                                               & 96                                                                                               & 128                                                                                             & 10.259                                                                                                  & 0.715                         & 0.799                         & 0.781                         & 0.707                         & 0.794                         & 0.961                         & 0.470                         & 0.517                         & 0.846                            & 0.409                           & 0.524                                                          & 0.680                                                           & 0.490                                                            & 0.529                                                         & 0.884                                                         & 0.674                                                  \\
                                                                                           & \cellcolor[HTML]{EFEFEF}75\%                                                                                          & \cellcolor[HTML]{EFEFEF}96                                                                       & \cellcolor[HTML]{EFEFEF}96                                                                      & \cellcolor[HTML]{EFEFEF}10.968                                                                          & \cellcolor[HTML]{EFEFEF}0.668 & \cellcolor[HTML]{EFEFEF}0.794 & \cellcolor[HTML]{EFEFEF}0.765 & \cellcolor[HTML]{EFEFEF}0.722 & \cellcolor[HTML]{EFEFEF}0.838 & \cellcolor[HTML]{EFEFEF}0.963 & \cellcolor[HTML]{EFEFEF}0.456 & \cellcolor[HTML]{EFEFEF}0.562 & \cellcolor[HTML]{EFEFEF}0.825    & \cellcolor[HTML]{EFEFEF}0.392   & \cellcolor[HTML]{EFEFEF}0.480                                  & \cellcolor[HTML]{EFEFEF}0.630                                   & \cellcolor[HTML]{EFEFEF}0.451                                    & \cellcolor[HTML]{EFEFEF}0.481                                 & \cellcolor[HTML]{EFEFEF}0.872                                 & \cellcolor[HTML]{EFEFEF}0.660                          \\
                                                                                           & 75\%                                                                                                                  & 64                                                                                               & 128                                                                                             & 10.916                                                                                                  & 0.690                         & 0.795                         & 0.764                         & 0.687                         & 0.789                         & 0.957                         & 0.454                         & 0.510                         & 0.832                            & 0.394                           & 0.521                                                          & 0.640                                                           & 0.520                                                            & 0.495                                                         & 0.858                                                         & 0.661                                                  \\
\multirow{-5}{*}{\textbf{RACE (Ours)}}                                                            & \cellcolor[HTML]{EFEFEF}50\%                                                                                          & \cellcolor[HTML]{EFEFEF}64                                                                       & \cellcolor[HTML]{EFEFEF}64                                                                      & \cellcolor[HTML]{EFEFEF}\textbf{13.538}                                                                          & \cellcolor[HTML]{EFEFEF}0.562 & \cellcolor[HTML]{EFEFEF}0.785 & \cellcolor[HTML]{EFEFEF}0.713 & \cellcolor[HTML]{EFEFEF}0.656 & \cellcolor[HTML]{EFEFEF}0.748 & \cellcolor[HTML]{EFEFEF}0.944 & \cellcolor[HTML]{EFEFEF}0.438 & \cellcolor[HTML]{EFEFEF}0.475 & \cellcolor[HTML]{EFEFEF}0.769    & \cellcolor[HTML]{EFEFEF}0.385   & \cellcolor[HTML]{EFEFEF}0.464                                  & \cellcolor[HTML]{EFEFEF}0.470                                   & \cellcolor[HTML]{EFEFEF}0.403                                    & \cellcolor[HTML]{EFEFEF}0.419                                 & \cellcolor[HTML]{EFEFEF}0.770                                 & \cellcolor[HTML]{EFEFEF}0.600                          \\ \hline
\multicolumn{21}{c}{\textbf{Granite 4.100 3B}}                                                                                                                                                                                                                                                                                                                                                                                                                                                                                                                                                                                                                                                                                                                                                                                                                                                                                                                                                                                                                                                                                                                                                                                                                                \\ \hline
\textbf{Base}                                                                              & \cellcolor[HTML]{EFEFEF}100\%                                                                                         & \cellcolor[HTML]{EFEFEF}64                                                                       & \cellcolor[HTML]{EFEFEF}64                                                                      & \cellcolor[HTML]{EFEFEF}11.578                                                                          & \cellcolor[HTML]{EFEFEF}0.662 & \cellcolor[HTML]{EFEFEF}0.794 & \cellcolor[HTML]{EFEFEF}0.746 & \cellcolor[HTML]{EFEFEF}0.695 & \cellcolor[HTML]{EFEFEF}0.790 & \cellcolor[HTML]{EFEFEF}0.960 & \cellcolor[HTML]{EFEFEF}0.434 & \cellcolor[HTML]{EFEFEF}0.504 & \cellcolor[HTML]{EFEFEF}0.751    & \cellcolor[HTML]{EFEFEF}0.366   & \cellcolor[HTML]{EFEFEF}0.476                                  & \cellcolor[HTML]{EFEFEF}0.570                                   & \cellcolor[HTML]{EFEFEF}0.461                                    & \cellcolor[HTML]{EFEFEF}0.414                                 & \cellcolor[HTML]{EFEFEF}0.813                                 & \cellcolor[HTML]{EFEFEF}0.629                          \\ \hline
                                                                                           & 87.50\%                                                                                                               & 48                                                                                               & 64                                                                                              & \textbf{11.775}                                                                                                  & 0.671                         & 0.796                         & 0.736                         & 0.693                         & 0.780                         & 0.958                         & 0.426                         & 0.502                         & 0.736                            & 0.380                           & 0.466                                                          & 0.530                                                           & 0.490                                                            & 0.419                                                         & 0.806                                                         & \textbf{0.626}                                                  \\
                                                                                           & \cellcolor[HTML]{EFEFEF}75\%                                                                                          & \cellcolor[HTML]{EFEFEF}48                                                                       & \cellcolor[HTML]{EFEFEF}48                                                                      & \cellcolor[HTML]{EFEFEF}13.615                                                                          & \cellcolor[HTML]{EFEFEF}0.626 & \cellcolor[HTML]{EFEFEF}0.787 & \cellcolor[HTML]{EFEFEF}0.738 & \cellcolor[HTML]{EFEFEF}0.679 & \cellcolor[HTML]{EFEFEF}0.736 & \cellcolor[HTML]{EFEFEF}0.949 & \cellcolor[HTML]{EFEFEF}0.430 & \cellcolor[HTML]{EFEFEF}0.500 & \cellcolor[HTML]{EFEFEF}0.626    & \cellcolor[HTML]{EFEFEF}0.385   & \cellcolor[HTML]{EFEFEF}0.471                                  & \cellcolor[HTML]{EFEFEF}0.520                                   & \cellcolor[HTML]{EFEFEF}0.412                                    & \cellcolor[HTML]{EFEFEF}0.380                                 & \cellcolor[HTML]{EFEFEF}0.772                                 & \cellcolor[HTML]{EFEFEF}0.601                          \\
                                                                                           & 75\%                                                                                                                  & 32                                                                                               & 64                                                                                              & \textbf{11.918}                                                                                                  & 0.675                         & 0.799                         & 0.731                         & 0.693                         & 0.776                         & 0.957                         & 0.426                         & 0.500                         & 0.748                            & 0.376                           & 0.468                                                          & 0.540                                                           & 0.490                                                            & 0.422                                                         & 0.818                                                         & \textbf{0.628}                                                  \\
\multirow{-4}{*}{\textbf{PALU}}                                                            & \cellcolor[HTML]{EFEFEF}50\%                                                                                          & \cellcolor[HTML]{EFEFEF}32                                                                       & \cellcolor[HTML]{EFEFEF}32                                                                      & \cellcolor[HTML]{EFEFEF}17.483                                                                          & \cellcolor[HTML]{EFEFEF}0.539 & \cellcolor[HTML]{EFEFEF}0.760 & \cellcolor[HTML]{EFEFEF}0.646 & \cellcolor[HTML]{EFEFEF}0.628 & \cellcolor[HTML]{EFEFEF}0.652 & \cellcolor[HTML]{EFEFEF}0.914 & \cellcolor[HTML]{EFEFEF}0.412 & \cellcolor[HTML]{EFEFEF}0.445 & \cellcolor[HTML]{EFEFEF}0.708    & \cellcolor[HTML]{EFEFEF}0.362   & \cellcolor[HTML]{EFEFEF}0.456                                  & \cellcolor[HTML]{EFEFEF}0.440                                   & \cellcolor[HTML]{EFEFEF}0.333                                    & \cellcolor[HTML]{EFEFEF}0.332                                 & \cellcolor[HTML]{EFEFEF}0.642                                 & \cellcolor[HTML]{EFEFEF}\textbf{0.551}                          \\ \hline
                                                                                           & \cellcolor[HTML]{EFEFEF}87.50\%                                                                                       & \cellcolor[HTML]{EFEFEF}48                                                                       & \cellcolor[HTML]{EFEFEF}64                                                                      & \cellcolor[HTML]{EFEFEF}\textbf{11.775}                                                                       & \cellcolor[HTML]{EFEFEF}0.671 & \cellcolor[HTML]{EFEFEF}0.796 & \cellcolor[HTML]{EFEFEF}0.736 & \cellcolor[HTML]{EFEFEF}0.693 & \cellcolor[HTML]{EFEFEF}0.779 & \cellcolor[HTML]{EFEFEF}0.958 & \cellcolor[HTML]{EFEFEF}0.426 & \cellcolor[HTML]{EFEFEF}0.502 & \cellcolor[HTML]{EFEFEF}0.736 & \cellcolor[HTML]{EFEFEF}0.379 
                                                                  & \cellcolor[HTML]{EFEFEF}0.466                          &\cellcolor[HTML]{EFEFEF}0.53                               & \cellcolor[HTML]{EFEFEF}0.49                                  & \cellcolor[HTML]{EFEFEF}0.419                                   & \cellcolor[HTML]{EFEFEF}0.806                                                               & \cellcolor[HTML]{EFEFEF}\textbf{0.626}                          \\
                                                                                           & 75\%                                                                                                                  & 48                                                                                               & 48                                                                                              & 13.481                                                                                                  & 0.577                         & 0.785                         & 0.723                         & 0.657                        & 0.757                         & 0.952                         & 0.416                        & 0.513                         & 0.729                           & 0.376                           & 0.440                                                          & 0.56                                                           & 0.421                                                            & 0.360                                                         & 0.761                                                         & 0.601                                                 \\
                                                                                           & \cellcolor[HTML]{EFEFEF}75\%                                                                                          & \cellcolor[HTML]{EFEFEF}32                                                                       & \cellcolor[HTML]{EFEFEF}64                                                                      & \cellcolor[HTML]{EFEFEF}29.25                                                                          & \cellcolor[HTML]{EFEFEF}0.162 & \cellcolor[HTML]{EFEFEF}0.724 & \cellcolor[HTML]{EFEFEF}0.545 & \cellcolor[HTML]{EFEFEF}0.539 & \cellcolor[HTML]{EFEFEF}0.566 & \cellcolor[HTML]{EFEFEF}0.765 & \cellcolor[HTML]{EFEFEF}0.37 & \cellcolor[HTML]{EFEFEF}0.354 & \cellcolor[HTML]{EFEFEF}0.388    & \cellcolor[HTML]{EFEFEF}0.301   & \cellcolor[HTML]{EFEFEF}0.429                                 & \cellcolor[HTML]{EFEFEF}0.35                                  & \cellcolor[HTML]{EFEFEF}0.186                                    & \cellcolor[HTML]{EFEFEF}0.255                                 & \cellcolor[HTML]{EFEFEF}0.196                                & \cellcolor[HTML]{EFEFEF}0.408                          \\
\multirow{-4}{*}{\textbf{\begin{tabular}[c]{@{}c@{}}Eigen\\ Attention\end{tabular}}}       & 50\%                                                                                                                  & 32                                                                                               & 32                                                                                              & 23.361                                                                                                  & 0.313                         & 0.744                      & 0.587                         & 0.588                         & 0.641                         & 0.874                       & 0.388                         & 0.387                         & 0.646                            & 0.343                          & 0.411                                                          & 0.27                                                           & 0.157                                                           & 0.265                                                         & 0.365                                                        & 0.465                                                 \\ \hline
                                                                                           & \cellcolor[HTML]{EFEFEF}87.50\%                                                                                       & \cellcolor[HTML]{EFEFEF}48                                                                       & \cellcolor[HTML]{EFEFEF}64                                                                      & \cellcolor[HTML]{EFEFEF}26.024                                                                          & \cellcolor[HTML]{EFEFEF}0.221 & \cellcolor[HTML]{EFEFEF}0.719 & \cellcolor[HTML]{EFEFEF}0.586 & \cellcolor[HTML]{EFEFEF}0.517 & \cellcolor[HTML]{EFEFEF}0.643 & \cellcolor[HTML]{EFEFEF}0.813 & \cellcolor[HTML]{EFEFEF}0.356 & \cellcolor[HTML]{EFEFEF}0.357 & \cellcolor[HTML]{EFEFEF}0.583    & \cellcolor[HTML]{EFEFEF}0.294   & \cellcolor[HTML]{EFEFEF}0.395                                  & \cellcolor[HTML]{EFEFEF}0.230                                   & \cellcolor[HTML]{EFEFEF}0.235                                    & \cellcolor[HTML]{EFEFEF}0.257                                 & \cellcolor[HTML]{EFEFEF}0.215                                 & \cellcolor[HTML]{EFEFEF}0.428                          \\
\multirow{-2}{*}{\textbf{\begin{tabular}[c]{@{}c@{}}Thin Keys\\ Full Values\end{tabular}}} & 75\%                                                                                                                  & 32                                                                                               & 64                                                                                              & 40.573                                                                                                  & 0.152                         & 0.694                         & 0.544                         & 0.531                         & 0.560                         & 0.760                         & 0.340                         & 0.315                         & 0.603                            & 0.286                           & 0.432                                                          & 0.220                                                           & 0.206                                                            & 0.244                                                         & 0.194                                                         & 0.405                                                  \\ \hline
                                                                                           & \cellcolor[HTML]{EFEFEF}87.50\%                                                                                       & \cellcolor[HTML]{EFEFEF}48                                                                       & \cellcolor[HTML]{EFEFEF}64                                                                      & \cellcolor[HTML]{EFEFEF}11.928                                                                          & \cellcolor[HTML]{EFEFEF}0.642 & \cellcolor[HTML]{EFEFEF}0.791 & \cellcolor[HTML]{EFEFEF}0.732 & \cellcolor[HTML]{EFEFEF}0.678 & \cellcolor[HTML]{EFEFEF}0.784 & \cellcolor[HTML]{EFEFEF}0.942 & \cellcolor[HTML]{EFEFEF}0.430 & \cellcolor[HTML]{EFEFEF}0.478 & \cellcolor[HTML]{EFEFEF}0.777    & \cellcolor[HTML]{EFEFEF}0.382   & \cellcolor[HTML]{EFEFEF}0.469                                  & \cellcolor[HTML]{EFEFEF}0.530                                   & \cellcolor[HTML]{EFEFEF}0.490                                    & \cellcolor[HTML]{EFEFEF}0.396                                 & \cellcolor[HTML]{EFEFEF}0.782                                 & \cellcolor[HTML]{EFEFEF}0.620                          \\
                                                                                           & 75\%                                                                                                                  & 48                                                                                               & 48                                                                                              & 12.795                                                                                                  & 0.575                         & 0.782                         & 0.712                         & 0.662                         & 0.755                         & 0.948                         & 0.424                         & 0.453                         & 0.759                            & 0.373                           & 0.467                                                          & 0.340                                                           & 0.373                                                            & 0.357                                                         & 0.672                                                         & 0.577                                                  \\
                                                                                           & \cellcolor[HTML]{EFEFEF}75\%                                                                                          & \cellcolor[HTML]{EFEFEF}32                                                                       & \cellcolor[HTML]{EFEFEF}64                                                                      & \cellcolor[HTML]{EFEFEF}13.963                                                                          & \cellcolor[HTML]{EFEFEF}0.563 & \cellcolor[HTML]{EFEFEF}0.780 & \cellcolor[HTML]{EFEFEF}0.706 & \cellcolor[HTML]{EFEFEF}0.639 & \cellcolor[HTML]{EFEFEF}0.754 & \cellcolor[HTML]{EFEFEF}0.945 & \cellcolor[HTML]{EFEFEF}0.428 & \cellcolor[HTML]{EFEFEF}0.449 & \cellcolor[HTML]{EFEFEF}0.724    & \cellcolor[HTML]{EFEFEF}0.344   & \cellcolor[HTML]{EFEFEF}0.438                                  & \cellcolor[HTML]{EFEFEF}0.270                                   & \cellcolor[HTML]{EFEFEF}0.275                                    & \cellcolor[HTML]{EFEFEF}0.273                                 & \cellcolor[HTML]{EFEFEF}0.444                                 & \cellcolor[HTML]{EFEFEF}0.535                          \\
\multirow{-4}{*}{\textbf{RAP}}                                                             & 50\%                                                                                                                  & 32                                                                                               & 32                                                                                              & 17.766                                                                                                  & 0.359                         & 0.700                         & 0.641                         & 0.588                         & 0.684                         & 0.914                         & 0.350                         & 0.403                         & 0.605                            & 0.344                           & 0.405                                                          & 0.190                                                           & 0.206                                                            & 0.243                                                         & 0.262                                                         & 0.460                                                  \\ \hline
                                                                                           & \cellcolor[HTML]{EFEFEF}87.50\%                                                                                       & \cellcolor[HTML]{EFEFEF}48                                                                       & \cellcolor[HTML]{EFEFEF}64                                                                      & \cellcolor[HTML]{EFEFEF}12.092                                                                          & \cellcolor[HTML]{EFEFEF}0.633 & \cellcolor[HTML]{EFEFEF}0.789 & \cellcolor[HTML]{EFEFEF}0.733 & \cellcolor[HTML]{EFEFEF}0.672 & \cellcolor[HTML]{EFEFEF}0.783 & \cellcolor[HTML]{EFEFEF}0.961 & \cellcolor[HTML]{EFEFEF}0.436 & \cellcolor[HTML]{EFEFEF}0.491 & \cellcolor[HTML]{EFEFEF}0.752    & \cellcolor[HTML]{EFEFEF}0.348   & \cellcolor[HTML]{EFEFEF}0.456                                  & \cellcolor[HTML]{EFEFEF}0.510                                   & \cellcolor[HTML]{EFEFEF}0.491                                    & \cellcolor[HTML]{EFEFEF}0.438                                 & \cellcolor[HTML]{EFEFEF}0.819                                 & \cellcolor[HTML]{EFEFEF}0.621                          \\
                                                                                           & 75\%                                                                                                                  & 48                                                                                               & 48                                                                                              & 13.148                                                                                                  & 0.611                         & 0.787                         & 0.718                         & 0.670                         & 0.773                         & 0.957                         & 0.432                         & 0.482                         & 0.735                            & 0.355                           & 0.441                                                          & 0.500                                                           & 0.462                                                            & 0.406                                                         & 0.782                                                         & 0.607                                                  \\
                                                                                           & \cellcolor[HTML]{EFEFEF}75\%                                                                                          & \cellcolor[HTML]{EFEFEF}32                                                                       & \cellcolor[HTML]{EFEFEF}64                                                                      & \cellcolor[HTML]{EFEFEF}13.041                                                                          & \cellcolor[HTML]{EFEFEF}0.586 & \cellcolor[HTML]{EFEFEF}0.776 & \cellcolor[HTML]{EFEFEF}0.687 & \cellcolor[HTML]{EFEFEF}0.659 & \cellcolor[HTML]{EFEFEF}0.763 & \cellcolor[HTML]{EFEFEF}0.957 & \cellcolor[HTML]{EFEFEF}0.426 & \cellcolor[HTML]{EFEFEF}0.480 & \cellcolor[HTML]{EFEFEF}0.757    & \cellcolor[HTML]{EFEFEF}0.363   & \cellcolor[HTML]{EFEFEF}0.430                                  & \cellcolor[HTML]{EFEFEF}0.430                                   & \cellcolor[HTML]{EFEFEF}0.373                                    & \cellcolor[HTML]{EFEFEF}0.344                                 & \cellcolor[HTML]{EFEFEF}0.622                                 & \cellcolor[HTML]{EFEFEF}0.577                          \\
\multirow{-4}{*}{\textbf{RACE (Ours)}}                                                            & 50\%                                                                                                                  & 32                                                                                               & 32                                                                                              & \textbf{16.210}                                                                                                  & 0.516                         & 0.757                         & 0.647                         & 0.586                         & 0.686                         & 0.918                         & 0.418                         & 0.417                         & 0.683                            & 0.337                           & 0.462                                                          & 0.380                                                           & 0.295                                                            & 0.322                                                         & 0.624                                                         & 0.537                                                  \\ \hline
\multicolumn{21}{c}{\textbf{Llama 3.200 3B Instruct}}                                                                                                                                                                                                                                                                                                                                                                                                                                                                                                                                                                                                                                                                                                                                                                                                                                                                                                                                                                                                                                                                                                                                                                                                                         \\ \hline
\textbf{Base}                                                                              & 100\%                                                                                                                 & 128                                                                                              & 128                                                                                             & 13.397                                                                                                  & 0.659                         & 0.757                         & 0.706                         & 0.680                         & 0.681                         & 0.932                         & 0.360                         & 0.460                         & 0.784                            & 0.410                           & 0.496                                                          & 0.470                                                           & 0.353                                                            & 0.463                                                         & 0.769                                                         &                   0.599                            \\ \hline
                                                                                           & \cellcolor[HTML]{EFEFEF}87.50\%                                                                                       & \cellcolor[HTML]{EFEFEF}112                                                                      & \cellcolor[HTML]{EFEFEF}112                                                                     & \cellcolor[HTML]{EFEFEF}13.818                                                                          & \cellcolor[HTML]{EFEFEF}0.649 & \cellcolor[HTML]{EFEFEF}0.756 & \cellcolor[HTML]{EFEFEF}0.703 & \cellcolor[HTML]{EFEFEF}0.680 & \cellcolor[HTML]{EFEFEF}0.657 & \cellcolor[HTML]{EFEFEF}0.933 & \cellcolor[HTML]{EFEFEF}0.374 & \cellcolor[HTML]{EFEFEF}0.463 & \cellcolor[HTML]{EFEFEF}0.776    & \cellcolor[HTML]{EFEFEF}0.411   & \cellcolor[HTML]{EFEFEF}0.467                                  & \cellcolor[HTML]{EFEFEF}0.480                                   & \cellcolor[HTML]{EFEFEF}0.333                                    & \cellcolor[HTML]{EFEFEF}0.442                                 & \cellcolor[HTML]{EFEFEF}0.761                                 & \cellcolor[HTML]{EFEFEF}0.592                          \\
                                                                                           & 87.50\%                                                                                                               & 96                                                                                               & 128                                                                                             & 13.561                                                                                                  & 0.656                         & 0.754                         & 0.705                         & 0.680                         & 0.670                         & 0.935                         & 0.370                         & 0.445                         & 0.775                            & 0.409                           & 0.487                                                          & 0.470                                                           & 0.343                                                            & 0.450                                                         & 0.772                                                         & 0.595                                                  \\
                                                                                           & \cellcolor[HTML]{EFEFEF}75\%                                                                                          & \cellcolor[HTML]{EFEFEF}96                                                                       & \cellcolor[HTML]{EFEFEF}96                                                                      & \cellcolor[HTML]{EFEFEF}14.719                                                                          & \cellcolor[HTML]{EFEFEF}0.637 & \cellcolor[HTML]{EFEFEF}0.742 & \cellcolor[HTML]{EFEFEF}0.689 & \cellcolor[HTML]{EFEFEF}0.670 & \cellcolor[HTML]{EFEFEF}0.653 & \cellcolor[HTML]{EFEFEF}0.918 & \cellcolor[HTML]{EFEFEF}0.382 & \cellcolor[HTML]{EFEFEF}0.452 & \cellcolor[HTML]{EFEFEF}0.773    & \cellcolor[HTML]{EFEFEF}0.414   & \cellcolor[HTML]{EFEFEF}0.446                                  & \cellcolor[HTML]{EFEFEF}0.450                                   & \cellcolor[HTML]{EFEFEF}0.382                                    & \cellcolor[HTML]{EFEFEF}0.415                                 & \cellcolor[HTML]{EFEFEF}0.754                                 & \cellcolor[HTML]{EFEFEF}0.585                          \\
                                                                                           & 75\%                                                                                                                  & 64                                                                                               & 128                                                                                             & \textbf{14.395}                                                                                                  & 0.638                         & 0.757                         & 0.688                         & 0.677                         & 0.653                         & 0.933                         & 0.380                         & 0.423                         & 0.757                            & 0.417                           & 0.478                                                          & 0.440                                                           & 0.333                                                            & 0.424                                                         & 0.749                                                         & 0.583                                                  \\
\multirow{-5}{*}{\textbf{PALU}}                                                            & \cellcolor[HTML]{EFEFEF}50\%                                                                                          & \cellcolor[HTML]{EFEFEF}64                                                                       & \cellcolor[HTML]{EFEFEF}64                                                                      & \cellcolor[HTML]{EFEFEF}22.435                                                                          & \cellcolor[HTML]{EFEFEF}0.543 & \cellcolor[HTML]{EFEFEF}0.695 & \cellcolor[HTML]{EFEFEF}0.594 & \cellcolor[HTML]{EFEFEF}0.635 & \cellcolor[HTML]{EFEFEF}0.591 & \cellcolor[HTML]{EFEFEF}0.904 & \cellcolor[HTML]{EFEFEF}0.344 & \cellcolor[HTML]{EFEFEF}0.371 & \cellcolor[HTML]{EFEFEF}0.739    & \cellcolor[HTML]{EFEFEF}0.394   & \cellcolor[HTML]{EFEFEF}0.440                                  & \cellcolor[HTML]{EFEFEF}0.400                                   & \cellcolor[HTML]{EFEFEF}0.235                                    & \cellcolor[HTML]{EFEFEF}0.331                                 & \cellcolor[HTML]{EFEFEF}0.613                                 & \cellcolor[HTML]{EFEFEF}\textbf{0.522}                          \\ \hline
                                                                                           & 87.50\%                                                                                                               & 112                                                                                              & 112                                                                                             & 14.618                                                                                                  & 0.633                         & 0.745                         & 0.701                         & 0.676                         & 0.702                         & 0.924                         & 0.362                         & 0.464                         & 0.769                            & 0.411                           & 0.470                                                          & 0.430                                                           & 0.314                                                            & 0.433                                                         & 0.767                                                         & 0.587                                                  \\
                                                                                           & \cellcolor[HTML]{EFEFEF}87.50\%                                                                                       & \cellcolor[HTML]{EFEFEF}96                                                                       & \cellcolor[HTML]{EFEFEF}128                                                                     & \cellcolor[HTML]{EFEFEF}\textbf{13.557}                                                                          & \cellcolor[HTML]{EFEFEF}0.655 & \cellcolor[HTML]{EFEFEF}0.757 & \cellcolor[HTML]{EFEFEF}0.706 & \cellcolor[HTML]{EFEFEF}0.684 & \cellcolor[HTML]{EFEFEF}0.684 & \cellcolor[HTML]{EFEFEF}0.932 & \cellcolor[HTML]{EFEFEF}0.360 & \cellcolor[HTML]{EFEFEF}0.469 & \cellcolor[HTML]{EFEFEF}0.792    & \cellcolor[HTML]{EFEFEF}0.413   & \cellcolor[HTML]{EFEFEF}0.502                                  & \cellcolor[HTML]{EFEFEF}0.490                                   & \cellcolor[HTML]{EFEFEF}0.333                                    & \cellcolor[HTML]{EFEFEF}0.444                                 & \cellcolor[HTML]{EFEFEF}0.760                                 & \cellcolor[HTML]{EFEFEF}\textbf{0.599}                              \\
                                                                                           & 75\%                                                                                                                  & 96                                                                                               & 96                                                                                              & 16.347                                                                                                  & 0.590                         & 0.743                         & 0.672                         & 0.667                         & 0.701                         & 0.927                         & 0.366                         & 0.446                         & 0.764                            & 0.409                           & 0.447                                                          & 0.440                                                           & 0.324                                                            & 0.409                                                         & 0.736                                                         & 0.576                                                  \\
                                                                                           & \cellcolor[HTML]{EFEFEF}75\%                                                                                          & \cellcolor[HTML]{EFEFEF}64                                                                       & \cellcolor[HTML]{EFEFEF}128                                                                     & \cellcolor[HTML]{EFEFEF}14.552                                                                          & \cellcolor[HTML]{EFEFEF}0.660 & \cellcolor[HTML]{EFEFEF}0.758 & \cellcolor[HTML]{EFEFEF}0.691 & \cellcolor[HTML]{EFEFEF}0.669 & \cellcolor[HTML]{EFEFEF}0.689 & \cellcolor[HTML]{EFEFEF}0.939 & \cellcolor[HTML]{EFEFEF}0.364 & \cellcolor[HTML]{EFEFEF}0.457 & \cellcolor[HTML]{EFEFEF}0.757    & \cellcolor[HTML]{EFEFEF}0.420   & \cellcolor[HTML]{EFEFEF}0.484                                  & \cellcolor[HTML]{EFEFEF}0.400                                   & \cellcolor[HTML]{EFEFEF}0.382                                    & \cellcolor[HTML]{EFEFEF}0.413                                 & \cellcolor[HTML]{EFEFEF}0.756                                 & \cellcolor[HTML]{EFEFEF}\textbf{0.589}                          \\
\multirow{-5}{*}{\textbf{\begin{tabular}[c]{@{}c@{}}Eigen\\ Attention\end{tabular}}}       & 50\%                                                                                                                  & 64                                                                                               & 64                                                                                              & 30.778                                                                                                  & 0.296                         & 0.702                         & 0.547                         & 0.594                         & 0.594                         & 0.888                         & 0.332                         & 0.350                         & 0.645                            & 0.366                           & 0.404                                                          & 0.340                                                           & 0.294                                                            & 0.298                                                         & 0.494                                                         & 0.476                                                  \\ \hline
                                                                                           & \cellcolor[HTML]{EFEFEF}87.50\%                                                                                       & \cellcolor[HTML]{EFEFEF}96                                                                       & \cellcolor[HTML]{EFEFEF}128                                                                     & \cellcolor[HTML]{EFEFEF}56.350                                                                          & \cellcolor[HTML]{EFEFEF}0.084 & \cellcolor[HTML]{EFEFEF}0.587 & \cellcolor[HTML]{EFEFEF}0.399 & \cellcolor[HTML]{EFEFEF}0.514 & \cellcolor[HTML]{EFEFEF}0.381 & \cellcolor[HTML]{EFEFEF}0.586 & \cellcolor[HTML]{EFEFEF}0.294 & \cellcolor[HTML]{EFEFEF}0.266 & \cellcolor[HTML]{EFEFEF}0.616    & \cellcolor[HTML]{EFEFEF}0.271   & \cellcolor[HTML]{EFEFEF}0.447                                  & \cellcolor[HTML]{EFEFEF}0.220                                   & \cellcolor[HTML]{EFEFEF}0.255                                    & \cellcolor[HTML]{EFEFEF}0.257                                 & \cellcolor[HTML]{EFEFEF}0.202                                 & \cellcolor[HTML]{EFEFEF}0.359                          \\
\multirow{-2}{*}{\textbf{\begin{tabular}[c]{@{}c@{}}Thin Keys\\ Full Values\end{tabular}}} & 75\%                                                                                                                  & 64                                                                                               & 128                                                                                             & 84.815                                                                                                  & 0.070                         & 0.593                         & 0.383                         & 0.497                         & 0.372                         & 0.563                         & 0.284                         & 0.258                         & 0.601                            & 0.269                           & 0.444                                                          & 0.250                                                           & 0.275                                                            & 0.252                                                         & 0.217                                                         & 0.355                                                  \\ \hline
                                                                                           & \cellcolor[HTML]{EFEFEF}87.50\%                                                                                       & \cellcolor[HTML]{EFEFEF}112                                                                      & \cellcolor[HTML]{EFEFEF}112                                                                     & \cellcolor[HTML]{EFEFEF}15.339                                                                          & \cellcolor[HTML]{EFEFEF}0.605 & \cellcolor[HTML]{EFEFEF}0.740 & \cellcolor[HTML]{EFEFEF}0.681 & \cellcolor[HTML]{EFEFEF}0.656 & \cellcolor[HTML]{EFEFEF}0.649 & \cellcolor[HTML]{EFEFEF}0.933 & \cellcolor[HTML]{EFEFEF}0.374 & \cellcolor[HTML]{EFEFEF}0.428 & \cellcolor[HTML]{EFEFEF}0.764    & \cellcolor[HTML]{EFEFEF}0.429   & \cellcolor[HTML]{EFEFEF}0.486                                  & \cellcolor[HTML]{EFEFEF}0.430                                   & \cellcolor[HTML]{EFEFEF}0.304                                    & \cellcolor[HTML]{EFEFEF}0.419                                 & \cellcolor[HTML]{EFEFEF}0.712                                 & \cellcolor[HTML]{EFEFEF}0.574                          \\
                                                                                           & 87.50\%                                                                                                               & 96                                                                                               & 128                                                                                             & 14.685                                                                                                  & 0.611                         & 0.750                         & 0.683                         & 0.642                         & 0.672                         & 0.935                         & 0.372                         & 0.438                         & 0.762                            & 0.411                           & 0.480                                                          & 0.380                                                           & 0.353                                                            & 0.413                                                         & 0.675                                                         & 0.572                                                  \\
                                                                                           & \cellcolor[HTML]{EFEFEF}75\%                                                                                          & \cellcolor[HTML]{EFEFEF}96                                                                       & \cellcolor[HTML]{EFEFEF}96                                                                      & \cellcolor[HTML]{EFEFEF}15.226                                                                          & \cellcolor[HTML]{EFEFEF}0.570 & \cellcolor[HTML]{EFEFEF}0.742 & \cellcolor[HTML]{EFEFEF}0.665 & \cellcolor[HTML]{EFEFEF}0.607 & \cellcolor[HTML]{EFEFEF}0.649 & \cellcolor[HTML]{EFEFEF}0.934 & \cellcolor[HTML]{EFEFEF}0.370 & \cellcolor[HTML]{EFEFEF}0.425 & \cellcolor[HTML]{EFEFEF}0.757    & \cellcolor[HTML]{EFEFEF}0.415   & \cellcolor[HTML]{EFEFEF}0.467                                  & \cellcolor[HTML]{EFEFEF}0.360                                   & \cellcolor[HTML]{EFEFEF}0.255                                    & \cellcolor[HTML]{EFEFEF}0.372                                 & \cellcolor[HTML]{EFEFEF}0.639                                 & \cellcolor[HTML]{EFEFEF}0.548                          \\
                                                                                           & 75\%                                                                                                                  & 64                                                                                               & 128                                                                                             & 15.218                                                                                                  & 0.572                         & 0.731                         & 0.654                         & 0.620                         & 0.633                         & 0.927                         & 0.362                         & 0.415                         & 0.728                            & 0.413                           & 0.462                                                          & 0.350                                                           & 0.343                                                            & 0.373                                                         & 0.558                                                         & 0.543                                                  \\
\multirow{-5}{*}{\textbf{RAP}}                                                             & \cellcolor[HTML]{EFEFEF}50\%                                                                                          & \cellcolor[HTML]{EFEFEF}64                                                                       & \cellcolor[HTML]{EFEFEF}64                                                                      & \cellcolor[HTML]{EFEFEF}18.540                                                                          & \cellcolor[HTML]{EFEFEF}0.309 & \cellcolor[HTML]{EFEFEF}0.716 & \cellcolor[HTML]{EFEFEF}0.583 & \cellcolor[HTML]{EFEFEF}0.574 & \cellcolor[HTML]{EFEFEF}0.598 & \cellcolor[HTML]{EFEFEF}0.898 & \cellcolor[HTML]{EFEFEF}0.340 & \cellcolor[HTML]{EFEFEF}0.361 & \cellcolor[HTML]{EFEFEF}0.563    & \cellcolor[HTML]{EFEFEF}0.314   & \cellcolor[HTML]{EFEFEF}0.435                                  & \cellcolor[HTML]{EFEFEF}0.270                                   & \cellcolor[HTML]{EFEFEF}0.275                                    & \cellcolor[HTML]{EFEFEF}0.258                                 & \cellcolor[HTML]{EFEFEF}0.407                                 & \cellcolor[HTML]{EFEFEF}0.460                          \\ \hline
                                                                                           & 87.50\%                                                                                                               & 112                                                                                              & 112                                                                                             & 13.884                                                                                                  & 0.643                         & 0.756                         & 0.696                         & 0.672                         & 0.695                         & 0.925                         & 0.362                         & 0.459                         & 0.768                            & 0.418                           & 0.473                                                          & 0.470                                                           & 0.324                                                            & 0.451                                                         & 0.758                                                         & 0.591                                                  \\
                                                                                           & \cellcolor[HTML]{EFEFEF}87.50\%                                                                                       & \cellcolor[HTML]{EFEFEF}96                                                                       & \cellcolor[HTML]{EFEFEF}128                                                                     & \cellcolor[HTML]{EFEFEF}13.755                                                                          & \cellcolor[HTML]{EFEFEF}0.639 & \cellcolor[HTML]{EFEFEF}0.760 & \cellcolor[HTML]{EFEFEF}0.699 & \cellcolor[HTML]{EFEFEF}0.681 & \cellcolor[HTML]{EFEFEF}0.691 & \cellcolor[HTML]{EFEFEF}0.927 & \cellcolor[HTML]{EFEFEF}0.362 & \cellcolor[HTML]{EFEFEF}0.451 & \cellcolor[HTML]{EFEFEF}0.772    & \cellcolor[HTML]{EFEFEF}0.429   & \cellcolor[HTML]{EFEFEF}0.489                                  & \cellcolor[HTML]{EFEFEF}0.480                                   & \cellcolor[HTML]{EFEFEF}0.343                                    & \cellcolor[HTML]{EFEFEF}0.435                                 & \cellcolor[HTML]{EFEFEF}0.754                                 & \cellcolor[HTML]{EFEFEF}0.594                          \\
                                                                                           & 75\%                                                                                                                  & 96                                                                                               & 96                                                                                              & 14.520                                                                                                  & 0.616                         & 0.753                         & 0.682                         & 0.667                         & 0.700                         & 0.925                         & 0.362                         & 0.447                         & 0.760                            & 0.420                           & 0.457                                                          & 0.480                                                           & 0.333                                                            & 0.421                                                         & 0.738                                                         & 0.584                                                  \\
                                                                                           & \cellcolor[HTML]{EFEFEF}75\%                                                                                          & \cellcolor[HTML]{EFEFEF}64                                                                       & \cellcolor[HTML]{EFEFEF}128                                                                     & \cellcolor[HTML]{EFEFEF}14.780                                                                          & \cellcolor[HTML]{EFEFEF}0.615 & \cellcolor[HTML]{EFEFEF}0.755 & \cellcolor[HTML]{EFEFEF}0.675 & \cellcolor[HTML]{EFEFEF}0.648 & \cellcolor[HTML]{EFEFEF}0.690 & \cellcolor[HTML]{EFEFEF}0.916 & \cellcolor[HTML]{EFEFEF}0.350 & \cellcolor[HTML]{EFEFEF}0.435 & \cellcolor[HTML]{EFEFEF}0.750    & \cellcolor[HTML]{EFEFEF}0.403   & \cellcolor[HTML]{EFEFEF}0.489                                  & \cellcolor[HTML]{EFEFEF}0.450                                   & \cellcolor[HTML]{EFEFEF}0.324                                    & \cellcolor[HTML]{EFEFEF}0.387                                 & \cellcolor[HTML]{EFEFEF}0.684                                 & \cellcolor[HTML]{EFEFEF}0.571                          \\
\multirow{-5}{*}{\textbf{RACE (Ours)}}                                                            & 50\%                                                                                                                  & 64                                                                                               & 64                                                                                              & \textbf{16.016}                                                                                                  & 0.499                         & 0.740                         & 0.625                         & 0.586                         & 0.622                         & 0.885                         & 0.360                         & 0.381                         & 0.689                            & 0.373                           & 0.454                                                          & 0.340                                                           & 0.333                                                            & 0.371                                                         & 0.556                                                         & \textbf{0.522}                                           \\ \hline   
\end{tabular}
}
\end{table}

\begin{table}[]
\caption{Per-task results of Qwen2.5 7B and Llama3.2 3B Instruct on LongBench.
We report zero shot longbench score across the
evaluated model families and KV-cache retention ratios. The final column reports
the average across tasks. For each KV Cache compression level per model, we highlight the best performing average result. }
\label{tab:lb_full}
\resizebox{\columnwidth}{!}{%
\begin{tabular}{llrrrrrrrrrrrrrrrrr}
\hline
\rowcolor[HTML]{EFEFEF} 
                                                        & \textbf{\begin{tabular}[c]{@{}l@{}}KV Cache\\ Retained\end{tabular}} & \multicolumn{1}{l}{\cellcolor[HTML]{EFEFEF}\textbf{\begin{tabular}[c]{@{}l@{}}Narrative\\  QA\end{tabular}}} & \multicolumn{1}{l}{\cellcolor[HTML]{EFEFEF}\textbf{Qasper}} & \multicolumn{1}{l}{\cellcolor[HTML]{EFEFEF}\textbf{\begin{tabular}[c]{@{}l@{}}Multifield\\  QA\end{tabular}}} & \multicolumn{1}{l}{\cellcolor[HTML]{EFEFEF}\textbf{\begin{tabular}[c]{@{}l@{}}Hotpot\\  QA\end{tabular}}} & \multicolumn{1}{l}{\cellcolor[HTML]{EFEFEF}\textbf{Wikimqa}} & \multicolumn{1}{l}{\cellcolor[HTML]{EFEFEF}\textbf{Musique}} & \multicolumn{1}{l}{\cellcolor[HTML]{EFEFEF}\textbf{gov\_report}} & \multicolumn{1}{l}{\cellcolor[HTML]{EFEFEF}\textbf{qmsum}} & \multicolumn{1}{l}{\cellcolor[HTML]{EFEFEF}\textbf{multi\_news}} & \multicolumn{1}{l}{\cellcolor[HTML]{EFEFEF}\textbf{trec}} & \multicolumn{1}{l}{\cellcolor[HTML]{EFEFEF}\textbf{triviaqa}} & \multicolumn{1}{l}{\cellcolor[HTML]{EFEFEF}\textbf{samsum}} & \multicolumn{1}{l}{\cellcolor[HTML]{EFEFEF}\textbf{\begin{tabular}[c]{@{}l@{}}passage\\ \_count\end{tabular}}} & \multicolumn{1}{l}{\cellcolor[HTML]{EFEFEF}\textbf{\begin{tabular}[c]{@{}l@{}}passage\_\\ retrieval\_en\end{tabular}}} & \multicolumn{1}{l}{\cellcolor[HTML]{EFEFEF}\textbf{lcc}} & \multicolumn{1}{l}{\cellcolor[HTML]{EFEFEF}\textbf{repobench}} & \multicolumn{1}{l}{\cellcolor[HTML]{EFEFEF}\textbf{Average}} \\ \hline
\multicolumn{19}{c}{\textbf{Llama3.2 3B Instruct.}}                                                                                                                                                                                                                                                                                                                                                                                                                                                                                                                                                                                                                                                                                                                                                                                                                                                                                                                                                                                                                                                                                                                                                                                                                                                                                                                                                                                                                                                     \\ \hline
\rowcolor[HTML]{EFEFEF} 
\textbf{Uncompressed}                                   & 100\%                                                                & 21.25                                                                                                        & 39.89                                                       & 50.06                                                                                                         & 47.31                                                                                                     & 36.75                                                        & 19.15                                                        & 33.91                                                            & 23.5                                                       & 26.64                                                            & 71                                                        & 88.41                                                         & 42.62                                                       & 8.5                                                                                                            & 75                                                                                                                     & 52.01                                                    & 52.5                                                           & 43.03                                                        \\
                                                        & 87.5\%                                                               & 24.81                                                                                                        & 33.42                                                       & 49.61                                                                                                         & 47.52                                                                                                     & 36.52                                                        & 17.9                                                         & 33.95                                                            & 23                                                         & 26.98                                                            & 73                                                        & 90.47                                                         & 42.91                                                       & 4                                                                                                              & 63.5                                                                                                                   & 51.26                                                    & 53.9                                                           & \textbf{42.05}                                               \\
                                                        & \cellcolor[HTML]{EFEFEF}75\%                                         & \cellcolor[HTML]{EFEFEF}22.2                                                                                 & \cellcolor[HTML]{EFEFEF}30.79                               & \cellcolor[HTML]{EFEFEF}45.3                                                                                  & \cellcolor[HTML]{EFEFEF}40.44                                                                             & \cellcolor[HTML]{EFEFEF}35.29                                & \cellcolor[HTML]{EFEFEF}15.96                                & \cellcolor[HTML]{EFEFEF}32.88                                    & \cellcolor[HTML]{EFEFEF}22.73                              & \cellcolor[HTML]{EFEFEF}25.77                                    & \cellcolor[HTML]{EFEFEF}71                                & \cellcolor[HTML]{EFEFEF}87.78                                 & \cellcolor[HTML]{EFEFEF}41.16                               & \cellcolor[HTML]{EFEFEF}3.5                                                                                    & \cellcolor[HTML]{EFEFEF}58.5                                                                                           & \cellcolor[HTML]{EFEFEF}44.34                            & \cellcolor[HTML]{EFEFEF}44.45                                  & \cellcolor[HTML]{EFEFEF}\textbf{38.88}                       \\
\multirow{-3}{*}{\textbf{RACE}}                         & 50\%                                                                 & 11.45                                                                                                        & 15.08                                                       & 25.46                                                                                                         & 22.86                                                                                                     & 20.01                                                        & 15.52                                                        & 21.44                                                            & 20.5                                                       & 20.97                                                            & 59.5                                                      & 78.45                                                         & 35.32                                                       & 1.18                                                                                                           & 5.75                                                                                                                   & 34.83                                                    & 35.91                                                          & 26.51                                                        \\
\rowcolor[HTML]{EFEFEF} 
\cellcolor[HTML]{EFEFEF}                                & 87.5\%                                                               & 23.49                                                                                                        & 32.55                                                       & 46.14                                                                                                         & 41.67                                                                                                     & 36.7                                                         & 14.63                                                        & 30.91                                                            & 22.6                                                       & 26.31                                                            & 68.5                                                      & 87.58                                                         & 40.96                                                       & 4.5                                                                                                            & 45                                                                                                                     & 49.3                                                     & 52.24                                                          & 38.94                                                        \\
\cellcolor[HTML]{EFEFEF}                                & 75\%                                                                 & 19.49                                                                                                        & 27.26                                                       & 41.91                                                                                                         & 30.35                                                                                                     & 31.14                                                        & 13.96                                                        & 25.55                                                            & 21.63                                                      & 25.39                                                            & 64                                                        & 84.46                                                         & 39.81                                                       & 4                                                                                                              & 23                                                                                                                     & 47.55                                                    & 49.75                                                          & 34.33                                                        \\
\rowcolor[HTML]{EFEFEF} 
\multirow{-3}{*}{\cellcolor[HTML]{EFEFEF}\textbf{RAP}}  & 50\%                                                                 & 7.79                                                                                                         & 6.42                                                        & 19                                                                                                            & 6.16                                                                                                      & 6.81                                                         & 8.77                                                         & 9.2                                                              & 7.94                                                       & 5.48                                                             & 28                                                        & 40.2                                                          & 11.99                                                       & 4                                                                                                              & 5                                                                                                                      & 34.8                                                     & 32.67                                                          & 14.64                                                        \\
                                                        & 87.5\%                                                               & 21.9                                                                                                         & 31.34                                                       & 49.22                                                                                                         & 45.26                                                                                                     & 36.56                                                        & 17.96                                                        & 31.82                                                            & 23.32                                                      & 25.58                                                            & 67                                                        & 89.59                                                         & 30.3                                                        & 2.79                                                                                                           & 74.5                                                                                                                   & 46.9                                                     & 50.69                                                          & 40.30                                                        \\
                                                        & \cellcolor[HTML]{EFEFEF}75\%                                         & \cellcolor[HTML]{EFEFEF}21.91                                                                                & \cellcolor[HTML]{EFEFEF}30.94                               & \cellcolor[HTML]{EFEFEF}46.85                                                                                 & \cellcolor[HTML]{EFEFEF}40.73                                                                             & \cellcolor[HTML]{EFEFEF}32.06                                & \cellcolor[HTML]{EFEFEF}17.78                                & \cellcolor[HTML]{EFEFEF}30.95                                    & \cellcolor[HTML]{EFEFEF}22.63                              & \cellcolor[HTML]{EFEFEF}24.6                                     & \cellcolor[HTML]{EFEFEF}66                                & \cellcolor[HTML]{EFEFEF}88.7                                  & \cellcolor[HTML]{EFEFEF}28.8                                & \cellcolor[HTML]{EFEFEF}2.21                                                                                   & \cellcolor[HTML]{EFEFEF}56.5                                                                                           & \cellcolor[HTML]{EFEFEF}43.62                            & \cellcolor[HTML]{EFEFEF}38.81                                  & \cellcolor[HTML]{EFEFEF}37.07                                \\
\multirow{-3}{*}{\textbf{PALU}}                         & 50\%                                                                 & 20.47                                                                                                        & 30.26                                                       & 37.77                                                                                                         & 33.88                                                                                                     & 29.15                                                        & 14.18                                                        & 24.79                                                            & 22.24                                                      & 21.8                                                             & 48                                                        & 79.63                                                         & 25.5                                                        & 1                                                                                                              & 9                                                                                                                      & 25.74                                                    & 34.71                                                          & \textbf{28.63}                                               \\ \hline
\rowcolor[HTML]{EFEFEF} 
\multicolumn{19}{c}{\cellcolor[HTML]{EFEFEF}\textbf{Qwen2.5 7B}}                                                                                                                                                                                                                                                                                                                                                                                                                                                                                                                                                                                                                                                                                                                                                                                                                                                                                                                                                                                                                                                                                                                                                                                                                                                                                                                                                                                                                                        \\ \hline
\textbf{Uncompressed}                                   & 100\%                                                                & 17.56                                                                                                        & 14.86                                                       & 27.7                                                                                                          & 29.5                                                                                                      & 15.51                                                        & 11.51                                                        & 30.96                                                            & 23.07                                                      & 26.75                                                            & 72.5                                                      & 89.31                                                         & 45.83                                                       & 3.5                                                                                                            & 24.5                                                                                                                   & 63.07                                                    & 51.58                                                          & 34.23                                                        \\
\rowcolor[HTML]{EFEFEF} 
\cellcolor[HTML]{EFEFEF}                                & 87.5\%                                                               & 17.98                                                                                                        & 20.87                                                       & 41.23                                                                                                         & 41.75                                                                                                     & 31.3                                                         & 18.85                                                        & 28                                                               & 22.99                                                      & 26.7                                                             & 72.5                                                      & 90.25                                                         & 44.98                                                       & 2.5                                                                                                            & 19                                                                                                                     & 53.35                                                    & 44.22                                                          & \textbf{36.03}                                               \\
\cellcolor[HTML]{EFEFEF}                                & 75\%                                                                 & 17.11                                                                                                        & 18.33                                                       & 40.39                                                                                                         & 41.08                                                                                                     & 30.36                                                        & 18.76                                                        & 25.78                                                            & 22.27                                                      & 26.25                                                            & 70.5                                                      & 87.13                                                         & 44.07                                                       & 3                                                                                                              & 8.5                                                                                                                    & 40.51                                                    & 46.8                                                           & \textbf{33.80}                                               \\
\rowcolor[HTML]{EFEFEF} 
\multirow{-3}{*}{\cellcolor[HTML]{EFEFEF}\textbf{RACE}} & 50\%                                                                 & 4.9                                                                                                          & 14.27                                                       & 22.66                                                                                                         & 19.21                                                                                                     & 7.23                                                         & 3.84                                                         & 18.37                                                            & 21.1                                                       & 20.89                                                            & 69                                                        & 84.26                                                         & 39.01                                                       & 2.8                                                                                                            & 8.5                                                                                                                    & 37.1                                                     & 38.07                                                          & \textbf{25.70}                                               \\
                                                        & 87.5\%                                                               & 13.1                                                                                                         & 21.46                                                       & 27.31                                                                                                         & 10.36                                                                                                     & 18.43                                                        & 14.52                                                        & 26.22                                                            & 21.15                                                      & 25.7                                                             & 72.5                                                      & 89.56                                                         & 44.53                                                       & 4                                                                                                              & 8                                                                                                                      & 44.44                                                    & 50.62                                                          & 30.74                                                        \\
                                                        & \cellcolor[HTML]{EFEFEF}75\%                                         & \cellcolor[HTML]{EFEFEF}5.78                                                                                 & \cellcolor[HTML]{EFEFEF}21.38                               & \cellcolor[HTML]{EFEFEF}26.8                                                                                  & \cellcolor[HTML]{EFEFEF}13.88                                                                             & \cellcolor[HTML]{EFEFEF}11.9                                 & \cellcolor[HTML]{EFEFEF}5.26                                 & \cellcolor[HTML]{EFEFEF}23.31                                    & \cellcolor[HTML]{EFEFEF}23.32                              & \cellcolor[HTML]{EFEFEF}23.75                                    & \cellcolor[HTML]{EFEFEF}68.5                              & \cellcolor[HTML]{EFEFEF}86.99                                 & \cellcolor[HTML]{EFEFEF}44.21                               & \cellcolor[HTML]{EFEFEF}3                                                                                      & \cellcolor[HTML]{EFEFEF}4.5                                                                                            & \cellcolor[HTML]{EFEFEF}44.79                            & \cellcolor[HTML]{EFEFEF}42.49                                  & \cellcolor[HTML]{EFEFEF}28.12                                \\
\multirow{-3}{*}{\textbf{RAP}}                          & 50\%                                                                 & 5.16                                                                                                         & 9.4                                                         & 15.36                                                                                                         & 7.57                                                                                                      & 8.48                                                         & 4.09                                                         & 15.67                                                            & 21.49                                                      & 18.78                                                            & 52                                                        & 73.1                                                          & 33.92                                                       & 1.17                                                                                                           & 2.8                                                                                                                    & 27.8                                                     & 42.13                                                          & 21.18                                                        \\
\rowcolor[HTML]{EFEFEF} 
\cellcolor[HTML]{EFEFEF}                                & 87.5\%                                                               & 7.94                                                                                                         & 12.25                                                       & 29.8                                                                                                          & 36.21                                                                                                     & 26.83                                                        & 18.46                                                        & 29.05                                                            & 22.63                                                      & 26.35                                                            & 72.5                                                      & 89.49                                                         & 34.25                                                       & 2                                                                                                              & 19                                                                                                                     & 49.72                                                    & 40.75                                                          & 32.33                                                        \\
\cellcolor[HTML]{EFEFEF}                                & 75\%                                                                 & 8.15                                                                                                         & 16.48                                                       & 28.76                                                                                                         & 36.94                                                                                                     & 28.71                                                        & 18.8                                                         & 27.25                                                            & 22.58                                                      & 25.16                                                            & 70.5                                                      & 89.93                                                         & 34.45                                                       & 4                                                                                                              & 10                                                                                                                     & 41.39                                                    & 33.24                                                          & 31.02                                                        \\
\rowcolor[HTML]{EFEFEF} 
\multirow{-3}{*}{\cellcolor[HTML]{EFEFEF}\textbf{PALU}} & 50\%                                                                 & 4.15                                                                                                         & 8.95                                                        & 23.03                                                                                                         & 10.8                                                                                                      & 9.37                                                         & 5.83                                                         & 21.21                                                            & 21.25                                                      & 22.3                                                             & 69.5                                                      & 83.81                                                         & 33.82                                                       & 1.5                                                                                                            & 8.5                                                                                                                    & 31.17                                                    & 31.17                                                          & 24.15                                                        \\ \hline
\end{tabular}}
\end{table}

\paragraph{Language Modeling and Zero-Shot Performance}
The main quality experiments evaluate Qwen2.5-7B, Llama-3.2-3B-Instruct, and
Granite-4.1-3B. We compare against PALU-GLRD, EigenAttention, Thin Keys, and
RAP at matched or reported KV-cache budgets. The full results are presented in Table \ref{tab:ZS_full} while only RAP, a reconstruction-free method like RACE, and PALU-GLRD, a reconstruction-based method, are included in Figure \ref{fig:KV_vary}. Across all three model families, RACE consistently outperforms the
reconstruction-free baselines RAP and Thin Keys, while remaining competitive
with the reconstruction-based EigenAttention and PALU. The difference is most
pronounced under aggressive compression: At 50\% KV-cache retention, RACE achieves the lowest WikiText perplexity across all three model families while remaining competitive with PALU in average zero-shot accuracy. among the evaluated compression methods on each of Qwen2.5, Granite,
and Llama. At more moderate compression, RACE remains close to the
reconstruction-based methods; at 75\% retention it preserves 98.5\%, 96.5\%,
and 97.5\% of the corresponding uncompressed average zero-shot accuracy on
Qwen, Granite, and Llama, respectively. The per-task results also show that the effect of compression is
not uniform across benchmarks: some tasks remain close to the uncompressed
baseline even at lower retention, whereas others are more sensitive to reduced
attention dimensionality.  Table \ref{tab:lb_full} presents the per-task results of Qwen2.5-7B and Llama-3.2-3B-Instruct on 16 English tasks in LongBench. RACE outperforms RAP at all evaluated retention levels and exceeds PALU at 87.5\% and 75\% retention on both model families. At 50\% retention, RACE remains competitive with PALU, outperforming it on Qwen while trailing it modestly on Llama. Together, these results show that the
quality-compression behavior is consistent across model families rather than
being specific to a single architecture, while retaining the reconstruction-free
execution of RACE.

\paragraph{Additional Systems Results}
The systems results in Figure \ref{fig:mems} show that RACE reduces both the fixed attention-parameter footprint and the context-dependent KV-cache footprint. As expected, the reduction in attention parameter memory aligns with the compression. However, since the folded projection weights are smaller in comparison to the result of the model parameters and memory consumption elements, the reduction in memory is less prominant in the total parameter memory. Additionally parameter memory is independent of context length. KV-cache memory scales with sequence length and therefore yields increasingly large absolute savings at longer contexts. At 256K tokens, the gap between the base model and the compressed variants becomes substantial, illustrating that the main systems benefit of RACE grows with context length while preserving native reduced-dimensional execution.

\paragraph{Compatibility with Token compression and Quantization}
Table 5 reports the full per-task results when RACE is combined with token-compression and KV-quantization methods. Across H2O, SnapKV, and KIVI, applying RACE on top of an already compressed KV cache provides additional multiplicative memory reduction, but also introduces further accuracy and perplexity degradation. This effect is generally more pronounced for token-compression methods than for quantization.

For H2O and SnapKV, reducing the head dimension with RACE consistently lowers the total KV-cache footprint from 50\% to 37.5\% and 25\%, or from 25\% to 18.75\% and 12.5\%, depending on the token-retention ratio. The corresponding degradation is gradual across most zero-shot tasks, while WikiText perplexity increases more noticeably at aggressive combined compression. For example, with H2O at 50\% token retention, adding RACE at 75\% and 50\% head retention reduces total KV-cache usage to 37.5\% and 25\%, respectively, while average zero-shot accuracy decreases from 0.706 to 0.674 and 0.604. SnapKV follows a similar trend, with average accuracy decreasing from 0.710 to 0.676 and 0.612 under the same settings.

RACE is comparatively more complementary with KIVI. With 4-bit quantization, adding RACE at 75\% and 50\% head retention reduces the total KV cache from 50\% to 37.5\% and 25\%, while average accuracy decreases from 0.702 to 0.678 and 0.608. Similar behavior is observed with 2-bit quantization. Overall, the results suggest that RACE composes cleanly with both token and precision compression, but the cost of applying additional head-dimension compression is larger once information has already been removed along the token dimension. In contrast, quantization leaves more redundancy for RACE to exploit, resulting in a milder degradation at comparable total cache budgets.

\begin{table}[t]
\caption{Per-task results of Qwen2.5 7B on Wikitext (perplexity) and zero-shot accuracy on combining RACE with Token compression and quantization methods. }
\label{tab:complementary}
\resizebox{\columnwidth}{!}{%
\begin{tabular}{cccccccccc}
\hline
\textbf{\begin{tabular}[c]{@{}c@{}}Tokens \\  /Precision \\ Retained\end{tabular}} & \textbf{\begin{tabular}[c]{@{}c@{}}RACE\\KV Cache \\Reatained\end{tabular}} & \textbf{\begin{tabular}[c]{@{}c@{}}Total\\ KV Cache\\ Retained\end{tabular}} & \textbf{wikitext}             & \textbf{hellaswag}           & \textbf{piqa}                & \textbf{arc\_challenge}      & \textbf{boolq}               & \textbf{\begin{tabular}[c]{@{}c@{}}mmlu\_\\ college\_cs\end{tabular}} & \textbf{Average}              \\ \hline
\multicolumn{10}{c}{\cellcolor[HTML]{EFEFEF}\textbf{H20 (Token Compression)}}                                                                                                                                                                                                                                                                                                                                                                                                                                               \\ \hline
                                                                                   & 100\%                                                                               & 50\%                                                                         & 12.66                         & 0.66                         & 0.85                         & 0.48                         & 0.82                         & 0.72                                                                  & 0.706                         \\
                                                                                   & \cellcolor[HTML]{EFEFEF}75\%                                                        & \cellcolor[HTML]{EFEFEF}37.5\%                                               & \cellcolor[HTML]{EFEFEF}15.8  & \cellcolor[HTML]{EFEFEF}0.65 & \cellcolor[HTML]{EFEFEF}0.83 & \cellcolor[HTML]{EFEFEF}0.45 & \cellcolor[HTML]{EFEFEF}0.81 & \cellcolor[HTML]{EFEFEF}0.63                                          & \cellcolor[HTML]{EFEFEF}0.674 \\
\multirow{-3}{*}{\begin{tabular}[c]{@{}c@{}}512 \\ (50\%)\end{tabular}}            & 50\%                                                                                & 25\%                                                                         & 17.44                         & 0.62                         & 0.79                         & 0.44                         & 0.74                         & 0.43                                                                  & 0.604                         \\
                                                                                   & \cellcolor[HTML]{EFEFEF}100\%                                                       & \cellcolor[HTML]{EFEFEF}25\%                                                 & \cellcolor[HTML]{EFEFEF}13.62 & \cellcolor[HTML]{EFEFEF}0.66 & \cellcolor[HTML]{EFEFEF}0.85 & \cellcolor[HTML]{EFEFEF}0.48 & \cellcolor[HTML]{EFEFEF}0.79 & \cellcolor[HTML]{EFEFEF}0.66                                          & \cellcolor[HTML]{EFEFEF}0.688 \\
                                                                                   & 75\%                                                                                & 18.75\%                                                                      & 16.12                         & 0.65                         & 0.83                         & 0.45                         & 0.78                         & 0.63                                                                  & 0.668                         \\
\multirow{-3}{*}{\begin{tabular}[c]{@{}c@{}}256\\ (25\%)\end{tabular}}             & \cellcolor[HTML]{EFEFEF}50\%                                                        & \cellcolor[HTML]{EFEFEF}12.5\%                                               & \cellcolor[HTML]{EFEFEF}21.81 & \cellcolor[HTML]{EFEFEF}0.62 & \cellcolor[HTML]{EFEFEF}0.79 & \cellcolor[HTML]{EFEFEF}0.42 & \cellcolor[HTML]{EFEFEF}0.71 & \cellcolor[HTML]{EFEFEF}0.4                                           & \cellcolor[HTML]{EFEFEF}0.588 \\ \hline
\multicolumn{10}{c}{\textbf{SnapKV (Token Compression)}}                                                                                                                                                                                                                                                                                                                                                                                                                                                                    \\ \hline
                                                                                   & \cellcolor[HTML]{EFEFEF}100\%                                                       & \cellcolor[HTML]{EFEFEF}50\%                                                 & \cellcolor[HTML]{EFEFEF}12.55 & \cellcolor[HTML]{EFEFEF}0.68 & \cellcolor[HTML]{EFEFEF}0.86 & \cellcolor[HTML]{EFEFEF}0.48 & \cellcolor[HTML]{EFEFEF}0.82 & \cellcolor[HTML]{EFEFEF}0.71                                          & \cellcolor[HTML]{EFEFEF}0.71  \\
                                                                                   & 75\%                                                                                & 37.5\%                                                                       & 15.61                         & 0.65                         & 0.83                         & 0.46                         & 0.81                         & 0.63                                                                  & 0.676                         \\
\multirow{-3}{*}{\begin{tabular}[c]{@{}c@{}}512 \\ (50\%)\end{tabular}}            & \cellcolor[HTML]{EFEFEF}50\%                                                        & \cellcolor[HTML]{EFEFEF}25\%                                                 & \cellcolor[HTML]{EFEFEF}18.63 & \cellcolor[HTML]{EFEFEF}0.62 & \cellcolor[HTML]{EFEFEF}0.79 & \cellcolor[HTML]{EFEFEF}0.45 & \cellcolor[HTML]{EFEFEF}0.76 & \cellcolor[HTML]{EFEFEF}0.44                                          & \cellcolor[HTML]{EFEFEF}0.612 \\
                                                                                   & 100\%                                                                               & 25\%                                                                         & 12.76                         & 0.68                         & 0.85                         & 0.48                         & 0.81                         & 0.7                                                                   & 0.704                         \\
                                                                                   & \cellcolor[HTML]{EFEFEF}75\%                                                        & \cellcolor[HTML]{EFEFEF}18.75\%                                              & \cellcolor[HTML]{EFEFEF}15.66 & \cellcolor[HTML]{EFEFEF}0.65 & \cellcolor[HTML]{EFEFEF}0.83 & \cellcolor[HTML]{EFEFEF}0.45 & \cellcolor[HTML]{EFEFEF}0.79 & \cellcolor[HTML]{EFEFEF}0.63                                          & \cellcolor[HTML]{EFEFEF}0.67  \\
\multirow{-3}{*}{\begin{tabular}[c]{@{}c@{}}256\\ (25\%)\end{tabular}}             & 50\%                                                                                & 12.5\%                                                                       & 18.44                         & 0.62                         & 0.79                         & 0.46                         & 0.74                         & 0.44                                                                  & 0.61                          \\ \hline
\multicolumn{10}{c}{\cellcolor[HTML]{EFEFEF}\textbf{KIVI (Quantization)}}                                                                                                                                                                                                                                                                                                                                                                                                                                                   \\ \hline
                                                                                   & 100\%                                                                               & 25\%                                                                         & 17.04                         & 0.69                         & 0.86                         & 0.48                         & 0.78                         & 0.7                                                                   & 0.702                         \\
                                                                                   & \cellcolor[HTML]{EFEFEF}75\%                                                        & \cellcolor[HTML]{EFEFEF}18.75\%                                              & \cellcolor[HTML]{EFEFEF}18.17 & \cellcolor[HTML]{EFEFEF}0.66 & \cellcolor[HTML]{EFEFEF}0.84 & \cellcolor[HTML]{EFEFEF}0.48 & \cellcolor[HTML]{EFEFEF}0.76 & \cellcolor[HTML]{EFEFEF}0.6                                           & \cellcolor[HTML]{EFEFEF}0.668 \\
\multirow{-3}{*}{\begin{tabular}[c]{@{}c@{}}2 bits\\ (25\%)\end{tabular}}          & 50\%                                                                                & 12.5\%                                                                       & 21.81                         & 0.64                         & 0.8                          & 0.44                         & 0.71                         & 0.42                                                                  & 0.602                         \\
                                                                                   & \cellcolor[HTML]{EFEFEF}100\%                                                       & \cellcolor[HTML]{EFEFEF}50\%                                                 & \cellcolor[HTML]{EFEFEF}15.7  & \cellcolor[HTML]{EFEFEF}0.69 & \cellcolor[HTML]{EFEFEF}0.86 & \cellcolor[HTML]{EFEFEF}0.48 & \cellcolor[HTML]{EFEFEF}0.78 & \cellcolor[HTML]{EFEFEF}0.7                                           & \cellcolor[HTML]{EFEFEF}0.702 \\
                                                                                   & 75\%                                                                                & 37.5\%                                                                       & 16.58                         & 0.66                         & 0.84                         & 0.48                         & 0.78                         & 0.63                                                                  & 0.678                         \\
\multirow{-3}{*}{\begin{tabular}[c]{@{}c@{}}4 bits\\ (50\%)\end{tabular}}          & \cellcolor[HTML]{EFEFEF}50\%                                                        & \cellcolor[HTML]{EFEFEF}25\%                                                 & \cellcolor[HTML]{EFEFEF}17.14 & \cellcolor[HTML]{EFEFEF}0.64 & \cellcolor[HTML]{EFEFEF}0.8  & \cellcolor[HTML]{EFEFEF}0.44 & \cellcolor[HTML]{EFEFEF}0.74 & \cellcolor[HTML]{EFEFEF}0.42                                          & \cellcolor[HTML]{EFEFEF}0.608 \\ \hline
\end{tabular}}
\end{table}

\paragraph{Ablations}

\begin{figure}[t]
    \centering
    \includegraphics[width=1\linewidth]{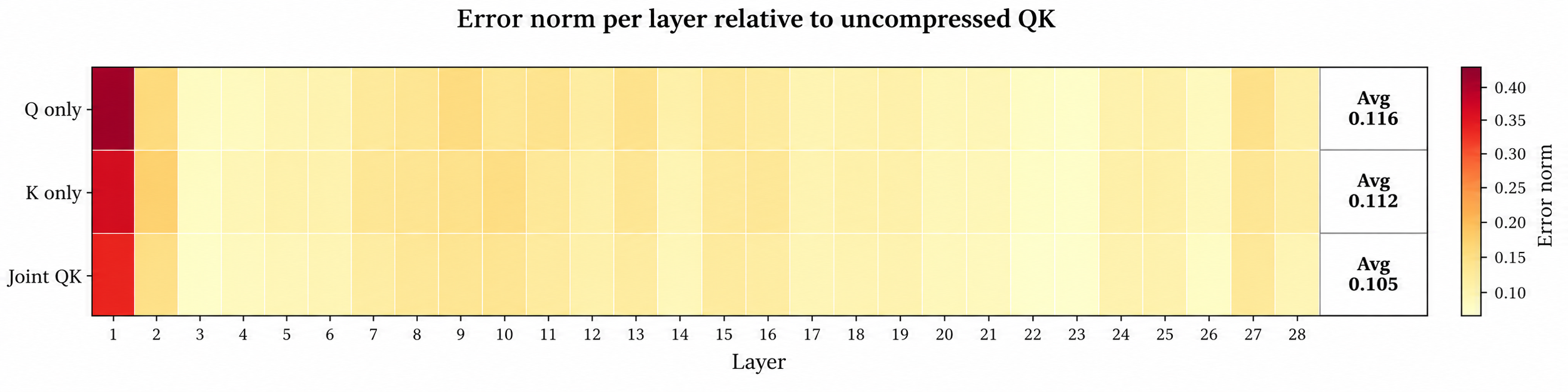}
    \caption{Layer-wise normalised approximation error as in Eq.\ref{eq:pair_energy_bound} with respect to uncompressed model for joint and single-sided energy selectors. Joint Q/K selection yields the lowest average error compared to Q-only or K-only selection.}
    \label{fig:error}
\end{figure}
\begin{table}[t]
\caption{Ablation of selection criteria for RoPE-pair retention. We compare RACE’s joint Q/K activation-energy criterion against magnitude, random selection, gradient Fisher, attention-score selection, and Q-only or K-only energy. Results are reported on WikiText perplexity and zero-shot accuracy across nine tasks at 87.5\%, 75\%, and 50\% KV-cache retention. Joint Q/K energy is consistently strongest overall, with the advantage becoming more pronounced at higher compression.}
\label{tab:ablations}
\resizebox{\columnwidth}{!}{%
\begin{tabular}{ccccccccccccc}
\hline
\textbf{\begin{tabular}[c]{@{}c@{}}KV Cache \\ Retained\end{tabular}} & \textbf{Selection Criteria}                    & \textbf{wikitext}             & \textbf{piqa}                & \textbf{wino}                & \textbf{acr\_c}              & \textbf{sciq}                & \textbf{arc\_easy}           & \textbf{boolq}               & \textbf{race}                & \textbf{mmlu\_cs}            & \textbf{mmlu\_phy}           & \textbf{Avg}                 \\ \hline
                                                                      & \cellcolor[HTML]{EFEFEF}Joint QK Energy (RACE) & \cellcolor[HTML]{EFEFEF}13.88 & \cellcolor[HTML]{EFEFEF}0.74 & \cellcolor[HTML]{EFEFEF}0.67 & \cellcolor[HTML]{EFEFEF}0.46 & \cellcolor[HTML]{EFEFEF}0.93 & \cellcolor[HTML]{EFEFEF}0.70 & \cellcolor[HTML]{EFEFEF}0.77 & \cellcolor[HTML]{EFEFEF}0.42 & \cellcolor[HTML]{EFEFEF}0.48 & \cellcolor[HTML]{EFEFEF}0.35 & \cellcolor[HTML]{EFEFEF}0.61 \\
                                                                      & Magnitude                                      & 14.58                         & 0.74                         & 0.64                         & 0.44                         & 0.89                         & 0.68                         & 0.71                         & 0.41                         & 0.37                         & 0.25                         & 0.57                         \\
                                                                      & \cellcolor[HTML]{EFEFEF}Random                 & \cellcolor[HTML]{EFEFEF}14.95 & \cellcolor[HTML]{EFEFEF}0.71 & \cellcolor[HTML]{EFEFEF}0.63 & \cellcolor[HTML]{EFEFEF}0.45 & \cellcolor[HTML]{EFEFEF}0.89 & \cellcolor[HTML]{EFEFEF}0.67 & \cellcolor[HTML]{EFEFEF}0.65 & \cellcolor[HTML]{EFEFEF}0.39 & \cellcolor[HTML]{EFEFEF}0.35 & \cellcolor[HTML]{EFEFEF}0.25 & \cellcolor[HTML]{EFEFEF}0.55 \\
                                                                      & Gradient Fisher                                & 14.58                         & 0.76                         & 0.66                         & 0.44                         & 0.87                         & 0.69                         & 0.74                         & 0.41                         & 0.38                         & 0.25                         & 0.58                         \\
                                                                      & \cellcolor[HTML]{EFEFEF}Attention Score        & \cellcolor[HTML]{EFEFEF}14.13 & \cellcolor[HTML]{EFEFEF}0.74 & \cellcolor[HTML]{EFEFEF}0.66 & \cellcolor[HTML]{EFEFEF}0.44 & \cellcolor[HTML]{EFEFEF}0.92 & \cellcolor[HTML]{EFEFEF}0.69 & \cellcolor[HTML]{EFEFEF}0.75 & \cellcolor[HTML]{EFEFEF}0.40 & \cellcolor[HTML]{EFEFEF}0.40 & \cellcolor[HTML]{EFEFEF}0.27 & \cellcolor[HTML]{EFEFEF}0.59 \\
                                                                      & Energy Q only                                  & 14.02                         & 0.75                         & 0.66                         & 0.45                         & 0.93                         & 0.69                         & 0.75                         & 0.42                         & 0.40                         & 0.32                         & 0.60                         \\
\multirow{-7}{*}{87.5\%}                                              & \cellcolor[HTML]{EFEFEF}Energy K only          & \cellcolor[HTML]{EFEFEF}14.03 & \cellcolor[HTML]{EFEFEF}0.76 & \cellcolor[HTML]{EFEFEF}0.67 & \cellcolor[HTML]{EFEFEF}0.45 & \cellcolor[HTML]{EFEFEF}0.93 & \cellcolor[HTML]{EFEFEF}0.67 & \cellcolor[HTML]{EFEFEF}0.75 & \cellcolor[HTML]{EFEFEF}0.43 & \cellcolor[HTML]{EFEFEF}0.45 & \cellcolor[HTML]{EFEFEF}0.30 & \cellcolor[HTML]{EFEFEF}0.60 \\
                                                                      & Joint QK Energy (RACE)                         & 14.95                         & 0.76                         & 0.66                         & 0.43                         & 0.93                         & 0.70                         & 0.76                         & 0.42                         & 0.48                         & 0.33                         & 0.61                         \\
                                                                      & \cellcolor[HTML]{EFEFEF}Magnitude              & \cellcolor[HTML]{EFEFEF}15.64 & \cellcolor[HTML]{EFEFEF}0.73 & \cellcolor[HTML]{EFEFEF}0.61 & \cellcolor[HTML]{EFEFEF}0.43 & \cellcolor[HTML]{EFEFEF}0.88 & \cellcolor[HTML]{EFEFEF}0.65 & \cellcolor[HTML]{EFEFEF}0.67 & \cellcolor[HTML]{EFEFEF}0.39 & \cellcolor[HTML]{EFEFEF}0.34 & \cellcolor[HTML]{EFEFEF}0.24 & \cellcolor[HTML]{EFEFEF}0.55 \\
                                                                      & Random                                         & 17.74                         & 0.72                         & 0.56                         & 0.39                         & 0.84                         & 0.62                         & 0.60                         & 0.36                         & 0.25                         & 0.21                         & 0.50                         \\
                                                                      & \cellcolor[HTML]{EFEFEF}Gradient Fisher        & \cellcolor[HTML]{EFEFEF}14.76 & \cellcolor[HTML]{EFEFEF}0.75 & \cellcolor[HTML]{EFEFEF}0.60 & \cellcolor[HTML]{EFEFEF}0.42 & \cellcolor[HTML]{EFEFEF}0.86 & \cellcolor[HTML]{EFEFEF}0.67 & \cellcolor[HTML]{EFEFEF}0.73 & \cellcolor[HTML]{EFEFEF}0.39 & \cellcolor[HTML]{EFEFEF}0.40 & \cellcolor[HTML]{EFEFEF}0.27 & \cellcolor[HTML]{EFEFEF}0.57 \\
                                                                      & Attention Score                                & 14.87                         & 0.74                         & 0.64                         & 0.42                         & 0.91                         & 0.67                         & 0.73                         & 0.40                         & 0.42                         & 0.32                         & 0.58                         \\
                                                                      & \cellcolor[HTML]{EFEFEF}Energy Q only          & \cellcolor[HTML]{EFEFEF}16.78 & \cellcolor[HTML]{EFEFEF}0.74 & \cellcolor[HTML]{EFEFEF}0.66 & \cellcolor[HTML]{EFEFEF}0.42 & \cellcolor[HTML]{EFEFEF}0.91 & \cellcolor[HTML]{EFEFEF}0.68 & \cellcolor[HTML]{EFEFEF}0.75 & \cellcolor[HTML]{EFEFEF}0.41 & \cellcolor[HTML]{EFEFEF}0.43 & \cellcolor[HTML]{EFEFEF}0.33 & \cellcolor[HTML]{EFEFEF}0.59 \\
\multirow{-7}{*}{75\%}                                                & Energy K only                                  & 16.17                         & 0.75                         & 0.65                         & 0.44                         & 0.92                         & 0.68                         & 0.74                         & 0.42                         & 0.43                         & 0.33                         & 0.60                         \\
                                                                      & \cellcolor[HTML]{EFEFEF}Joint QK Energy (RACE) & \cellcolor[HTML]{EFEFEF}16.02 & \cellcolor[HTML]{EFEFEF}0.69 & \cellcolor[HTML]{EFEFEF}0.59 & \cellcolor[HTML]{EFEFEF}0.38 & \cellcolor[HTML]{EFEFEF}0.89 & \cellcolor[HTML]{EFEFEF}0.62 & \cellcolor[HTML]{EFEFEF}0.69 & \cellcolor[HTML]{EFEFEF}0.37 & \cellcolor[HTML]{EFEFEF}0.34 & \cellcolor[HTML]{EFEFEF}0.33 & \cellcolor[HTML]{EFEFEF}0.54 \\
                                                                      & Magnitude                                      & 18.73                         & 0.71                         & 0.54                         & 0.32                         & 0.77                         & 0.49                         & 0.50                         & 0.31                         & 0.23                         & 0.22                         & 0.45                         \\
                                                                      & \cellcolor[HTML]{EFEFEF}Random                 & \cellcolor[HTML]{EFEFEF}26.09 & \cellcolor[HTML]{EFEFEF}0.62 & \cellcolor[HTML]{EFEFEF}0.50 & \cellcolor[HTML]{EFEFEF}0.30 & \cellcolor[HTML]{EFEFEF}0.65 & \cellcolor[HTML]{EFEFEF}0.44 & \cellcolor[HTML]{EFEFEF}0.47 & \cellcolor[HTML]{EFEFEF}0.27 & \cellcolor[HTML]{EFEFEF}0.26 & \cellcolor[HTML]{EFEFEF}0.22 & \cellcolor[HTML]{EFEFEF}0.41 \\
                                                                      & Gradient Fisher                                & 17.80                         & 0.72                         & 0.55                         & 0.35                         & 0.80                         & 0.53                         & 0.44                         & 0.33                         & 0.30                         & 0.27                         & 0.48                         \\
                                                                      & \cellcolor[HTML]{EFEFEF}Attention Score        & \cellcolor[HTML]{EFEFEF}17.86 & \cellcolor[HTML]{EFEFEF}0.70 & \cellcolor[HTML]{EFEFEF}0.59 & \cellcolor[HTML]{EFEFEF}0.33 & \cellcolor[HTML]{EFEFEF}0.73 & \cellcolor[HTML]{EFEFEF}0.49 & \cellcolor[HTML]{EFEFEF}0.50 & \cellcolor[HTML]{EFEFEF}0.35 & \cellcolor[HTML]{EFEFEF}0.33 & \cellcolor[HTML]{EFEFEF}0.31 & \cellcolor[HTML]{EFEFEF}0.48 \\
                                                                      & Energy Q only                                  & 21.01                         & 0.72                         & 0.57                         & 0.35                         & 0.86                         & 0.55                         & 0.66                         & 0.37                         & 0.30                         & 0.33                         & 0.52                         \\
\multirow{-7}{*}{50\%}                                                & \cellcolor[HTML]{EFEFEF}Energy K only          & \cellcolor[HTML]{EFEFEF}20.32 & \cellcolor[HTML]{EFEFEF}0.71 & \cellcolor[HTML]{EFEFEF}0.57 & \cellcolor[HTML]{EFEFEF}0.35 & \cellcolor[HTML]{EFEFEF}0.87 & \cellcolor[HTML]{EFEFEF}0.54 & \cellcolor[HTML]{EFEFEF}0.66 & \cellcolor[HTML]{EFEFEF}0.38 & \cellcolor[HTML]{EFEFEF}0.30 & \cellcolor[HTML]{EFEFEF}0.32 & \cellcolor[HTML]{EFEFEF}0.52 \\ \hline
\end{tabular}
}
\end{table}

We first examine whether the retained RoPE pairs should be selected using information from queries and keys jointly or from either projection alone. Across all retention ratios, the joint Q/K energy criterion gives the strongest overall results. At 87.5\% retention, Q-only and K-only selection remain close to the joint criterion, but the gap becomes clearer as compression increases. At 75\% retention, joint Q/K energy achieves 0.61 average accuracy with a WikiText perplexity of 14.95, compared with 0.59/16.78 for Q-only and 0.60/16.17 for K-only. At 50\% retention, the joint criterion retains 0.54 average accuracy and 16.02 perplexity, while both Q-only and K-only fall to 0.52 average accuracy with substantially higher perplexity.

This trend is also reflected directly in the approximation of the uncompressed attention interaction. Figure \ref{fig:error} shows the layer-wise error norm of the compressed \(QK^\top\) product relative to the uncompressed model. Joint Q/K selection gives the lowest average error, 0.105, compared with 0.112 for K-only and 0.116 for Q-only, and the reduction is observed across most layers rather than being concentrated in a small subset. Together, the downstream results and the layer-wise approximation error support deriving the shared selector from both Q and K rather than from either side in isolation.

We next compare different scoring criteria that all operate on joint Q/K information. Joint activation energy consistently outperforms magnitude, random selection, gradient-Fisher scoring, and attention-based scoring, with the differences becoming larger at stronger compression. At 50\% retention, RACE achieves 0.54 average accuracy, compared with 0.48 for gradient Fisher and attention score, 0.45 for magnitude, and 0.41 for random selection. The same trend is reflected in WikiText perplexity, where joint Q/K energy also achieves the lowest value. Notably, the gradient- and attention-based alternatives do not recover the performance of the simple forward-pass energy criterion despite requiring more complex statistics. Overall, these ablations indicate that both components of the selector are important: the ranking should be derived jointly from Q and K, and joint activation energy provides an effective criterion for identifying the RoPE pairs that best preserve the original attention interaction.

\end{document}